\documentclass[journal]{IEEEtran}
\IEEEoverridecommandlockouts
\usepackage[T1]{fontenc}
\usepackage{balance}
\usepackage{silence}
\IEEEoverridecommandlockouts

\usepackage{cite}
\usepackage{graphicx}
\usepackage{picinpar}
\usepackage{stfloats}
\usepackage{url}
\usepackage[latin1]{inputenc}
\usepackage{colortbl}
\usepackage{multirow}
\usepackage{pifont}
\usepackage{enumerate}
\usepackage{pbox}
\usepackage{amsmath,amssymb,amsfonts}
\usepackage{amsthm}
\usepackage{algorithmic}
\usepackage{algorithm}
\usepackage{bm}
\usepackage{graphbox}
\usepackage{textcomp}
\usepackage{xcolor}
\usepackage[export]{adjustbox}
\usepackage{color}
\usepackage{tikz}
\def\BibTeX{{\rm B\kern-.05em{\sc i\kern-.025em b}\kern-.08em
    T\kern-.1667em\lower.7ex\hbox{E}\kern-.125emX}}

\definecolor{LightCyan}{rgb}{0.8,0.8,1.0}
\definecolor{LightRed}{rgb}{1.0,0.8,0.8}
\definecolor{LightGreen}{rgb}{0.8,1.0,0.8}
\definecolor{LightYellow}{rgb}{1.0,1.0,0.8}

\newtheorem{corollary}{Corollary}

\newtheorem{proposition}{Proposition}

\ActivateWarningFilters[balance]

\makeatletter
\let\NAT@parse\undefined
\makeatother
\usepackage[hidelinks]{hyperref}

\title{AVOCADO: Adaptive Optimal Collision \\ Avoidance driven by Opinion}

\author{Diego Martinez-Baselga*, \and Eduardo Sebasti\'{a}n*, \and Eduardo Montijano, \\ \and Luis Riazuelo, \and Carlos Sag\"{u}\'{e}s and Luis Montano
\thanks{The authors are with the DIIS - I3A, Universidad de Zaragoza, Spain (\texttt{\small \{diegomartinez, esebastian, emonti, riazuelo, montano, csagues\}@unizar.es}). *Equal contribution}%
\thanks{This work has been supported by Spanish projects PID2022-139615OB-I00, PID2021-125514NB-I00, PID2021-124137OB-I00 and TED2021-130224B-I00 funded by MCIN/AEI/10.13039/501100011033/FEDER-UE, by ERDF A way of making Europe and by the European Union NextGenerationEU/PRTR, ONR Global grant N62909-24-1-2081, DGA T45-23R, and Spanish grants FPU19-05700 and PRE2020-094415.}%
}

\newcommand\copyrighttext{%
  \footnotesize \textcopyright This paper has been accepted for publication at IEEE Transactions on Robotics. Please, when citing the paper, refer to the official manuscript with the following DOI: 10.1109/TRO.2025.3552350.}
\newcommand\copyrightnotice{%
\begin{tikzpicture}[remember picture,overlay]
\node[anchor=south,yshift=10pt] at (current page.south) {\fbox{\parbox{\dimexpr\textwidth-\fboxsep-\fboxrule\relax}{\copyrighttext}}};
\end{tikzpicture}%
}

\begin{document}
\maketitle
\copyrightnotice

%%%%%%%%%%%%          
% ABSTRACT %
%%%%%%%%%%%%

\begin{abstract}
We present \textsf{AVOCADO} (AdaptiVe Optimal Collision Avoidance Driven by Opinion), a novel navigation approach to address holonomic robot collision avoidance when the robot does not know how cooperative the other agents in the environment are. \textsf{AVOCADO} departs from a Velocity Obstacle's (VO) formulation akin to the Optimal Reciprocal Collision Avoidance method. However, instead of assuming reciprocity, it poses an adaptive control problem to adapt to the cooperation level of other robots and agents in real time. This is achieved through a novel nonlinear opinion dynamics design that relies solely on sensor observations. As a by-product, we leverage tools from the opinion dynamics formulation to naturally avoid the deadlocks in geometrically symmetric scenarios that typically suffer VO-based planners.
Extensive numerical simulations show that \textsf{AVOCADO} surpasses existing motion planners in mixed cooperative/non-cooperative navigation environments in terms of success rate, time to goal and computational time. In addition, we conduct multiple real experiments that verify that \textsf{AVOCADO} is able to avoid collisions in environments crowded with other robots and humans.
\end{abstract}


%%%%%%%%%%%%%%%%           
% INTRODUCTION %
%%%%%%%%%%%%%%%%

\section{Introduction}\label{sec:intro}
\IEEEPARstart{R}{obot} navigation in dynamic environments is the task of moving a robot from one location to another while avoiding collisions with obstacles that are in motion around the environment\cite{choset2005principles}. Collision avoidance is particularly challenging when the dynamic obstacles are agents (humans, other robots, etc.), and the problem of accounting for the effect of the robot motion on the other agents is still unsolved\cite{mavrogiannis2023core}. Agents have their own intent and they exhibit different degrees of cooperation, i.e. they adjust their behavior to avoid collision with the robot to a certain extent that is unknown to the robot. Thus, it is hard for robots to predict their future motion and effectively plan to avoid potential collisions\cite{fisac2018probabilistically,rudenko2020human}. The ability of avoiding collisions in crowded dynamic environments is a low-level feature that is of key importance in robotic applications such as social robotics\cite{fridovich2020confidence,mirsky2024conflict}, robots in the wild \cite{tian2020search,tabib2021autonomous,soria2021predictive} or autonomous driving\cite{paden2016survey, song2023reaching}. 

Existing solutions (Section \ref{sec:related}) make one or more of the following assumptions: (i) the degree of cooperation between the robot and an agent is known a priori (e.g., the robot knows that the agent %is polite and will leave space for the robot to cross
is going to collaborate in the collision avoidance and how), restricting the solution to predetermined cooperative settings\cite{van2011reciprocal,alonso2018cooperative,han2020cooperative,boldrer2023rule}; (ii) the robot can communicate with other robots and/or agents\cite{bajcsy2019scalable,li2020graph,patwardhan2022distributing,zhang2023neural}, which might not be possible during deployment due to a lack of cooperation, communication delays, packet losses or disturbances from external factors; (iii) the computational capabilities of the robot are sufficient to do inference on learning-based policies or long-term horizon planning\cite{long2017deep,everett2021collision,brito2019model,poddar2023crowd}, which is hard to ensure in low-cost onboard platforms if we take into account that collision avoidance is just a low-level feature among a myriad of desired high-level and computationally demanding capabilities (e.g., multi-robot active perception or semantic mapping \cite{morilla2023robust, asgharivaskasi2023semantic}).

\begin{figure}
    \centering
    \includegraphics[trim={2cm 10cm 15cm 7cm}, clip, width=\columnwidth]{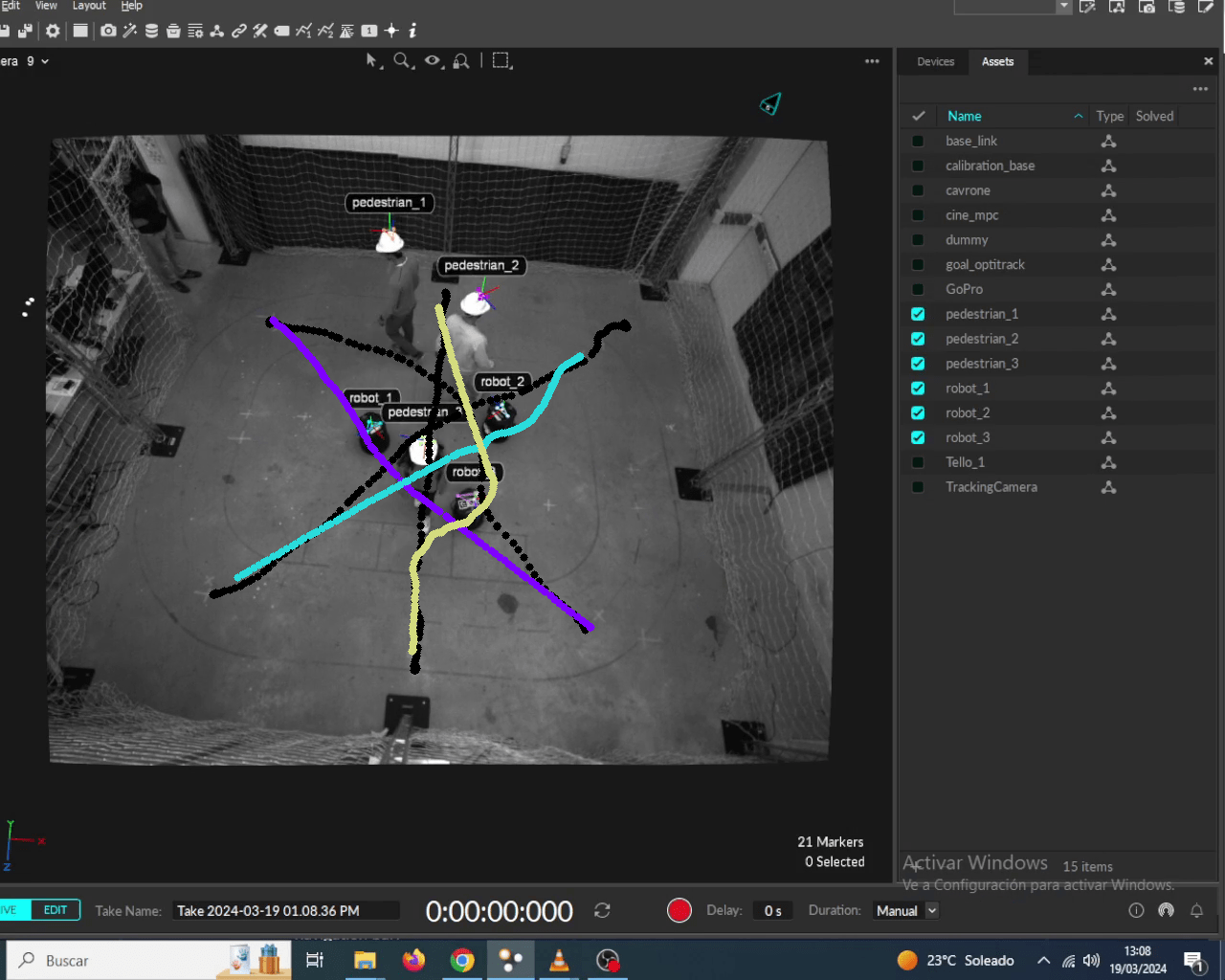}
    \caption{Illustrative example of one of the experiments with real robots and humans. Robots using \textsf{AVOCADO} adapt online to avoid collisions with the other entities in the arena despite not knowing the degree of cooperation of the other robots and humans. Section \ref{sec:experiments} discusses all the experiments in detail.}
    \label{fig:teaser}
\end{figure}

To cope with all these issues, \textbf{the main contribution} of this work is \textsf{AVOCADO} (\textbf{A}dapti\textbf{V}e \textbf{O}ptimal \textbf{C}ollision \textbf{A}voidance \textbf{D}riven by \textbf{O}pinion), a novel method for collision avoidance in cooperative and non-cooperative dynamic environments (Section \ref{sec:solution}). \textsf{AVOCADO} is built upon a geometrical formulation based on the classical Velocity Obstacle (\textsf{VO}, \cite{fiorini1993motion, fiorini1998motion}) framework (Section \ref{sec:prosta}). However, different from other \textsf{VO}-based approaches that assume either no cooperation or reciprocity \cite{van2008reciprocal,van2011reciprocal,alonso2013optimal}, we consider that the degree of cooperation between the robot and the agent is unknown and do not assume any predetermined behavior for the agent or communication infrastructure. To address uncertainty on the degree of cooperation, we propose a novel nonlinear opinion dynamics adaptive scheme that enables the robot to adapt to the agent's behavior in real-time from just its onboard sensor observations.
% Given a predefined desired velocity of an holonomic robot, we propose an optimization program that finds the closest velocity to the desired one that fulfills a set of dynamic geometrical constraints from the \textsf{VO}.  
Furthermore, we exploit the nonlinear opinion dynamics adaptive law to develop a method that avoids the geometrical deadlocks coming from symmetries that suffer some existing  \textsf{VO}-based methods. We validate our proposal with simulations (Section \ref{sec:simulations}) and real experiments (Section \ref{sec:experiments}), showing that \textsf{AVOCADO} surpasses exiting collision avoidance approaches in environments that mix cooperative and non-cooperative agents in terms of avoided collisions while featuring an inexpensive computational cost. The scenarios include examples that combine real robots and humans in tight spaces, concluding (Section \ref{sec:conclusion}) that \textsf{AVOCADO} is a promising alternative for low-cost robot navigation in crowded environments where agents exhibit unknown degrees of cooperation.  

The code of the work is implemented in C++/Python and can be accessed in the repository of the work\footnote{\url{https://github.com/dmartinezbaselga/AVOCADO}}, as well as the videos of the experiments.

\section{Related work}\label{sec:related}
Research on collision avoidance was primarily dominated by geometrical approaches. Departing from robot observations, methods originally only considered static obstacles, as the Dynamic Window Approach~\cite{brock1999high}. Others like potential fields\cite{khatib1986real}, inevitable collision states \cite{petti2005safe} or social forces\cite{ferrer2013robot} compute control actions that reactively lead the robot towards collision-free regions. They are simple to compute but lack formal performance guarantees, leading to over-conservative policies. By contrast, \textsf{VO}-based methods\cite{vesentini2024survey} preserve the computational simplicity but guarantee optimality in terms of deviation with respect to the desired predefined velocity either for holonomic or non-holonomic robots \cite{alonso2013optimal, rufli2013reciprocal,bareiss2015generalized}. Reciprocal \textsf{VO} (\textsf{RVO}) extends VO by assuming reciprocity in the collision avoidance \cite{van2008reciprocal}, i.e, it assumes that agents encountered in the environment cooperate equally to avoid the collision. Specially relevant is Optimal Reciprocal Collision Avoidance \textsf{ORCA}~\cite{van2011reciprocal}, which builds a half-space geometrical constraint per agent that is incorporated in a linear program. Nevertheless, the reciprocal avoidance assumption may not be applied in mixed cooperative/non-cooperative environments, which are the ones considered in this paper.

Some solutions \cite{guo2021vr} remove the 
reciprocity assumption in \textsf{ORCA} and formulate an additional optimization problem to reduce the time the agents take to reach the goal. Instead of using optimization, 
\textsf{ORCA} can also be implemented as a low-level module integrated in a high-level learned policy that accounts for social and traffic cues\cite{qin2023srl}. Both approaches require 1-hop communication, which is not feasible in a non-cooperative environment due to the lack of cooperation among agents. Instead, \textsf{AVOCADO} is fully decentralized and only relies on the perception of the robot.

Many recent methods use deep learning for robot navigation. Either from a supervised\cite{long2017deep,tai2018socially,pokle2019deep, xie2021towards, sebastian2023lemurs} or reinforcement learning perspective\cite{long2018towards,fan2020distributed,everett2021collision,ourari2022nearest,cui2023scalable,martinez2023improving, sebastian2023physics}, deep learning develop end-to-end policies that map the perception observations into control commands\cite{han2020cooperative} or sub-goals that are fed into a low-level horizon planning controller\cite{brito2021go}. For instance, Long-Short Term Memory (LSTM) networks can be combined with barrier certificates\cite{yu2023sequential} to improve collision avoidance success rate and scalability. There are also works that combine \textsf{ORCA} scenario knowledge with reinforcement learning~\cite{qin2023srl,han2022reinforcement}. Nonetheless, deep learning approaches have a high memory and computational cost, often lack collision avoidance guarantees and fail in out-of-distribution settings.

The lack of guarantees of purely learning-based methods is addressed by reachability methods\cite{bajcsy2019scalable}, control barrier functions\cite{wang2017safety} and Model Predictive Control (MPC)\cite{morgan2014model}. Reachability-based approaches depart from the Hamilton-Jacobi equation\cite{althoff2021set} to build a safe set in the state space that fulfills all the desired constrains. Forward computation of the Hamilton-Jacobi equation is highly demanding and over-conservative\cite{chen2016multi}, so learning-based alternatives\cite{bansal2021deepreach,julian2021reachability,li2021prediction} propose approximations that trade computation and performance guarantees. On the other hand, control barrier functions formulate an optimization problem that finds the closest control action to the desired one that satisfies a set of desired forward invariance constraints, stemming from a control stability definition of the desired safe set. It is not clear how to extend control barrier functions to dynamic obstacles with unknown behaviors\cite{long2022safe} or communication-denied settings\cite{panagou2015distributed,chen2020guaranteed}. Finally, MPC shares with reachability methods the use of a time-horizon propagation and with control barrier functions the optimization-based formulation\cite{schulman2014motion}. MPC is flexible because it enables to satisfy additional goals or constraints such as motion uncertainties\cite{ryu2024integrating} or local maps\cite{brito2019model}. However, it has a significant computational burden that is only alleviated by reducing the time-horizon and the number of considered nearby agents, resulting in undesired local minima or oscillations.  Different from these alternatives, \textsf{AVOCADO} is inexpensive to compute, is not over-conservative as it minimizes the deviation with respect to the desired velocity, and does not suffer from symmetry deadlocks and oscillations.

One key characteristic of the aforementioned methods is that, in non-cooperative environments with unknown intents, a prediction model of the behavior of the other agents is required to effectively avoid collisions\cite{de2023scenario, de2024topology}. The first option is to use learned predictors that directly obtain future trajectory states of the agents given current and past samples\cite{katyal2020intent, liu2023intention}. The second option is to rely on a learned model\cite{mavrogiannis2022winding}, such that it can be integrated in a MPC program\cite{poddar2023crowd} to simultaneously optimize over the trajectories of the robot and the other agents. In both cases, the success of the collision avoidance module is subject to the accuracy of the complex prediction modules; instead, \textsf{AVOCADO} adapts to
unknown degrees of cooperation through a novel nonlinear opinion dynamics adaptive law that is inexpensive to compute, does not require learning and is only based on current perception observations, without relying on prediction of future trajectories.

Opinion dynamics \cite{altafini2012consensus,jia2015opinion,cisneros2020polarization} originate from studies on how to model social interactions. The ineffectiveness of linear models to represent behaviors like saturation of information leads to nonlinear opinion dynamics\cite{bizyaeva2022nonlinear,leonard2024fast}. They have been applied in a wide variety of fields, such as explaining political polarization~\cite{leonard2021nonlinear} or perception and reaction to epidemics~\cite{ordorica2024opinion}; lately, they have been applied in robotic problems, mainly combining them with game theory for cooperative and non-cooperative multi-agent decision making \cite{park2021tuning,hu2023emergent,hu2024think} or multi-robot task allocation~\cite{bizyaeva2022switching}. Regarding collision avoidance, there is an important property of nonlinear opinion dynamics. They ensure the existence of a bifurcation in the state-space, i.e., the existence of two simultaneous stable equilibrium points. This feature has been exploited \cite{cathcart2023proactive} to design a collision avoidance method that avoids deadlocks, where the opinion represents the preference of the robot to move left or right. A recent work \cite{amorim2024spatially} extends opinion dynamics from the planar to the circle space to represent, e.g., the heading angle of a robot. Nevertheless, these approaches do not guarantee collision avoidance nor consider other motion patterns than only changing the heading angle. \textsf{AVOCADO} exploits 
nonlinear opinion dynamics to guarantee that the robot decides if the agent is cooperative or non-cooperative and in what degree, leaving the collision avoidance guarantees to a \textsf{VO}-based program.

%%%%%%%%%%%%%%%%%%%           
%% PRELIMINARIES %%
%%%%%%%%%%%%%%%%%%%

\section{Problem formulation}\label{sec:prosta}
Consider a robot navigating in an environment populated with $\mathsf{N}>0$ agents. The robot follows single integrator holonomic dynamics given by
\begin{equation}\label{eq:robot_dynamics}
    \dot{\mathbf{p}_r} = \mathbf{v}_r^{\mathsf{pre}},
\end{equation}
where $\mathbf{p}_r \in \mathbb{R}^2$ is the position of the robot and $\mathbf{v}_r^{\mathsf{pre}} \in \mathbb{R}^2$ is a prescribed velocity command generated by a higher-level motion planner. On the other hand, each agent is associated with an index $i \in \{1, \dots, \mathsf{N} \}$, and has a position $\mathbf{p}_i \in \mathbb{R}^2$ and a velocity $\mathbf{v}_i \in \mathbb{R}^2$. The robot can only sense the agents that are within a disc centered in the robot position and of perception radius $r_p > 0$
\begin{equation}\label{eq:perception_disc}
    \mathsf{D}_{p}(\mathbf{p}_r, r_p) = \{\mathbf{p} \mid ||\mathbf{p} - \mathbf{p}_r|| < r_p \},
\end{equation}
where $||\bullet||$ is the L2-norm of a vector and $\mathbf{p} \in \mathbb{R}^2$. Therefore, the set of neighboring agents of the robot is defined as \mbox{$\mathcal{N} = \{i \mid \mathbf{p}_i \in \mathsf{D}_p(\mathbf{p}_r, r_p)\}$}. The purpose of the paper is to develop an algorithm that avoids collisions between the robot and the agents. To that end, we define the set of locations that lead to collisions as a disc centered in the position of the robot and security radius $r_s > 0$
\begin{equation}\label{eq:collision_disc}
    \mathsf{D}_{c}(\mathbf{p}_r, r_s) = \{\mathbf{p} \mid ||\mathbf{p} - \mathbf{p}_r|| < r_s \},
\end{equation}
where $r_s$ is given by the geometry of the robot and safety requirements. Agents are also characterized by a collision radius $r_i>0$. The degree of cooperation of agent $i$ with respect to the robot is modeled through a continuous variable $\alpha_i \in [0, 1]$, where $\alpha_i = 1$ means that the agent is fully cooperative with the robot (i.e., agent $i$ will do all the efforts at its hands to avoid collision), whereas $\alpha_i = 0$ means that the agent is completely non-cooperative (i.e., agent $i$ will ignore the robot and continue its movement without any collision avoidance effort). We assume that agents are not competitive, i.e., agents do not try to force collision with the robot as in, e.g., pursuit-evasion\cite{chung2011search} or herding\cite{sebastian2022adaptive} problems. As in realistic mixed cooperative/non-cooperative environments, the degree of cooperation of the agents is unknown.

To model collisions, we exploit the concept of Velocity Obstacle\cite{fiorini1993motion, fiorini1998motion}. Given a certain time instant $t \geq 0$, the Velocity Obstacle of the robot and agent $i$ is the set of velocities that can lead to a collision
\begin{equation}\label{eq:vo_def}
    \mathsf{VO}_{i} = \{\mathbf{v}_r\mid \exists t \in [0,\tau] \text{ s.t. } t\mathbf{v}_r \in \mathsf{D}_c(\mathbf{p}_i-\mathbf{p}_r, r_s+r_i)\}.
\end{equation}
In $\mathsf{VO}_{i}$, $\tau>0$ is the time horizon to check collisions between the robot and agent $i$ happens for any velocity $\mathbf{v}_r \in \mathbb{R}^2$. In Section \ref{sec:solution} we make use of $\mathsf{VO}_{i}$ to develop an optimization program that uses the adapted degree of cooperation to compute a velocity $\mathbf{v}_r^*$ as close as possible to $\mathbf{v}_r^{\mathsf{pre}}$ that guarantees collision avoidance. 

To adapt to the unknown degree of cooperation, we rely on the concept of nonlinear opinion dynamics\cite{bizyaeva2022nonlinear}. Opinion dynamics emerge as a means of modeling interactive behaviors where nodes in a network can choose among a discrete set of options regarding one or more topics. 
Let $o_j \in [-1, 1]$ be the opinion that node $j$ has about a certain topic. Typically, $o_j=0$ represents neutrality whereas $o_j=\{-1, 1\}$ represent extreme opinions over a topic with two possible opinions. By interacting with other nodes, node $j$ can evolve its opinion. For a topic with two opinions, one possible nonlinear opinion dynamics design is 
\begin{align}\label{eq:opinion_original}
    \dot{o}_j = -d_j o_j + g(o_j)f\left(a_j o_j + \sum_{k=1, k\neq j}^{\mathsf{M}}c_{jk}o_k\right) + b_j.
\end{align}

In the equation above, $d_j>0$ tunes how fast the current opinion vanishes with time; $b_j \in \mathbb{R}$ is a bias that models inherent priorities that the node $j$ has on the topic; $g(o_j)$ is an attention term that evolves with time and weights the importance of the influence of other nodes; $a_j>0, c_{jk}>0$ are consensus-like weights that model the exchange of opinions among the $\mathsf{M}>0$ nodes of the network; and $f(\bullet)$ is a nonlinear function, such as $\mathsf{tanh}$ or $\mathsf{sigmoid}$, that bounds the influence of the interactions. More importantly, the nonlinear function in the opinion dynamics enforces a bifurcation phenomenon. The bifurcation phenomenon is characterized by the appearance of a pitchfork bifurcation in the evolution of the state variable $o_j$ that leads to two non-zero stable equilibrium, each of them associated to one of the opinions. Whether the state evolves to one or another depends on the interactions between nodes. For more details on this, we refer the reader to \cite{bizyaeva2022nonlinear}. In relation with this work, the inherent bifurcation property of nonlinear opinion dynamics and the conceptual relationships between the opinion and the parameters of Eq. \eqref{eq:opinion_original} motivate their use to estimate (``build an opinion'') on the unknown degree of cooperation of the other agents. 

In this work, we will exploit nonlinear opinion dynamics to design a novel adaptive law that adjusts the unknown degree of cooperation of the agents in real-time, using only relative position and velocity measurements from onboard sensors and without the need of any communication infrastructure, prior knowledge on their degree of cooperation or other unfeasible assumptions. As a practical note, all the expressions that involve continuous-time dynamic systems are implemented using an Euler forward discretization scheme, with sample time $T>0$. In this work, we select $T=0.05$s to match the real-time operation of the onboard sensing and computation capabilities of the robot.

%%%%%%%%%%%%%%           
%% SOLUTION %%
%%%%%%%%%%%%%%

\section[AVOCADO: Adaptive Optimal Collision Avoidance driven by Opinion]{AVOCADO: Adaptive Optimal Collision \\ Avoidance driven by Opinion}\label{sec:solution}
\subsection{Optimal Collision Avoidance for Unknown Degrees of Cooperation}\label{subsec:oca}

We propose \textsf{AVOCADO} to address the robot collision avoidance problem in multi-agent dynamic environments. The first step towards deriving \textsf{AVOCADO} is to consider the geometry defined by $\mathsf{VO}_i$ for all $i \in \{1, \dots, \mathsf{N}\}$. Let assume that the relative velocity $\mathbf{v}_r^{\mathsf{pre}} - \mathbf{v}_i \in \mathsf{VO}_i$ for some agent $i$. This means that, if robot and agent $i$ follow $\mathbf{v}_r^{\mathsf{pre}}$ and $ \mathbf{v}_i$ respectively, then collision will happen at any time in $[0, \tau]$. The velocity $\mathbf{v}_r^{\mathsf{pre}}$ is provided by some higher-level planner, while $\mathbf{v}_i$ is known from onboard sensor measurements. The latter can be measured or estimated in real-time from onboard cameras or LiDAR \cite{eppenberger2020leveraging,dong2020real}. Since cooperation of agent $i$ cannot be assumed, then the robot needs to choose a new velocity $\mathbf{v}_r^*$ to avoid the collision.

Let $\partial \mathsf{VO}_i$ denote the boundary of the velocity obstacle set $\mathsf{VO}_i$. Then, the minimum change in velocity that makes \mbox{$\mathbf{v}_r^{\mathsf{pre}} - \mathbf{v}_i \notin \mathsf{VO}_i$} is given by the vector $\mathbf{u}_i$
\begin{equation}\label{eq:u}
    \mathbf{u}_{i} = \underset{\mathbf{v}_r\in\partial\mathsf{VO}_{i}}{\arg\min} || \mathbf{v}_r-(\mathbf{v}_r^{\mathsf{pre}} - \mathbf{v}_i)||-(\mathbf{v}_r^{\mathsf{pre}} - \mathbf{v}_i)
\end{equation}
\begin{figure}
    \centering
    \includegraphics[width=\columnwidth]{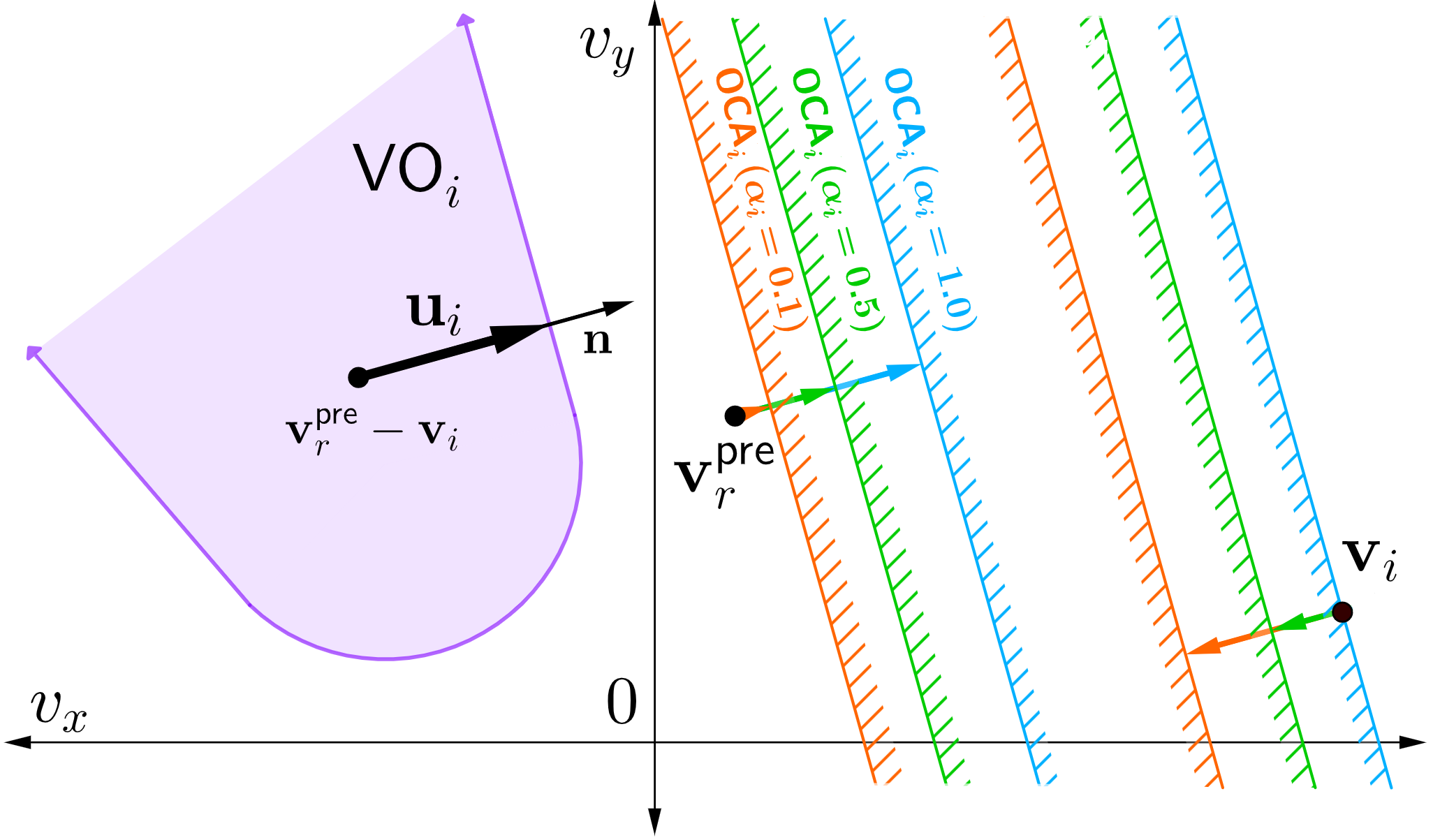}
    \caption{Geometry of $\mathsf{VO}_{i}$ (purple) and admissible velocities' regions associated to different degrees of cooperation (orange, green, blue). \textsf{AVOCADO} selects the closest velocity to the desired one that is inside all the admissible velocity sets. }
    \label{fig:VO_OD}
\end{figure}
Figure \ref{fig:VO_OD} depicts the geometrical reasoning behind Eq. \eqref{eq:u}. Given the Euclidean space of 2D velocities, vector $\mathbf{u}_i$ is a perpendicular vector to the boundary of the velocity obstacle set and with minimum magnitude. To ensure collision avoidance, the robot and agent $i$ together must, at least, exert vector $\mathbf{u}_i$. Let $\alpha_i$ be the degree of cooperation of agent $i$ with respect to the robot such that agent $i$ exerts the $\alpha_i$ part of $\mathbf{u}_i$. Then, to ensure collision avoidance, the robot has to exert, at least, the $(1-\alpha_i)$ part of $\mathbf{u}_i$, which defines a constraint set of admissible velocities
\begin{equation}\label{eq:OCA_set}
    \mathsf{OCA}_{i} = \{\mathbf{v}_r \mid (\mathbf{v}_r-(\mathbf{v}_r^{\mathsf{pre}} + (1-\alpha_i)\mathbf{u}_{i}))\cdot \mathbf{n} \geq 0\},
\end{equation}
where $\mathbf{n}$ is the normal to $\partial\mathsf{VO}_{i}$ at the point given by the head of vector $\mathbf{u}_i$. The set $\mathsf{OCA}_{i}$ defines a half-space constraint of admissible velocities for the robot to avoid collision. As depicted in Fig. \ref{fig:VO_OD}, the greater the value of $\alpha_i$, the more cooperative is agent $i$ and therefore the less severe the restriction on the admissible velocities for the robot. Note that $\alpha_i = 0.5$ corresponds to the \textsf{ORCA} algorithm \cite{van2011reciprocal}, which corresponds to perfect reciprocity between the robot and agent $i$.

Given $\mathsf{OCA}_{i}$ for all $i \in \{1,\hdots, \mathsf{N}\}$, \textsf{AVOCADO} selects the closest velocity to $\mathbf{v}_r^{\mathsf{pre}}$ that respects all the constraints sets at the same time. Formally, this is given by the following linear program
\begin{subequations}\label{eq:optimization}
\begin{alignat}{2}
\mathbf{v}_r^* = &\underset{\mathbf{v}_r}{\arg\min}         \quad ||\mathbf{v}_r^{\mathsf{pre}}-\mathbf{v}_r||
\label{eq:optProb}
\\
   s.t. &\quad \mathbf{v}_r \in \mathsf{OCA}_{i} \quad \forall i \in \mathcal{N}. \label{eq:constraint1}
\end{alignat}
\end{subequations}
Problem \eqref{eq:optimization} is guaranteed to have a solution unless the intersection of the $\mathsf{OCA}_i$ sets is empty. In this case, the robot computes the velocity that minimizes the distance with the half-planes $\mathsf{OCA}_i \quad \forall i\in \mathcal{N}$, i.e., it chooses the ``least unsafe'' velocity in the absence of a velocity that guarantees collision avoidance in that instant (see Section 5.3 of \cite{van2011reciprocal} for further details). We remark that \eqref{eq:optimization} is a natural extension of \textsf{ORCA} to settings where the degree of cooperation of the agents is not necessarily reciprocal. Fig.~\ref{fig:opt-problem} shows a representation of the constrained set \eqref{eq:constraint1}.

\begin{figure}
    \centering
    \includegraphics[width=0.5\linewidth]{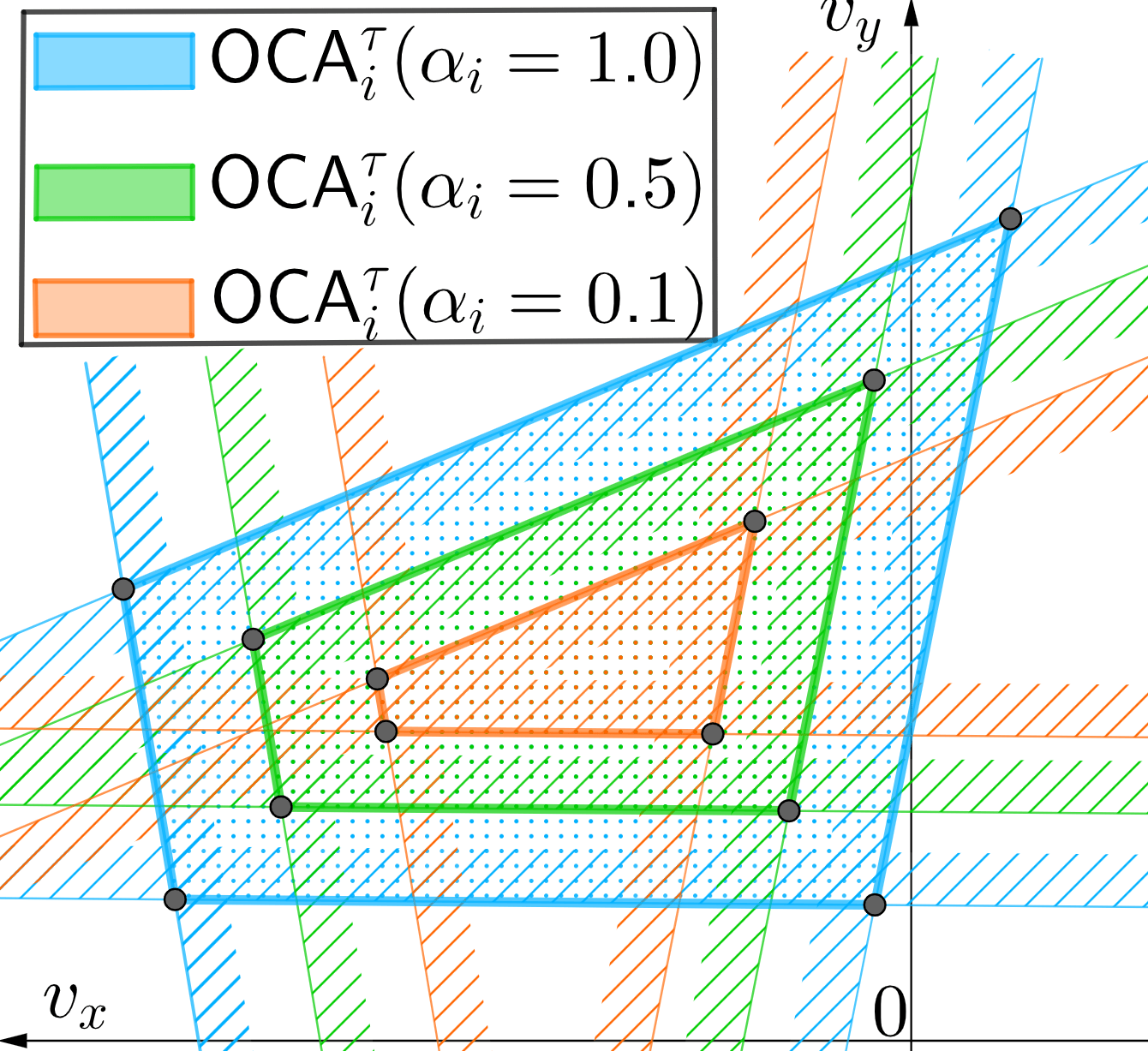}
    \caption{Representation of the intersection set specified by the $\mathsf{OCA}_i$ half-planes of $4$ agents for different values of $\alpha_i$. Note that $\alpha_i$ may be different for each of the agents, but it is considered the same in this figure for a simple visualization. If the intersection set of all half-planes is not empty, the problem is guaranteed to have a solution.}
    \label{fig:opt-problem}
\end{figure}

The challenge is how to estimate $\alpha_i$ for all perceived agents $i\in\mathcal{N}$ from local sensing. Next, we describe how \textsf{AVOCADO} exploits opinion dynamics to adapt online to the degree of cooperation of each agent.

\subsection{Nonlinear Opinion Dynamics Adaptive Law}\label{subsec:NOD}

To derive the adaptive law based on nonlinear opinion dynamics, we firstly define, for convenience, the following change of coordinates:
\begin{equation}\label{eq:change_coordinates}
    {o}_{i} = 2\alpha_{i} - 1 \Leftrightarrow \alpha_{i} = \frac{{o}_{i}+1}{2}.
\end{equation}
The purpose of this change is to shift the degree of cooperation to better suit the definition of nonlinear opinion dynamics provided in Eq. \eqref{eq:opinion_original}. From now on, we refer to ${o}_i$ as the shifted degree of cooperation of agent $i$ with respect to the robot. 

The (shifted) degree of cooperation associated to each agent is a quantity that only involves a pair-wise interaction between the robot and the agent. This motivates the following design for the nonlinear opinion dynamics adaptive law
\begin{equation}\label{eq:NODAL}
    \dot{{o}}_{i} = -d_{i} {o}_{i} + d_i A_i\mathsf{tanh}\left(a_{i} {o}_{i} + c_{i} {e}_i\right) + b_{i}.
\end{equation}
% \begin{equation}\label{eq:NODAL}
%     \dot{{o}}_{i} = -d_{i} {o}_{i} + d_i A_i\mathsf{tanh}\left(a_{i} {o}_{i} + c_{i} {e}_i\right) + b_{i}
% \end{equation}
Compared to \eqref{eq:opinion_original}, Eq. \eqref{eq:NODAL} simplifies the consensus term to a weighted sum of the current shifted degree of cooperation ${o}_{i}$ and a quantity ${e}_i$ that will be defined later and which represents an estimate of the opinion agent $i$ has about ${o}_i$. Ideally, the consensus term would include not only ${e}_i$, but also all the other agents. However, this would require some sort of communication infrastructure or $k$-hop neighboring dependence that is not feasible in a real setting that only depends on onboard sensing. Therefore, we assume that the opinion of the robot on the degree of cooperation of agent $i$ is a pair-wise interaction isolated from other agents. In this sense, the design of \eqref{eq:NODAL} manifests a fundamental difference with respect to existing nonlinear opinion dynamics formulations, tailored towards a feasible implementation in a real robotic platform. The quantities $a_i, b_i, c_i, d_i$ are gains with similar meaning to those in Eq. \eqref{eq:opinion_original} and which parameterized the adaptive law. For instance, $b_i=-0.5$ represents some inductive bias on the degree of cooperation of agent $i$ towards not cooperating with the robot. As nonlinear function, we choose $\mathsf{tanh}$ to respect the domain of ${o}_i$. Finally, $A_i$ represents the attention, and is a variable that dynamically evolves with time, in contrast to the algebraic relationship $g({o}_j)$ formulated in Eq. \eqref{eq:opinion_original}. As depicted in Fig. \ref{fig:NOD}, the idea is to design and tune \eqref{eq:NODAL} such that, when the agent $i$ is close to the robot and does not collaborate to avoid the collision, the (shifted) degree of cooperation associated to agent $i$ evolves such that the robot takes the responsibility. On the other hand, if agent $i$ modifies its velocity to avoid collision, then the (shifted) degree of cooperation evolves to enable the robot a more selfish behavior. The attention term $A_i$ appears scaled in \eqref{eq:NODAL} by gain $d_i$. The reason behind this scaling will be explained later, in Section \ref{subsec:tuning}.

\begin{figure}
    \centering
    \includegraphics[width=\columnwidth]{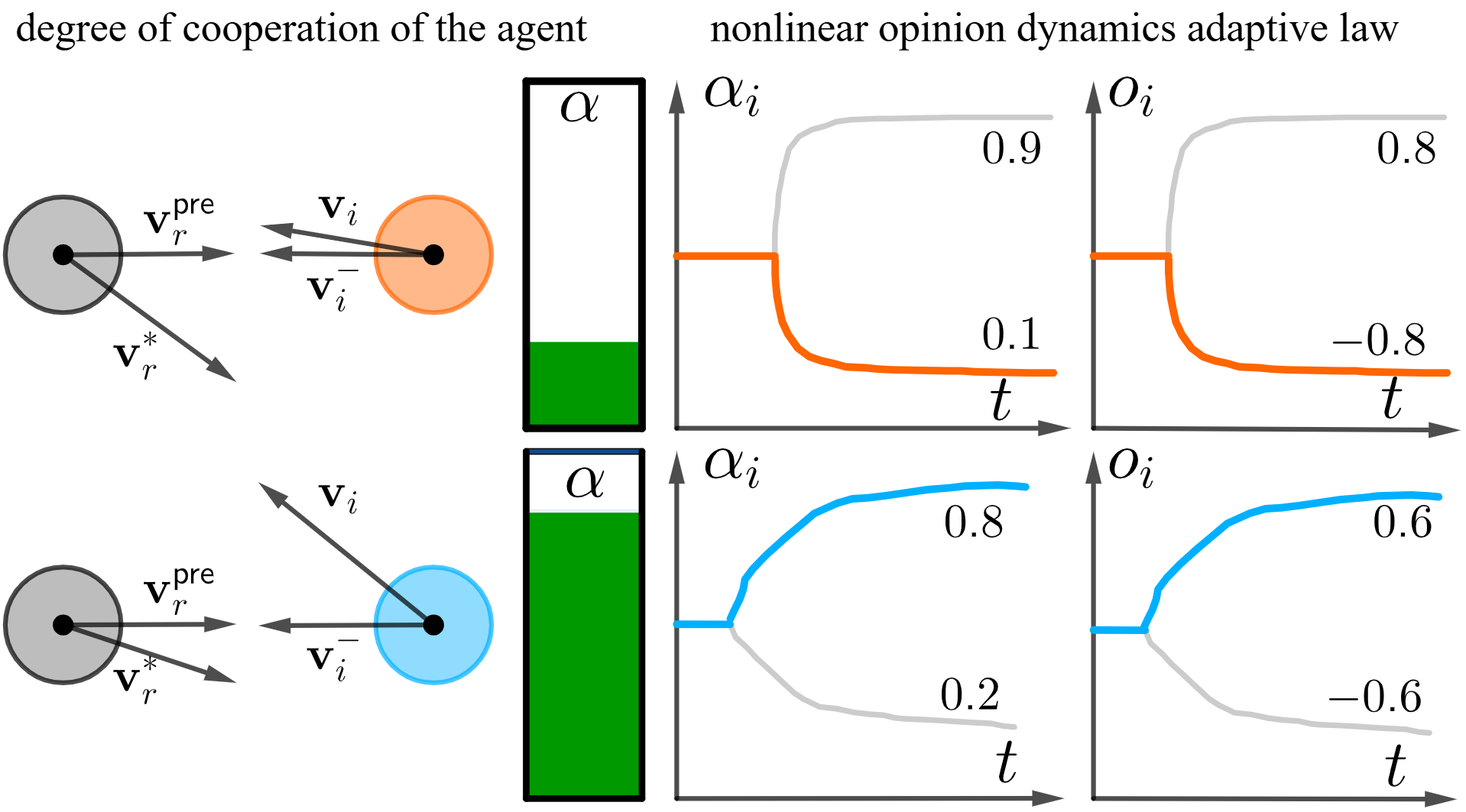}
    \caption{Evolution of the nonlinear opinion dynamics adaptive law for a non-cooperative (top, orange) or a cooperative (bottom, blue) agent.}
    \label{fig:NOD}
\end{figure}

There are two variables in Eq. \eqref{eq:NODAL} that are not readily available at the robot. The first one is the attention, which should capture the importance of adapting the shifted degree of cooperation or if, otherwise, the shifted degree of cooperation must not change. The second quantity is ${e}_i$, which a priori depends on the information at agent $i$. We will later prove how the robot can estimate ${e}_i$ from just onboard sensor observations. Now, we focus on the attention $A_i$. 

The goal is to design an attention mechanism for the nonlinear opinion dynamics adaptive law that evolves such that the shifted degree of cooperation changes when the robot and agent $i$ are close to each other, whereas it remains constant if the robot and the agent are far. The attention $A_i$ is modeled as a dynamic quantity that evolves with time
\begin{equation}
\label{eq:u_dynamics}
\dot{A}_{i} = 
-\delta_i A_{i} +  (1-\delta_i)\mathsf{tanh}(\kappa_i \tau_{i}^{-1}),
\end{equation}
where $\kappa_i > 0$ and $\delta_i \in [0, 1)$ are gains and $\tau_{i}>0$ is the expected collision time. The dynamics in \eqref{eq:u_dynamics} evolve such that $A_i=0$ when the robot and agent $i$ are far from each other, and $A_i$ tends to $1$ when the robot and agent $i$ are close to collision. To compute the time to collision, we exploit the geometry of the problem.

A potential collision between the robot and agent $i$ happens at the intersection between (i) a circumference of radius \mbox{$R= r_s + r_i$} and center $\mathbf{p}_i$ and (ii) the line given by $\mathbf{v}_r^{\mathsf{pre}} - \mathbf{v}_r$ passing through point $\mathbf{p}_r$. Let $\mathbf{p}_r = (p^x_r,p^y_r)$ and \mbox{$\mathbf{p}_i = (p^x_i,p^y_i)$}. The equation of the circumference with center $\mathbf{p}_i$ and radius $R$ is
\begin{equation}\label{eq:tau_1}
    (p^x-p^x_i)^2 + (p^y-p^y_i)^2 = R^2,
\end{equation}
where the circumference is parameterized by $\mathbf{p} = (p^x, p^y)$. Let $\Tilde{\mathbf{v}}_{i}=\mathbf{v}_r^{\mathsf{pre}} - \mathbf{v}_i = (\Tilde{v}^x_{i}, \Tilde{v}^y_{i})$. Then, the intersection between (i) the circumference defined in \eqref{eq:tau_1} and (ii) a line of slope given by vector $\Tilde{\mathbf{v}}_{i}$ and point $\mathbf{p}_r$ happens at:
\begin{equation}
\begin{aligned}\label{eq:tau2}
    (p^x_r + \tau_{i} \Tilde{v}^x_{i}-p^x_i)^2 + (p^y_r + \tau_{i} \Tilde{v}^y_{i}-p^y_i)^2 = R^2. 
\end{aligned}
\end{equation}
The above equation leads to the second order equation
\begin{equation}\label{eq:tau3}
\beta_1\tau^2_i + \beta_2\tau_i + \beta_3 = 0
\end{equation}
with $\beta_1 \kern -0.1cm = \kern -0.1cm   (\Tilde{v}^x_{i})^2 \kern -0.1cm + \kern -0.1cm (\Tilde{v}^y_{i})^2$, $\beta_2 \kern -0.1cm = \kern -0.1cm 2 \Tilde{v}^x_{i} (p^x_r - p^x_i) \kern -0.1cm +\kern -0.1cm  2 \Tilde{v}^y_{i} (p^y_r - p^y_i)$ and $\beta_3 = R^2 - (p^x_r - p^x_i)^2 - (p^y_r - p^y_i)^2$.
The solutions of the quadratic equation can be the following:
\begin{enumerate}
    \item if $\beta_2^2 < 4\beta_1\beta_3$, then there is no intersection and, therefore, no potential collision. As a consequence, we set $\tau_{i} = \infty$.
    \item if $\beta_2^2 = 4\beta_1\beta_3$, then the solution is given by a double root $\tau_i = \frac{-\beta_2}{2\beta_1}$. Otherwise, if $\beta_2^2 > 4\beta_1\beta_3$, then there are two real roots. In both cases, the following reasoning holds: (i) if both roots are positive, then we take the minimum among them; (ii) if both roots are negative, then there is no collision and $\tau_{i} = \infty$; (iii) if one root is positive and one root is negative, then collision has already happened and, therefore, motion is terminated.
    \end{enumerate}
Note that $\tau_i$ is inexpensive to compute and enables time-to-collision checking in real-time.

The next step is to define ${e}_i$ and develop a method to compute it from sensor observations. According to the definition of the nonlinear opinion dynamics in Eq. \eqref{eq:NODAL}, ${e}_i$ represents the opinion that agent $i$ has on the value of the shifted degree of cooperation ${o}_i$. Recalling the reasoning behind the nonlinear opinion dynamics design in \eqref{eq:opinion_original}, the consensus term $a_i {o}_i + c_i {e}_i$ aims at fusing the values that the robot and agent $i$ have on the degree of cooperation ${o}_i$. In the case of the robot, this is simply the current value of ${o}_i$; in the case of agent $i$, this is precisely the opinion that agent $i$ has on the value of the shifted degree of cooperation ${o}_i$. In typical opinion dynamics formulations it is assumed that these values can be communicated. This is not the case in this work, so we need to recover ${e}_i$ using an estimation method that relies solely on onboard sensing. 

Figure \ref{fig:projection} depicts the geometry of the problem. At a certain instant, the robot has a predefined velocity $\mathbf{v}_r^{\mathsf{pre}}$ but executes $\mathbf{v}_r^*$ to avoid collision. The difference between both velocities is $\mathbf{v}_r^* - \mathbf{v}_r^{\mathsf{pre}} = (1-\alpha_i) \mathbf{u}_i$ due to  constraint \eqref{eq:constraint1} in   \eqref{eq:optimization}. 
\begin{figure}
    \centering
    \includegraphics[width=\columnwidth]{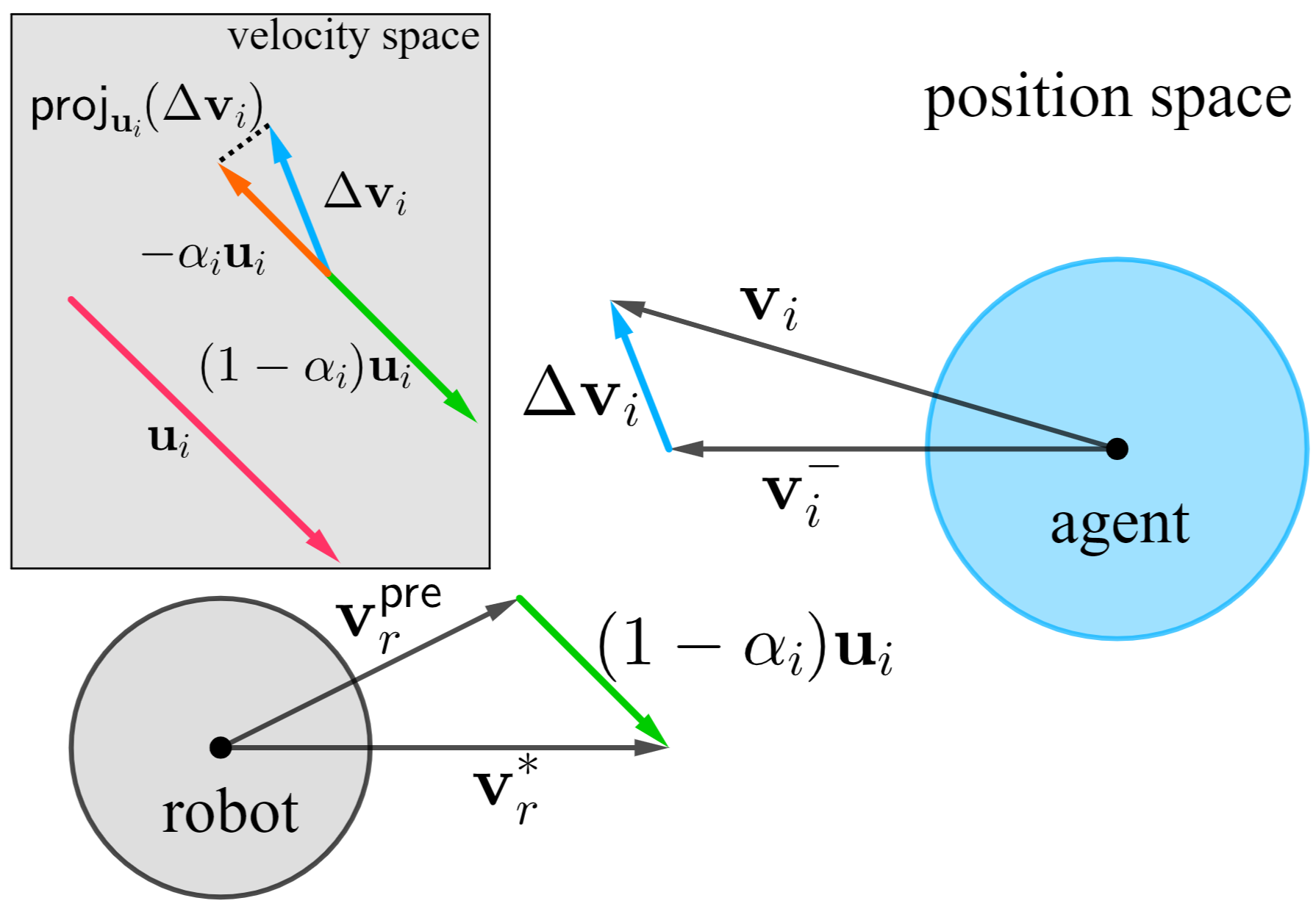}
    \caption{Geometry behind the projection estimator for ${e}_i$.}
    \label{fig:projection}
\end{figure}
On the other hand, agent $i$ executes $\mathbf{v}_i$, and executed $\mathbf{v}^{-}_i$ in the previous instant, where $\mathbf{v}_i^{-} \in \mathbb{R}^2$ denotes the velocity of agent $i$ sensed by the robot at the previous time instant. Since the robot has been in charge of exerting $(1-\alpha_i) \mathbf{u}_i$, the remaining part of $\mathbf{u}_i$ to avoid collision must be exerted by agent $i$, i.e., $\alpha_i\mathbf{u}_i$. Recall that $\mathbf{u}_i$ represents a change in the velocity. Henceforth, the remaining part of $\alpha_i\mathbf{u}_i$ is exerted through the change in velocity of agent $i$, namely, \mbox{$\Delta \mathbf{v}_i = \mathbf{v}_i - \mathbf{v}^{-}_i$}. However, in general, $\Delta \mathbf{v}_i \cdot \alpha_i \mathbf{u}_i \neq 0$, i.e., vectors $\Delta \mathbf{v}_i$ and $\alpha_i \mathbf{u}_i$ are not aligned. This means that it is the component of $\Delta \mathbf{v}_i$ parallel to $\alpha_i \mathbf{u}_i$ the one that exerts the change in velocities to avoid collision. Therefore, if we compute the magnitude of vector component of $\Delta \mathbf{v}_i$ parallel to $\mathbf{u}_i$ and compare it to the magnitude of $\mathbf{u}_i$, then we can obtain $\alpha_i$.

These insights motivate the following projection estimator of ${e}_i$:
\begin{equation}\label{eq:projection1}
    {e}_i = \mathsf{tanh} \left(\varepsilon\left(\frac{||\mathsf{proj}(\Delta\mathbf{v}_i, \mathbf{u}_{i})||}{||\mathbf{u}_{i}||} - \frac{1}{2}\right)\right). 
%     =
%     \\
%     &\tanh \left(\varepsilon\left(\frac{1}{2} - \frac{\left|\left|\frac{\Delta\mathbf{v}_i \cdot \mathbf{u}_{i}}{\mathbf{u}_{i} \cdot \mathbf{u}_{i}} \mathbf{u}_{i}\right|\right|}{||\mathbf{u}_{i}||}\right)\right)
% \end{aligned}.
\end{equation}
The projection operator $\mathsf{proj}(\mathbf{a}, \mathbf{b}) = \frac{\mathbf{a} \cdot \mathbf{b}}{\mathbf{b} \cdot \mathbf{b}} \mathbf{b}$ projects the vector $\Delta\mathbf{v}_i$ on vector $\mathbf{u}_i$. Furthermore, by normalizing with the norm $||\mathbf{u}_i||$, the estimate of $\alpha_i$ is extracted. The other operations translate the estimated degree of cooperation to the estimated shifted degree of cooperation ${e}_i$, where $\varepsilon>0$ tunes how sensitive is the $\mathsf{tanh}$ function to changes in $\Delta\mathbf{v}_i$. In this sense, we remark that Eq. \eqref{eq:projection1} is not a perfect estimator since it only relies on the sensor observations of the change in velocities $\Delta \mathbf{v}_i$. This motivates the use of the function $\mathsf{tanh}$, which simultaneously restricts the domain of the estimation to $[-1, 1]$ and allows a fast evolution of the estimated shifted degree of cooperation. For instance, if the robot observes that agent $i$ moves straight to it and $\Delta \mathbf{v}_i \approx 0$, then this means that ${e}_i \approx -1$ and the robot estimates that agent $i$ is not cooperating. It is worth noting that the projection estimator in \eqref{eq:projection1} directly uses the velocity measurements from onboard sensors, which are subject to disturbances. %, in contrast to filtered methods that model-based estimators such as the Kalman filter. Despite adding robustness against disturbances, the use of model-based methods requires known dynamics of the other agents, which is very hard to obtain in general settings as discussed in Section \ref{sec:related}. On the other hand, by 
However, designing the adaptive law as a non-linear opinion dynamic naturally provides robustness against small perturbations, and exploits past information in the opinion estimation by means of the memory term $-d_i {o}_i$ \cite{leonard2024fast}. Therefore, the proposed adaptive law can be framed as a nonlinear estimator with a dynamic model (Eq. \eqref{eq:NODAL}) and a measurement model (Eq. \eqref{eq:projection1}). The sensitivity and robustness of the general formulation of the nonlinear opinion dynamics model is analyzed in previous works \cite{bizyaeva2022nonlinear}. The next section analyzes specific considerations on how to tune the parameters of \textsf{AVOCADO}.

% By substituting $y_i$ by its estimate, the nonlinear opinion dynamics adaptive law is as follows:
% \begin{equation}\label{eq:NODAL_v2}
%     \dot{x}_{i} = -d_{i} x_{i} + A_i\mathsf{tanh}\left(a_{i} x_{i} + c_{i} {e}_i\right) + b_{i}.
% \end{equation}

\subsection{Nonlinear Opinion Dynamics Adaptive Law Tuning}\label{subsec:tuning}

To achieve fast and effective adaptation to the unknown degree of cooperation of the agent, the nonlinear opinion dynamics adaptive law of \textsf{AVOCADO} must be appropriately parameterized. 

First, we examine the stability properties of the adaptive law and the conditions under which bifurcation, i.e., instability of the initially stable equilibrium point of Eq. \eqref{eq:NODAL}, happens. To do so, we use the expression $\text{sech}(\bullet) = 1/\cosh (\bullet)$.
\begin{proposition}\label{prop:stability}
Let $d_i>0$, $a_i>0$ and $b_i \in [-1, 1]$. Then, ${o}_i = b_i/d_i$ is an unstable equilibrium point of the nonlinear opinion dynamics adaptive law in \eqref{eq:NODAL} if \mbox{$A_i>1/\left(a_i\text{sech}^2 \left(a_i \frac{b_i}{d_i}\right)\right)$}.  
\end{proposition}
\begin{proof}
Initially, $A_i=0$, ${e}_i=0$ and $\dot{{o}}_i=0$. Therefore, the robot has an initial opinion ${o}_i=b_i/d_i$ that is an equilibrium of the nonlinear opinion dynamics adaptive law in Eq. \eqref{eq:NODAL}. Nevertheless, whether this equilibrium is stable or unstable depends on the parameters of the adaptive law. Let
% \begin{equation}\label{eq:linearization}
%    \lambda = -d_i + A_i a_i \text{sech} \left(a_i \frac{b_i}{d_i}\right)
% \end{equation}

% ESTA!!!
\begin{equation}\label{eq:linearization}
   \lambda = -d_i + d_i A_i a_i \text{sech}^2 \left(a_i \frac{b_i}{d_i}\right)
\end{equation}
be the unique eigenvalue of \eqref{eq:NODAL} after linearization at \mbox{${o}_i=b_i/d_i$}. Equilibrium becomes unstable when $\lambda>0$. Assuming that $a_i$, $d_i$ and $b_i$ are fixed, and taking into account that $\text{sech}(\bullet) \in (0, 1]$ in all its domain, ${o}_i=b_i/d_i$ becomes an unstable equilibrium when $A_i>1/\left(a_i\text{sech}^2 \left(a_i \frac{b_i}{d_i}\right)\right)$.
\end{proof}
\begin{corollary}\label{coro:stability}
    If $b_i=0$, then $\text{sech}^2 \left(a_i \frac{b_i}{d_i}\right) = 1$ and the condition on the attention level reduces to $A_i > 1/a_i$. Since $\text{sech}^2 \left(a_i \frac{b_i}{d_i}\right) \in (0, 1)$, $A_i > 1/a_i$ is an upper bound on the attention value beyond which the equilibrium ${o}_i = b_i/d_i$ becomes an unstable equilibrium and, therefore, $A_i > 1/a_i$ is a sufficient condition for ${o}_i = b_i/d_i$ being an unstable equilibrium.
\end{corollary}
When the robot and agent $i$ are far from each other, the dynamics of the attention in \eqref{eq:u_dynamics} are designed such that $A_i=0$. When the robot and agent $i$ approach each other in a trajectory of potential collisions, $\mathsf{tanh}(\kappa_i\tau_i^{-1}) > 0$ and attention evolves towards $A_i>0$, with a speed determined by $\delta_i$. Since the attention dynamics in \eqref{eq:u_dynamics} is a linear system with bounded input, the attention is stable with equilibrium given by \mbox{$A_i^{\mathsf{eq}} = \mathsf{tanh}(\kappa_i(\tau_i^{\mathsf{eq}})^{-1})$}. Therefore, $\kappa_i$ must be large enough to ensure that $A_i>\frac{1}{a_i}$ when the time to collision $\tau_i$ is large enough to avoid collision. Moreover, since $A_i \in [0,1]$, $a_i$ must be designed such that $\frac{1}{a_i} < 1$ to ensure that the condition $A_i>\frac{1}{a_i}$ is possible for some $\tau_i$.   

Now, we examine the stability properties of the nonlinear opinion dynamics adaptive law in Eq. \eqref{eq:NODAL} after bifurcation.
\begin{proposition}\label{prop:equilibriums}
    Let, $d_i>0$, $a_i>0$ and $b_i \in [-1, 1]$. Assume that $A_i>\frac{1}{a_i}$ holds. Then, the nonlinear opinion dynamics adaptive law in Eq. \eqref{eq:NODAL} exhibits two stable equilibrium points, given by the following expression:
    \begin{equation}\label{eq:equilibriums}
      {o}_i^{\mathsf{eq}} = A_i\mathsf{tanh}(a_i {o}_i^{\mathsf{eq}} + c_i {e}_i) + \frac{b_i}{d_i}.  
    \end{equation}
    Besides, $|{o}_i^{\mathsf{eq}}| \in \left[\frac{|b_i|}{d_i} - 1, \frac{|b_i|}{d_i} + 1\right]$.
\end{proposition}
\begin{proof}
From the assumption that $A_i>\frac{1}{a_i}$ holds, in the transition   $A_i=\frac{1}{a_i}$ bifurcation happens \cite{bizyaeva2022nonlinear}.
The dynamics in Eq, \eqref{eq:NODAL} are bounded due to the term $-d_i {o}_i$, so equilibrium(s) point(s) exist(s) after bifurcation. By enforcing the equilibrium condition $\dot{{o}}_i=0$ in Eq. \eqref{eq:NODAL}, Eq. \eqref{eq:equilibriums} is obtained.
We recall that $A_i \in [0, 1]$ and $|\mathsf{tanh}(\bullet)| \in [0, 1]$. Therefore, the absolute value of the equilibrium $|{o}_i^{\mathsf{eq}}|$ is bounded and such that $|{o}_i^{\mathsf{eq}}| \in \left[\frac{|b_i|}{d_i} - 1, \frac{|b_i|}{d_i} + 1 \right]$.
\end{proof}
If there is no bias, $b_i = 0$ and 
${o}_i^{\mathsf{eq}} \in \left[- 1, 1 \right]$; otherwise, the bias introduces a shift in the domain that naturally encodes a priori knowledge on the degree of cooperation of the agent. Proposition \ref{prop:equilibriums} motivates the use of the scaling factor $d_i$ in the attention $A_i$: in the absence of this factor, the expression of the equilibrium points is ${o}_i^{\mathsf{eq}} = \frac{A_i\mathsf{tanh}(a_i {o}_i^{\mathsf{eq}} + c_i {e}_i) + b_i}{d_i}$, leading to an asymmetrical domain for $|{o}_i^{\mathsf{eq}}|$ with respect to $b_i/d_i$, i.e., not symmetrical with respect to the initial equilibrium before bifurcation. This asymmetry would derive in an intrinsic bias towards any of the two subsequent equilibrium points, which is an undesired effect that we overcome by including the scaling factor $d_i$.

Finally, $c_i$ shall be designed to guarantee that the nonlinear term $\mathsf{tanh}(\bullet)$ in the adaptive law is enough sensitive to the behavior of agent $i$, reacting fast enough to avoid collisions. On the other hand, $\varepsilon$ in~\eqref{eq:projection1} tunes how sensitive is the estimator of ${e}_i$ to changes in the speed of agent $i$. 

In summary, the nonlinear opinion dynamics adaptive law depends on six parameters: (i) $\varepsilon>0$ determines how sensitive is the estimator of ${e}_i$ and should be chosen large enough to ensure ${e}_i \approx 1$ with enough time to avoid collision but small enough to avoid noisy estimations; (ii) $\kappa_i > 0$ and $\delta_i \in [0, 1)$ determines how fast the adaptive law approaches the bifurcation point, and therefore establishes when the robot decides its opinion about the degree of cooperation of agent $i$; (iii) $d_i > 0$ sets the convergence speed of the dynamics of the adaptive law, associated to the forgetting factor $-d_i {o}_i$; (iv) $a_i$ determines the attention level necessary to reach the bifurcation, taking into account the value of $d_i$; (v) $b_i$ is a bias that encodes potential prior knowledge on the (shifted) degree of cooperation of agent $i$; (vi) $c_i$ weights the importance of the estimate ${e}_i$ compared to the current value of $a_i {o}_i$ such that, if $a_i >> c_i$, the robot tends to preserve its opinion over the (shifted) degree of cooperation, or heavily relies on the estimator if $c_i >> a_i$.

\subsection{Exploiting attention to avoid symmetry deadlocks}\label{subsec:symmetries}

The previous sections address the problem of unknown degree of cooperation in collision avoidance. However, there is an additional issue, typical in \textsf{VO}-based methods like \textsf{ORCA}, that is characterized by navigation deadlocks under symmetrical configurations \cite{battisti2020velocity}. Among other mechanisms, \textsf{VO}-based methods usually solve symmetry deadlocks by injecting some arbitrary noise in the perceived velocity, where this amount is typically handcrafted. Instead, we propose an adaptive noise injection mechanism that exploits the attention $A_i$ of the opinion dynamics to correlate the amount of noise with safety.

Specifically, our solution consists in injecting some small noise to 
the perceived velocity of agent $i$ when the agent is far from the robot, decreasing the noise to zero when the robot and the agent are close to a potential collision. When the robot and the agent are far from each other, the impact of noise is negligible because attention $A_i \approx 0$ and bifurcation does not happen. This preserves the neutrality of the opinion on the (shifted) degree of cooperation. When the robot and the agent are close to collision, noise is removed, so the performance guarantees on the adaptive law and the collision avoidance are preserved.

Mathematically speaking, we reformulate Eq. \eqref{eq:u} as 
\begin{equation}\label{eq:u_v2}
\mathbf{u}_{i} = \underset{\mathbf{v}_r\in\partial\mathsf{VO}_{i}}{\arg\min} || \mathbf{v}_r-(\mathbf{v}_r^{\mathsf{pre}} - \mathbf{v}_i^\mathsf{\mu})||-(\mathbf{v}_r^{\mathsf{pre}} - \mathbf{v}_i^\mathsf{\mu}),
\end{equation}
where
\begin{equation}\label{eq:noise}
    \mathbf{v}_i^\mathsf{\mu} = \mathbf{v}_i + (1-A_i)\mu(\sigma),
\end{equation}
with $\mu(\sigma) \sim \mathcal{U}(-\sigma,\sigma)$ a uniformly distributed perturbation bounded by $\sigma > 0$.

In essence, our method adds a uniformly distributed noise to the velocity of agent $i$ sensed by the robot. This noise depends on the attention mechanism developed for the nonlinear opinion dynamics adaptive law. This perturbed velocity is the one used to compute vector $\mathbf{u}_i$ and the $\textsf{OCA}_i^{\tau}$ constraint sets. Hence, the noise perturbs the potential symmetry between $\mathbf{v}^{\mathsf{pre}}_r$ and $\mathbf{v}_i$ only when the robot and agent $i$ are sufficiently close to each other. The intensity of the perturbation increases when the distance between robot and agent $i$ decreases in order to ensure that there is an effective deadlock breaking when robot and agent $i$ approach each other. It is interesting to remark that this deadlock breaking approach can also be applied to other \textsf{VO}-based methods such as \textsf{ORCA}, since it only requires the implementation of the attention mechanism in Eq. \eqref{eq:u_dynamics}.

\textsf{AVOCADO} is summarized in Algorithm \ref{al:avocado}. By combining a geometrical approach with a nonlinear opinion dynamics adaptive law, \textsf{AVOCADO} is able to adapt in real-time to the unknown degree of cooperation of the agents in a multi-agent setting. The unknown degree of cooperation of the agent induces an additional layer of complexity in collision avoidance, since, as it is shown in Section \ref{sec:simulations}, an incorrect assumption on the degree of cooperation leads to collision to existing geometrical, learning and model predictive approaches.
On the other hand, \textsf{AVOCADO} is agnostic to the nature of the agents the robot encounters, so the robot can take a fast and flexible decision irrespective of the dynamic model of the agent.
\textsf{AVOCADO} explicitly couples prediction with planning by estimating the effect its motion is producing in the other agents it encounters through the degree of cooperation $\alpha_i$ $\forall i \in \mathcal{N}$. In other words, the degree of cooperation $\alpha_i$, unless $0$, expresses the fact that the motion of agent $i$ is influenced by the velocity exerted by the robot. By designing an adaptive law that estimates $\alpha_i$  in the continuous domain  $\alpha_i \in [0,1]$, our method is able to deconflict the undesirable effects of such coupling by predicting how the agent reacts to the robot motion and plan the collision avoidance maneuver accordingly. No communication is involved, only relying on the perceived position and velocity of the nearby agents. All the mathematical operations are inexpensive to compute, except for the optimization program in line 11. Nonetheless, since it entails a linear program, with an efficient implementation it is proven that a robot can process thousands of nearby agents in milliseconds \cite{van2011reciprocal}. We validate the computational simplicity of \textsf{AVOCADO} in Section \ref{sec:simulations}. 

\begin{algorithm}
\caption{\textsf{AVOCADO}}\label{al:avocado}
\begin{algorithmic}[1]
\STATE \textbf{Parameters}: $r_s, r_p>0$, $\varepsilon>0$, $\kappa_i > 0$, $d_i > 0 $, $a_i, b_i, c_i \in \mathbb{R}$, $\sigma>0$, $\delta_i \in [0, 1)$
\FOR{all $t$}
    \STATE Get $\mathbf{p}_r$ and $\mathbf{v}_r^{\mathsf{pre}}$ from a higher-level planner.
    \STATE Measure $\mathbf{p}_i, \mathbf{v}_i \quad \forall i \in \mathcal{N}$ from sensors and get the stored $\mathbf{v}_i^{-} \quad \forall i \in \mathcal{N}$.
    \STATE Update the attention level $A_i$ using $\mathbf{p}_r$, $\mathbf{v}_r^{\mathsf{pre}}$, $\mathbf{p}_i$ and $\mathbf{v}_i$ through Eqs. \eqref{eq:u_dynamics} and \eqref{eq:tau3}, for all $i \in \mathcal{N}$.
    \STATE Apply Eq. \eqref{eq:noise} to compute $\mathbf{v}_i^{\mathsf{\mu}} \quad \forall i \in \mathcal{N}$.
    \STATE Calculate $\mathbf{u}_i$ through Eq. \eqref{eq:u_v2}.
    \STATE Estimate ${e}_i$ through Eq. \eqref{eq:projection1}.
    \STATE Update the opinion on the shifted degree of cooperation using the adaptive law in Eq. \eqref{eq:NODAL}.
    \STATE Build all admissible sets $\mathsf{OCA}_i$ through Eq. \eqref{eq:OCA_set}, for all $i \in \mathcal{N}$.
    \STATE Solve optimization problem \eqref{eq:optimization} to obtain $\mathbf{v}_r^*$.
    \STATE Apply $\mathbf{v}_r^*$ and store $\mathbf{v}_i$ as $\mathbf{v}_i^{-} \quad \forall i \in \mathcal{N}$.
\ENDFOR 
\end{algorithmic}
\end{algorithm}

%%%%%%%%%%%%%%%%           
%% SIMULATION %%
%%%%%%%%%%%%%%%%

\section{Simulated results}\label{sec:simulations}
First, we evaluate \textsf{AVOCADO} in simulated scenarios. In Section \ref{sec:experiments} we describe the results obtained in real settings with robots and humans.

In Section \ref{subsec:eval_AVOCADO} we analyze the impact of the different parameters in the behavior of $\textsf{AVOCADO}$. Next, in Section \ref{subsec:multiagent} we conduct extensive multi-agent simulations using two navigation settings. For each setting, we compare qualitatively and quantitatively $\textsf{AVOCADO}$ with existing state-of-the-art planners against other cooperative robots and non-cooperative agents. In all the simulated scenarios, we consider that robots are cooperative, i.e., they act using the same motion planner under comparison. For instance, if \textsf{AVOCADO} is assessed, then all robots use \textsf{AVOCADO}. Importantly, in our paper, a robot being cooperative does not mean knowledge on the degree of cooperation nor communication exchange, but just the robot uses a planner to avoid collisions.
In contrast, agents are non-cooperative, i.e., they are blind against the robots. Non-cooperative agents resolve \eqref{eq:optimization}, but their neighbor set $\mathcal{N}_i$ never includes the robots, only other agents. This is done to ensure collision avoidance among agents but complete non-cooperation with the robots, so the behavior is more complex than a simple dynamic obstacle with a fixed trajectory. A mixed cooperative/non-cooperative navigation scenario involves cooperative robots and non-cooperative agents. Finally, robots and agents have a disc shape of radius $r_s=0.2$m and a sensor range of $r_p=2.5$m.

The maximum velocity, $\mathbf{v}_r^{\mathsf{max}}$, of the robots is set to $1$m/s and $\mathbf{v}_i^{\mathsf{max}}$ of the agents is set to $0.75$m/s. The agents have a lower velocity than the robots, allowing the robots to escape if the agents move quickly towards them. For robots using \textsf{AVOCADO}, $\mathbf{v}_r^{\mathsf{pre}}=\mathbf{v}_r^{\mathsf{max}}\frac{\mathbf{p}_{r,j}^* - \mathbf{p}_{r,j}^t}{||\mathbf{p}_{r,j}^* - \mathbf{p}_{r,j}^t||}$, where $\mathbf{p}_{r,j}^*$ denotes the desired goal for robot $j$ and $\mathbf{p}_{r,j}^t$ is the position of robot $j$ at time $t$. Therefore, robots using \textsf{AVOCADO} are not helped by any higher-level planner to choose the desired velocity.

\subsection{Head-on scenarios}\label{subsec:eval_AVOCADO}

\begin{figure}
 \centering
 \begin{tabular}{@{}cc@{}}
      
      \footnotesize{a) $b_i/d_i=-0.9$} & \footnotesize{b) $b_i/d_i=-0.9$} \\
      \includegraphics[width=0.45\linewidth]{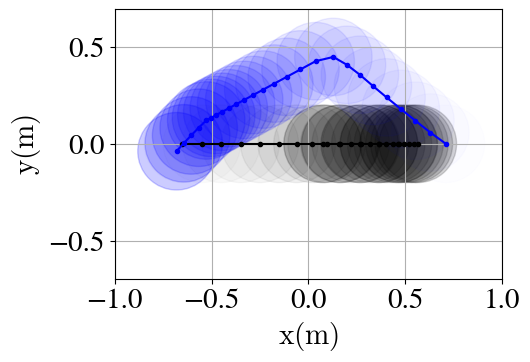} & 
      \includegraphics[width=0.45\linewidth]{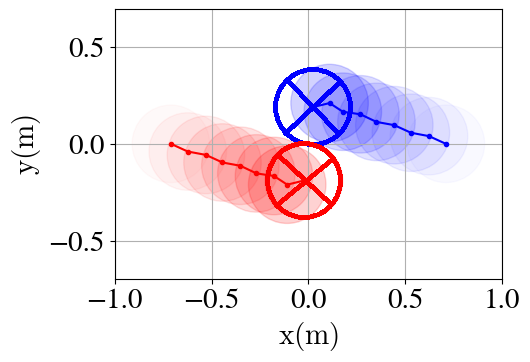}\\
      \footnotesize{c) $b_i/d_i=0$} & \footnotesize{d) $b_i/d_i=0$}\\
     \includegraphics[width=0.45\linewidth]{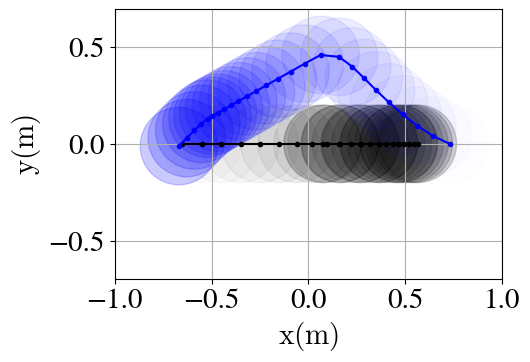} & 
      \includegraphics[width=0.45\linewidth]{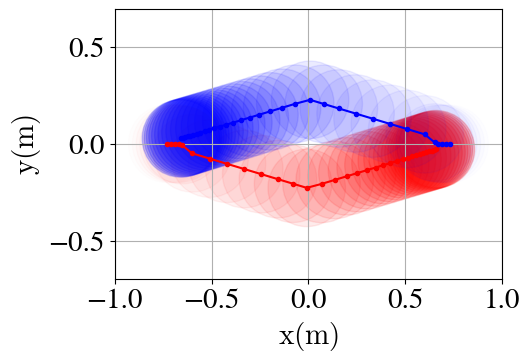}\\
      \footnotesize{e) $b_i/d_i=0.9$} & \footnotesize{f) $b_i/d_i=0.9$} \\
     \includegraphics[width=0.45\linewidth]{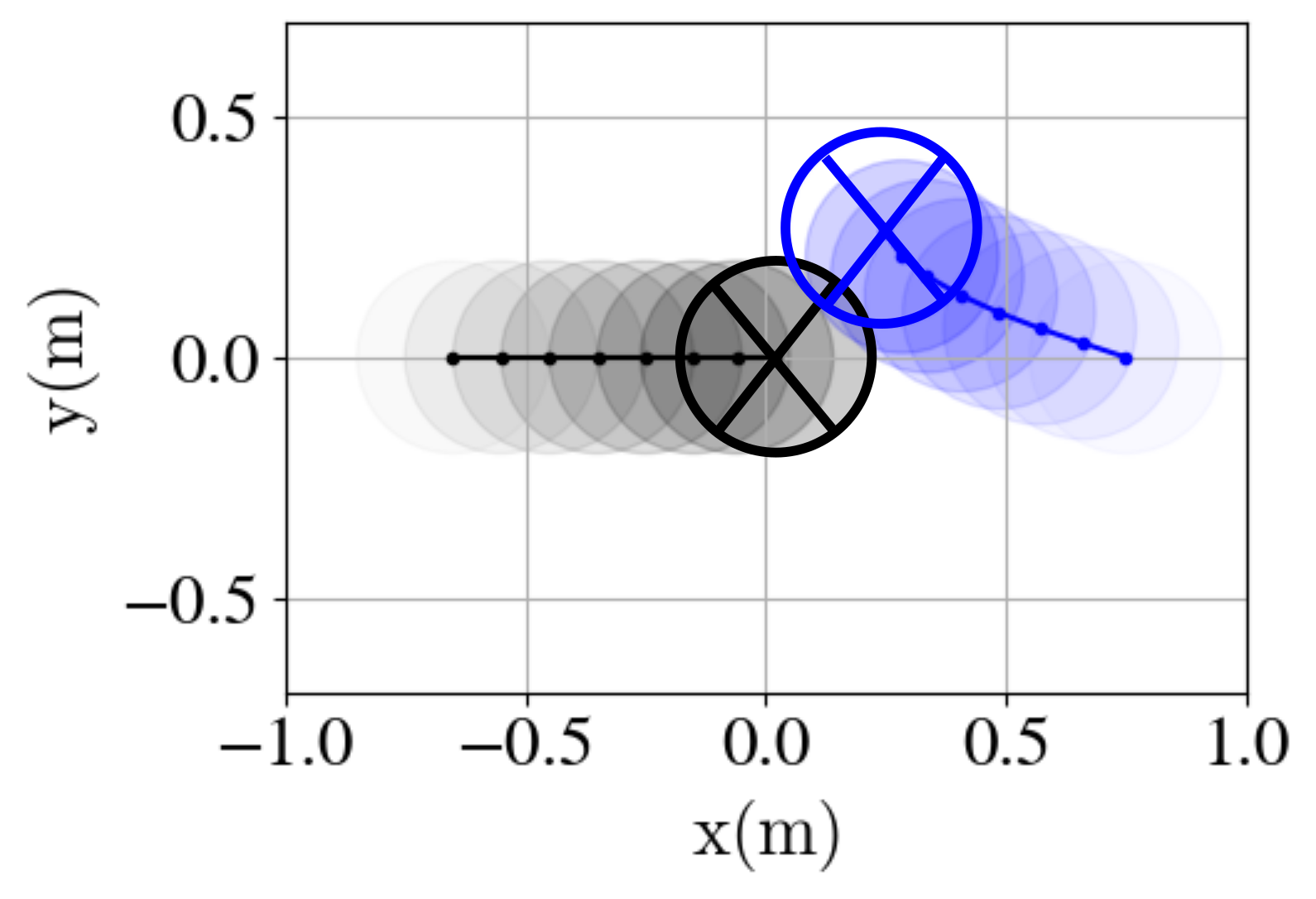} & 
      \includegraphics[width=0.45\linewidth]{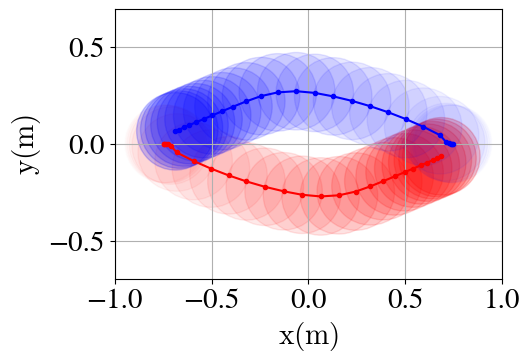}\\
 \end{tabular}
 \caption{Head-on scenarios for different ratio $b_i/d_i$ and degree of cooperation. The robot (blue) navigates using \textsf{AVOCADO}, whereas the agent is either non-cooperative (black, left column) or navigates using \textsf{AVOCADO} (red, right column). 
 The motion of the robot and the agent is represented with disks of increasing transparency as time advances, and solid dots represent the center of the disks. Collision is marked with crosses surrounded by circles}.
 \label{fig:bias}
 \end{figure}
 
We simulate head-on scenarios to analyze the impact of the different parameters in \textsf{AVOCADO}. 
A robot and either a cooperative robot or a non-cooperative agent face each other, where their goals are the starting position of the other. Fig.~\ref{fig:bias} illustrates these scenarios for different values of bias $b_i$ (a, c, e for the non-cooperative agent; b, d, f for the cooperative robot). The robot using \textsf{AVOCADO} starts, in scenarios a) and b), assuming that the other robot or agent is non-cooperative with a low degree of cooperation; in scenarios c) and d), the robot initially assumes reciprocity, with a degree of cooperation equal to $0.5$; and in e) and f) the robot initially assumes a great degree of cooperation. Extreme bias values can lead to collisions, as in b) and e). In a), the robot effectively avoids the collision with the agent, as the highly negative biased initial degree of cooperation (pessimistic robot) corresponds to reality. However, when the bias is incorrect, even if the agent is cooperative as in b), the robot takes some time to recover from the incorrect prior, resulting in collision. The opposite conflict arises when there is an initial highly positive bias in the degree of cooperation (optimistic). In case e), the robot is unable to correct the estimated degree of cooperation fast enough to realize that the agent is non-cooperative and take all the responsibility to avoid collision. Collision is avoided in case f), but the extreme bias leads to sub-optimal trajectories. Finally, a balanced bias allows the robot to react effectively to non-cooperative agents in c) and cooperative agents in d).

The impact of the other parameters is evaluated by studying the evolution of ${o}_i$ when the parameters change. The default parameters are specified in Table~\ref{tab:parameter-values}. Now, we focus on the non-cooperative head-on scenario.

\begin{table}[ht]
    \centering
    \begin{tabular}{c|c|c|c|c|c|c|c}
        Parameter & $a_i$ &
         $b_i$ & 
         $d_i$ & 
         $\kappa_i$ & 
         $\varepsilon$ &
         $\delta_i$ & 
         $b_i$ 
         \\
         \hline
         Value & 0.3 & 0.7 & 2 & 14.15 & 3.22 & 0.57 & 0 
    \end{tabular}
    \caption{Default parameter values}
    \label{tab:parameter-values}
\end{table}

Fig.~\ref{fig:alphas} depicts the evolution of ${o}_i$ for the different cases. As the robot detects that agent $i$ does not cooperate, ${o}_i$ decreases, trying to reach ${o}_i=-1$. Finally, when the collision is avoided and $A_i=0$, ${o}_i$ recovers to the initial value. In some cases, as in Fig.~\ref{fig:alphas}e) with $\delta=0.1$, ${o}_i$ does not recover to the initial value because agent $i$ leaves the sensor range of the robot before letting the nonlinear opinion dynamics to converge to ${o}_i=b_i/d_i$. First, regarding $d_i$ (Eq. \eqref{eq:NODAL}), greater values imply a greater forgetting factor and therefore, a faster convergence. For very large values ($d_i=6.5$), the nonlinear opinion dynamics are too reactive to the estimate ${e}_i$ and to the noise introduced (Eq.~\ref{eq:noise}), leading to abrupt changes in ${o}_i$. Second, regarding $a_i, c_i$ (Eq. \eqref{eq:NODAL}), since $c_i = 1- a_i$ and $a_i \in \{0.1, 0.3, 0.6, 0.9\}$, greater values of $a_i$ imply less reactivity against changes in the estimate ${e}_i$ and, henceforth, slowest convergence of the nonlinear opinion dynamics. Third, regarding $\varepsilon$ (Eq. \eqref{eq:projection1}), it is observed that beyond some value around $\varepsilon = 1$, the estimate of ${e}_i$ is reactive enough to guarantee a fast convergence of the nonlinear opinion dynamics. Fourth, the behavior of $\kappa_i$ (Eq. \eqref{eq:u_dynamics}) is similar to that of $\varepsilon$, in the sense that there is a point beyond which increasing $\kappa$ does not provide any improvement.  Fifth, the same holds for $\delta_i$ (Eq. \eqref{eq:u_dynamics}). Finally, regarding the ratio $b_i/d_i$ (Eq. \eqref{eq:NODAL}), it is very interesting to observe that different values lead to different initial and final equilibrium, equal to ${o}_i = b_i/d_i$. In all the cases, since the agent is always non-cooperative, ${o}_i$ evolves towards a low value in order to take the responsibility of avoiding the collision. The plots of Fig.~\ref{fig:alphas} show that, in general, very low parameters' values are not desirable, as they make the opinion grow slowly and the robot to react late. Nevertheless, very large values can make the nonlinear opinion dynamics adaptive law to be too sensitive to changes in the agent and the estimate ${e}_i$, leading to a non-smooth evolution of ${o}_i$. Therefore, it is suggested to design the parameters as low as possible to guarantee collision avoidance but obtain fluid maneuvers. The nominal parameters described in Table \ref{tab:parameter-values} have been tuned using Bayesian Optimization, constrained the domain of the parameters according to the conclusions drawn from Fig. \ref{fig:alphas}.

 \begin{figure}
 \centering
 \begin{tabular}{@{}cc@{}}
  \footnotesize{a) Different $d_i$} & \footnotesize{b) Different $a_i$, $c_i=1-a_i$} \\
  \includegraphics[width=0.47\linewidth]{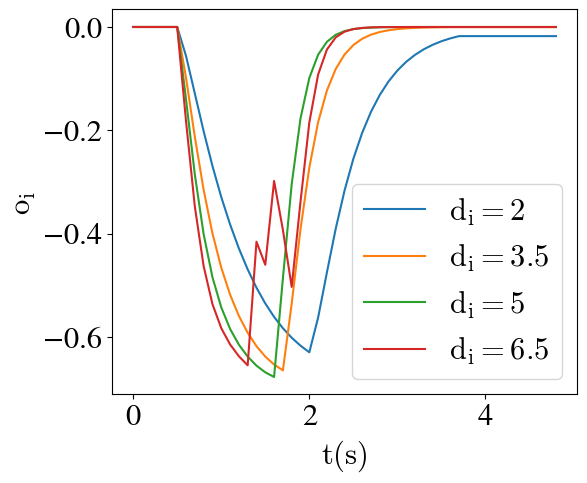} & 
  \includegraphics[width=0.47\linewidth]{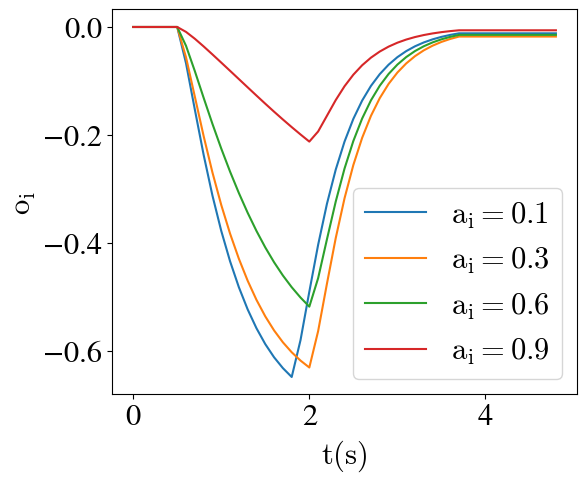} \\
  \footnotesize{c) Different $\varepsilon$} & \footnotesize{d) Different $\kappa_i$} \\
    \includegraphics[width=0.47\linewidth]{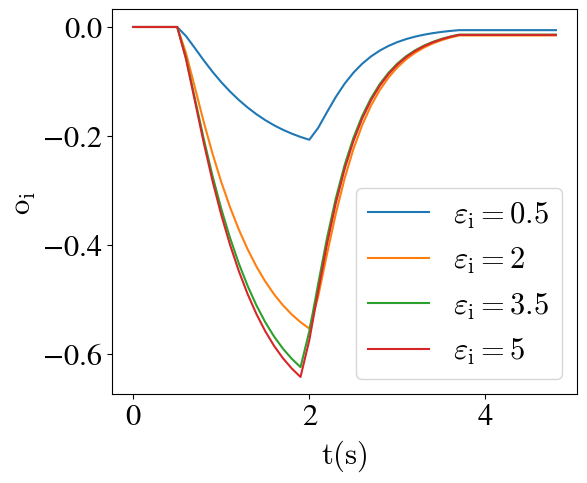} & 
  \includegraphics[width=0.47\linewidth]{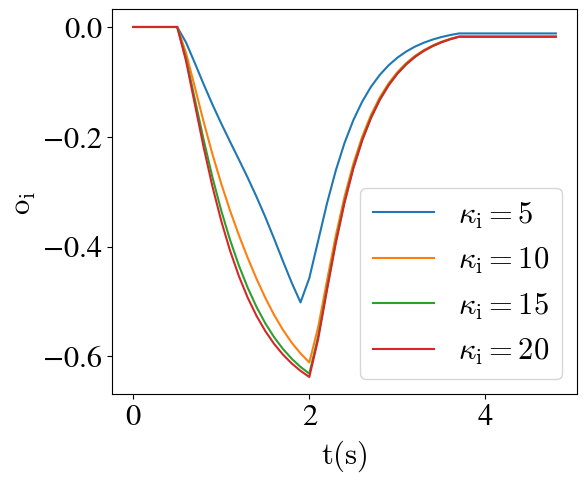} \\
  \footnotesize{e) Different $\delta_i$} & \footnotesize{f) Different $b_i$} \\
    \includegraphics[width=0.47\linewidth]{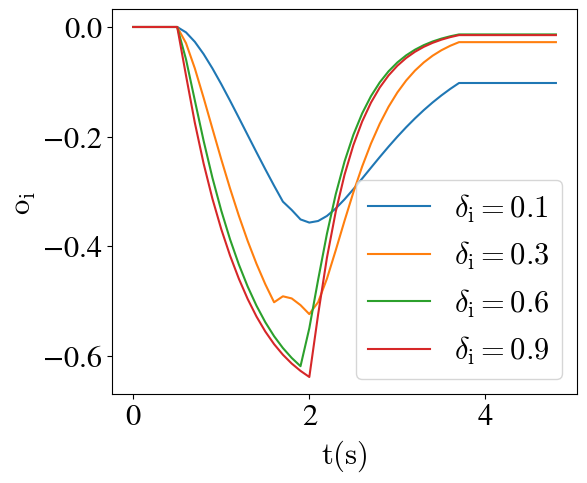} & 
  \includegraphics[width=0.47\linewidth]{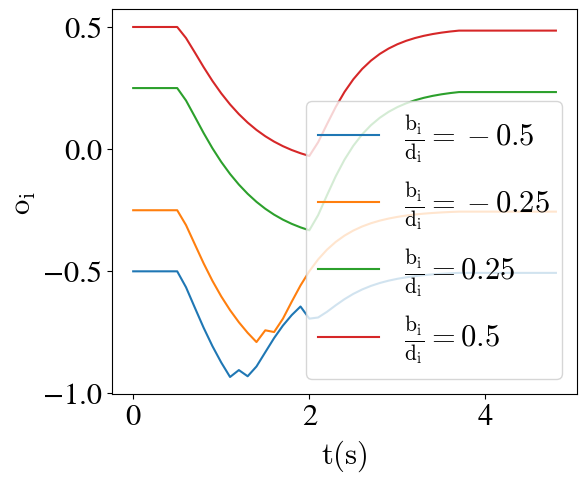} \\
 \end{tabular}
 \caption{Evolution of the degree of cooperation in the head-on scenario against a non-cooperative agent, varying the value of each parameter of 
 \textsf{AVOCADO}.}
 \label{fig:alphas}
 \end{figure}

\subsection{Multi-agent scenarios}\label{subsec:multiagent}

After studying the influence of the different parameters in \textsf{AVOCADO}, we conduct a series of multi-agent simulations to evaluate the performance of \textsf{AVOCADO} compared to existing state-of-the-art planners.  We compare \textsf{AVOCADO} with \textsf{ORCA} \cite{van2011reciprocal} and \textsf{RVO-RL} \cite{han2022reinforcement} as baselines planners that consider reciprocal collision avoidance, and \textsf{T-MPC} \cite{mavrogiannis2022winding,poddar2023crowd} and \textsf{SARL} \cite{chen2019crowd} as baseline planners for non-cooperative collision avoidance. The former is based on a model predictive control formulation and the latter is a neural-network-based planner trained using reinforcement learning. In this sense, 
\textsf{SARL} is retrained to satisfy the constraints in perception from the limited sensing radius of robots and agents. Moreover, we set $\sigma =0.0001 $ for \textsf{AVOCADO}.

We design two evaluation settings. The \textit{circle} scenario is characterized by an evenly spaced initial position of robots and agents, forming a circle and with the goal of reaching the initial position of the opposite robot or agent. On the other hand, the \textit{crossing} scenario is characterized by a random initial position of all the agents and robots in the border of a square. Robots and agents must navigate to a individually assigned random goal in the opposite side of the square. Cooperative robots are placed in the sides oriented to one axis and non-cooperative agents are placed in the sides oriented to the other axis. This arrangement enforces that cooperative robots traverse non-cooperative agents that move perpendicularly as a traffic flow they have to cross. Besides, to ensure that robots do not take trivial motion strategies (e.g., wait until all the agents have reached the goal), once an agent reaches its goal, the goal is modified to be the initial position of the agent, so the environment is always dynamic.

We show the qualitative behavior of the robots in non-cooperative \textit{circle}  and \textit{crossing} scenarios in Fig.~\ref{fig:non-collaborative}. The planners address the navigation conflict at the center of the circle in two different ways. \textsf{T-MPC} and \textsf{RVO-RL} completely avoid dangerous zones by taking a detour; this might not be possible in a constrained scenario with boundaries. Among the other three approaches, \textsf{ORCA} is the one that takes more risks, since it assumes reciprocity in the degree of cooperation,  colliding in both scenarios. \textsf{SARL} and \textsf{AVOCADO} achieve adaptation to the unknown degree of cooperation of the agents and other robots and manage to reach the goal in both scenarios; nonetheless, the movements exerted by the robots using \textsf{AVOCADO} are much more seamless than those exerted by \textsf{SARL}, which is a key aspect when transferring navigation policies to real robots.

 \begin{figure*}
 \centering
 \begin{tabular}{@{}ccccc@{}}
     \footnotesize{a) \textsf{ORCA}} & \footnotesize{b) \textsf{T-MPC}} & \footnotesize{c) \textsf{SARL}} & \footnotesize{d) \textsf{RVO-RL}} & \footnotesize{e) \textsf{AVOCADO}} \\\includegraphics[width=0.18\linewidth]{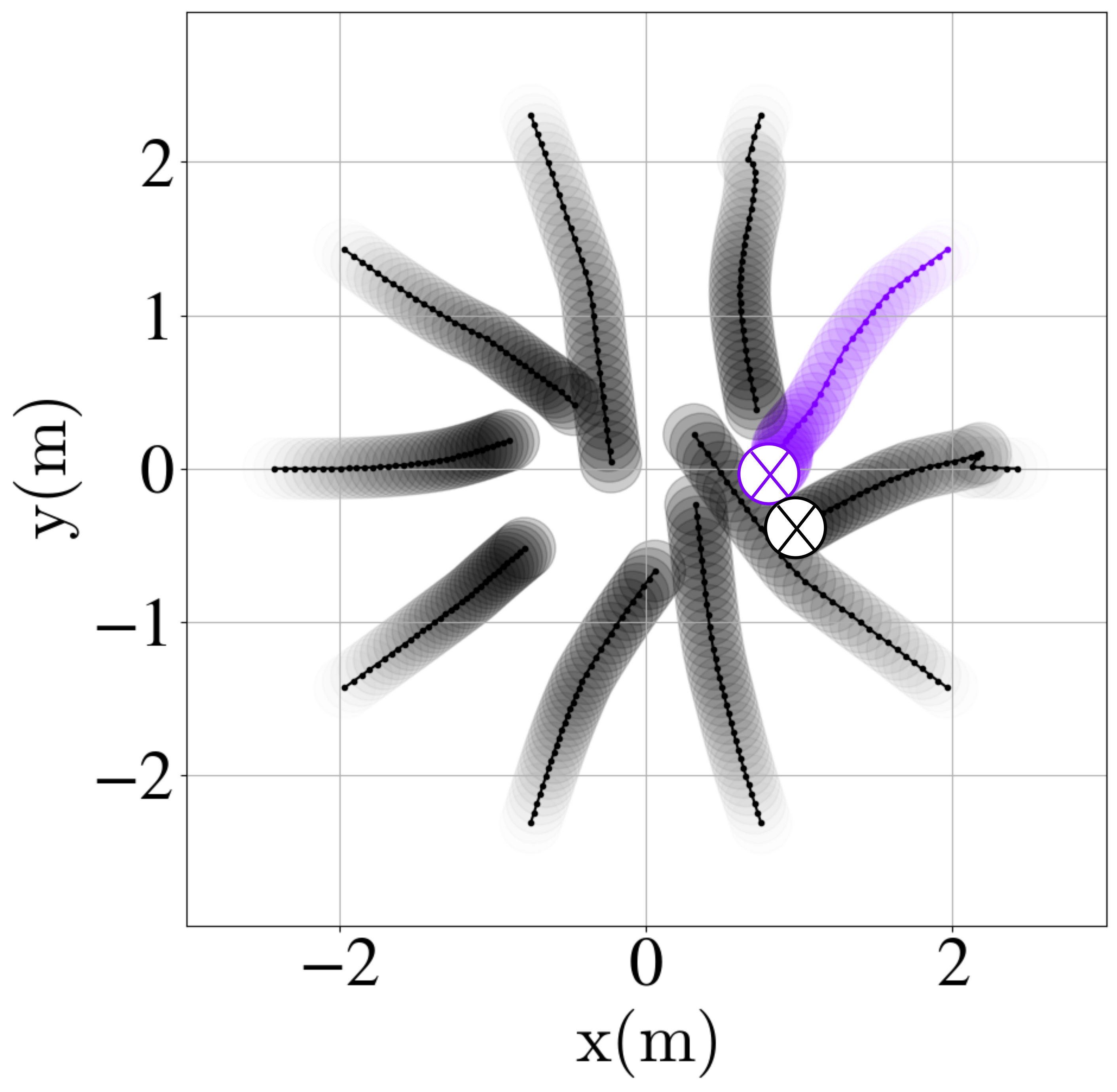} & 
     \includegraphics[width=0.18\linewidth]{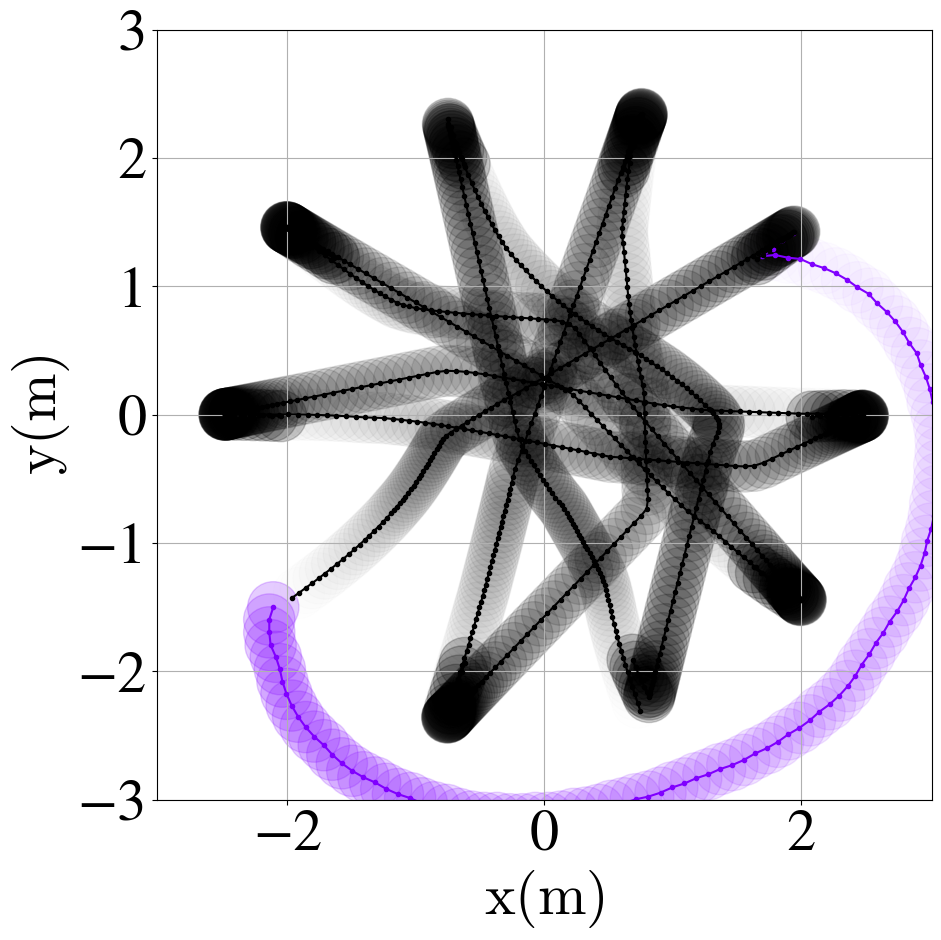} & 
      \includegraphics[width=0.18\linewidth]{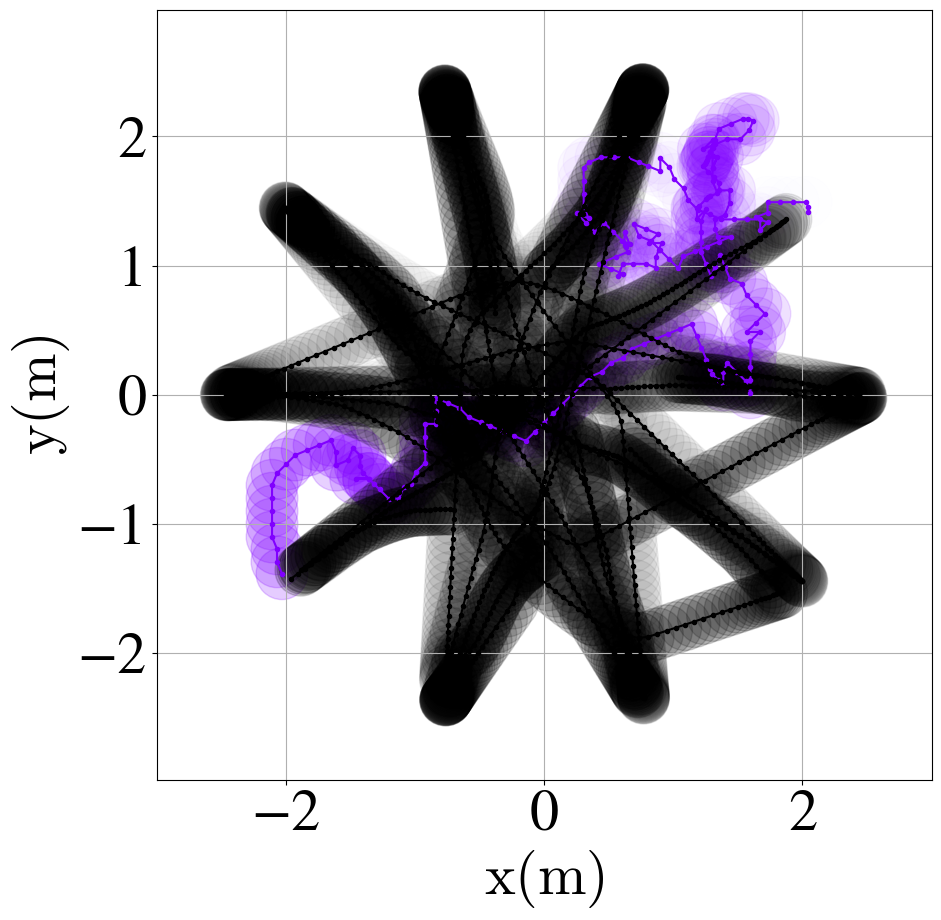} &
      \includegraphics[width=0.18\linewidth]{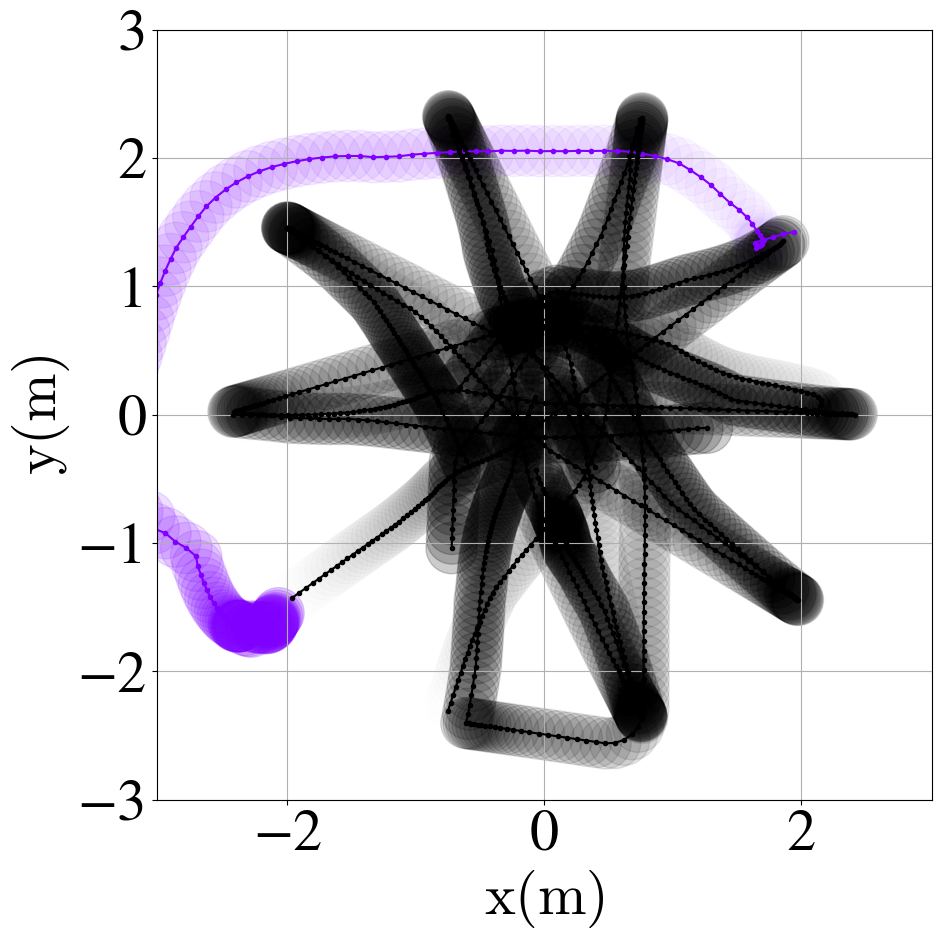} &
     \includegraphics[width=0.18\linewidth]{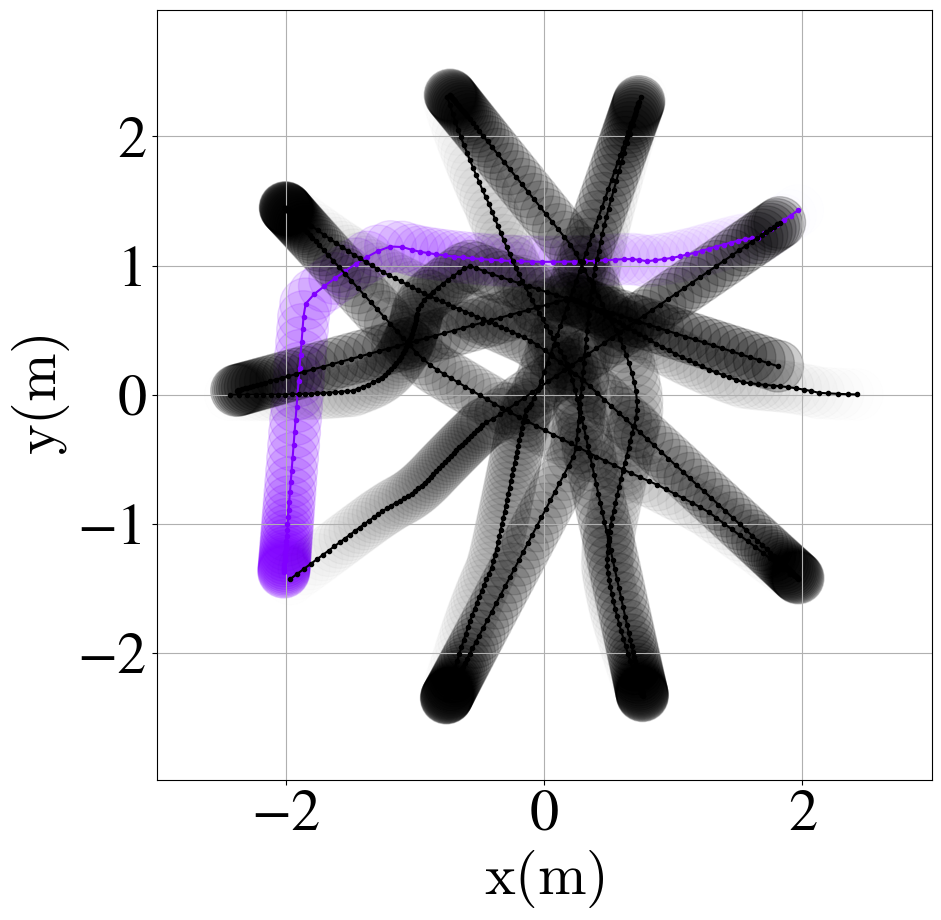}  
 \\
      \includegraphics[width=0.18\linewidth]{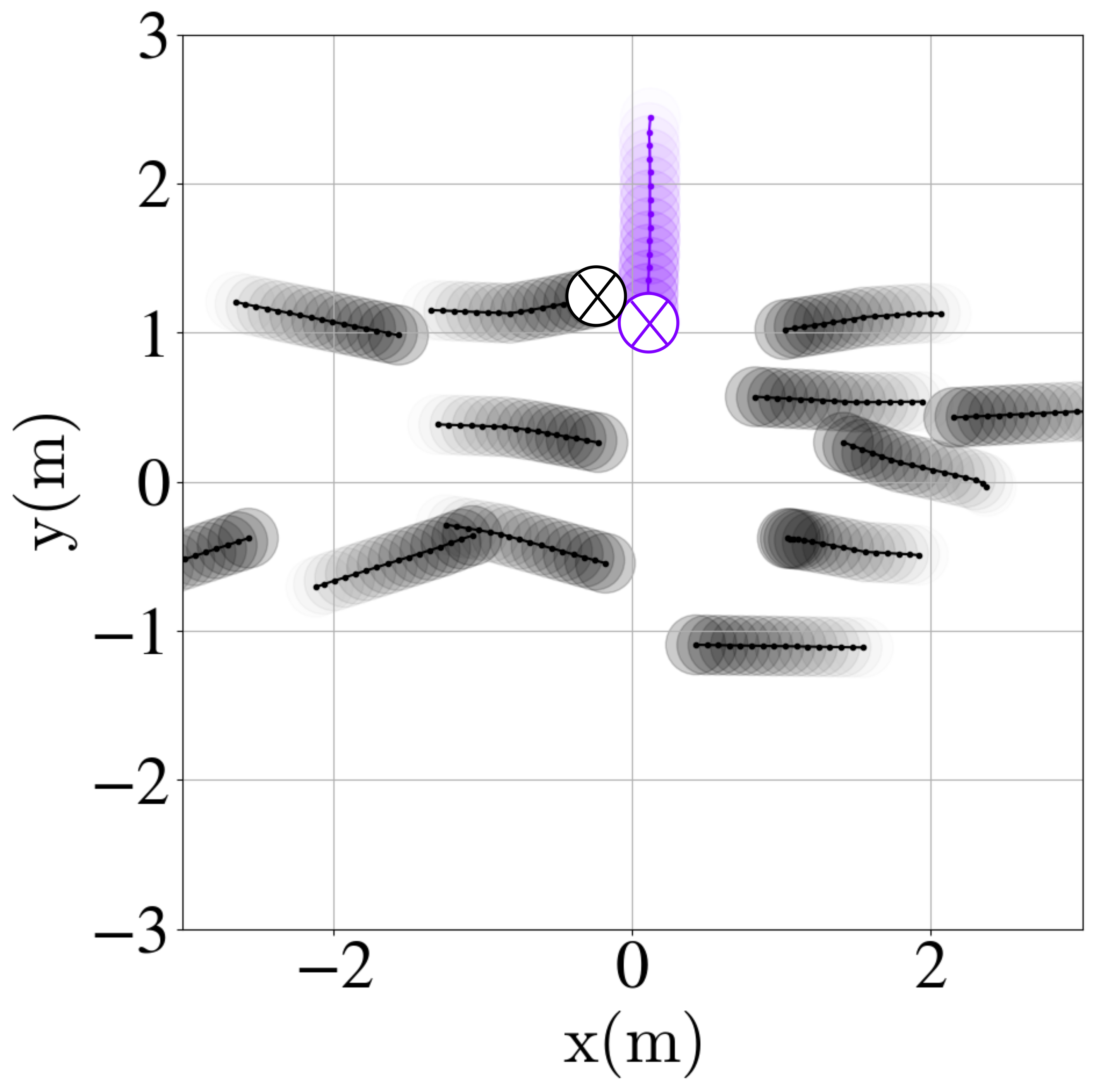} & 
      \includegraphics[width=0.18\linewidth]{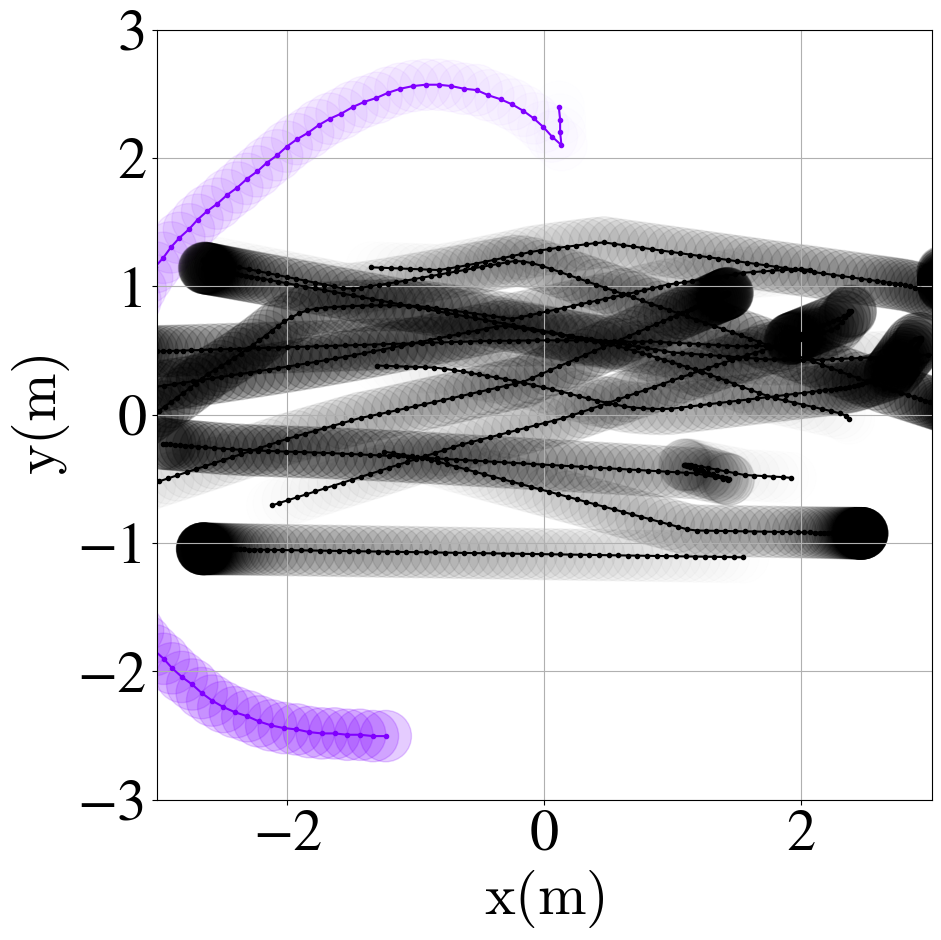} & 
      \includegraphics[width=0.18\linewidth]{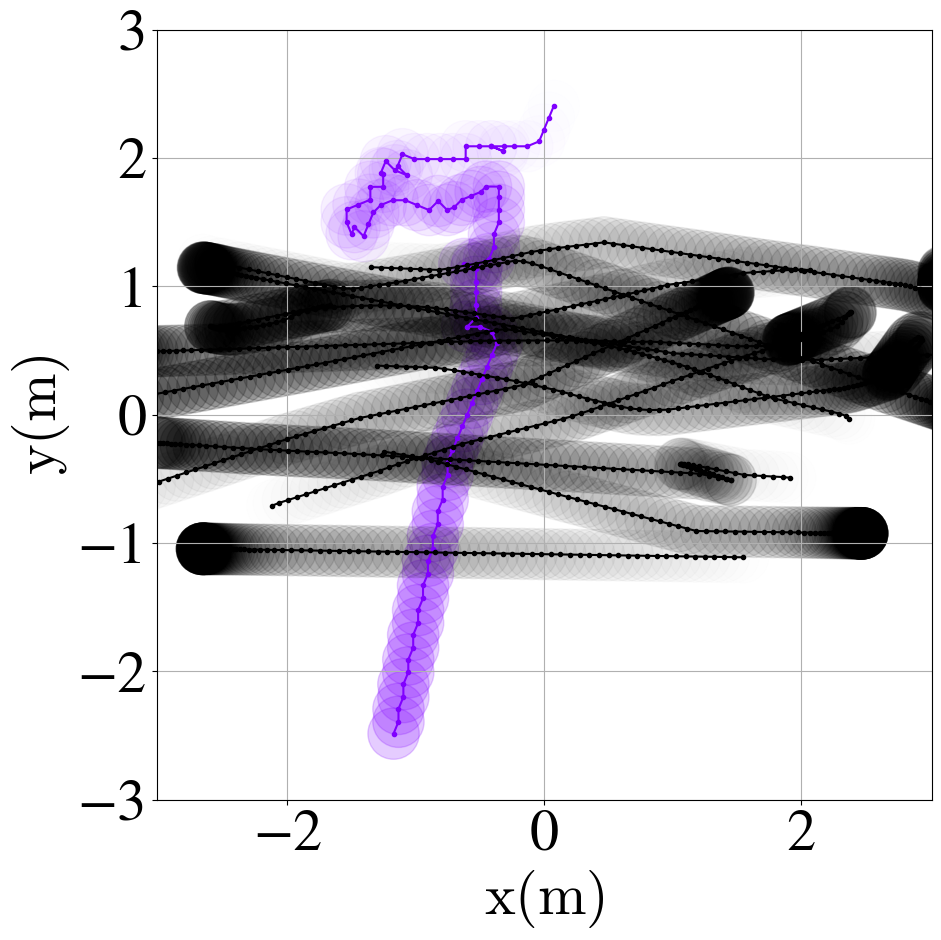} &
      \includegraphics[width=0.18\linewidth]{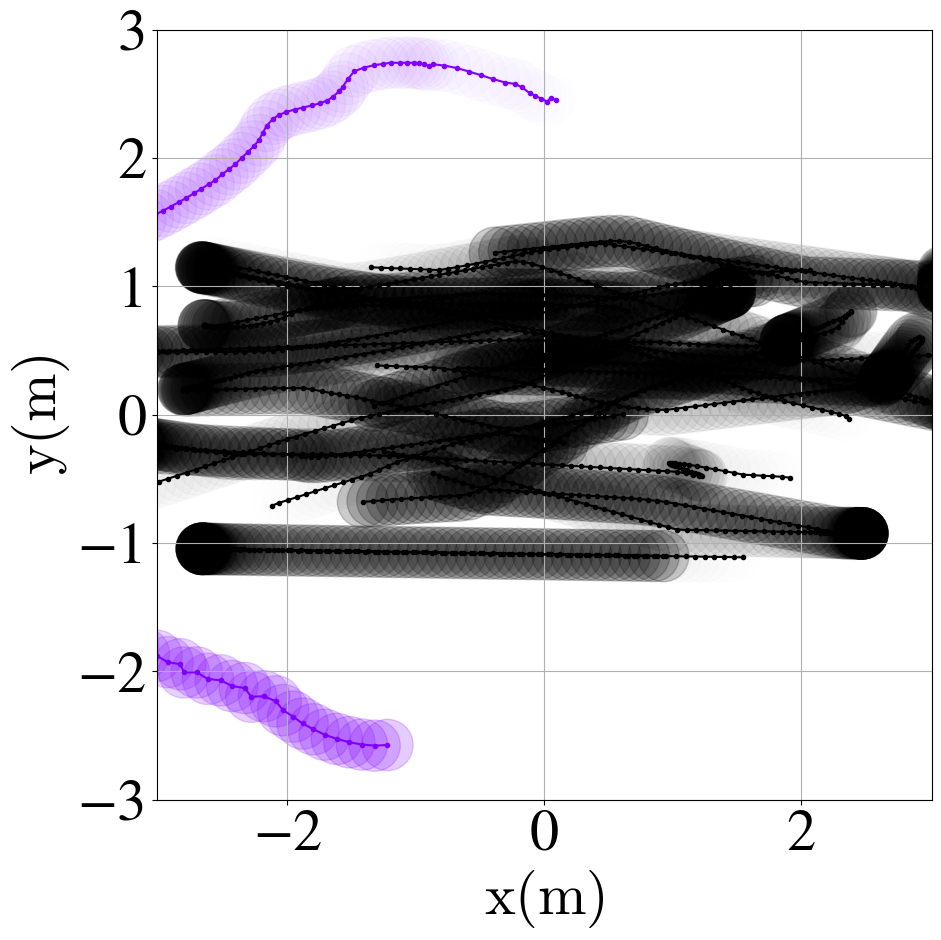} &
      \includegraphics[width=0.18\linewidth]{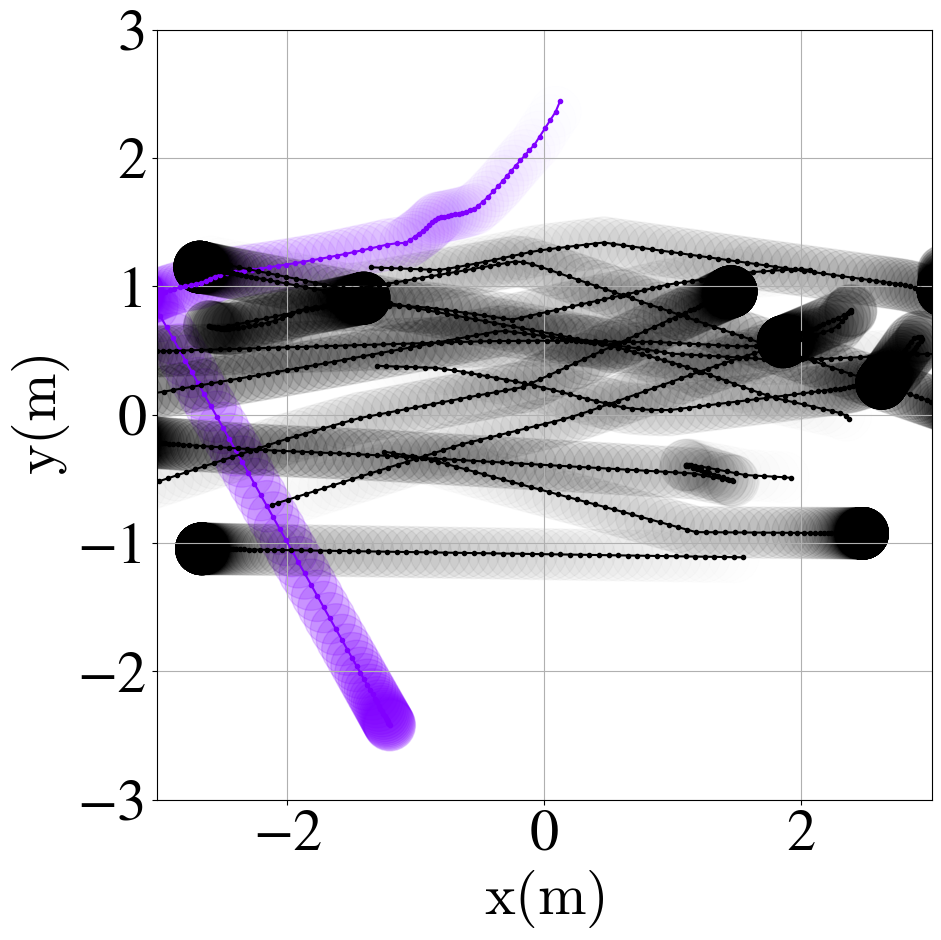}  \\
 \end{tabular}
 \caption{Non-cooperative \textit{circle} (top) and \textit{crossing} (bottom) scenarios with $10$ agents and different planners. The robot is colored in purple, and the non-cooperative agents are in black.  When the robot collides with an agent, it becomes transparent. The episode finishes when all the agents and the robot reach the goal or collide, or the simulation lasts more than $100$s.}
 \label{fig:non-collaborative}
 \end{figure*}

Fig.~\ref{fig:col-circles} shows examples of scenarios only populated with robots. Fig.~\ref{fig:col-circles}a demonstrates the deadlocks suffered by reciprocal approaches like \textsf{ORCA} when symmetries are encountered. \textsf{AVOCADO} overcomes this problem, as seen in Fig.~\ref{fig:col-circles}e, by exploiting the attention mechanism in the nonlinear opinion dynamics adaptive law. \textsf{T-MPC} and \textsf{RVO-RL} take long unnecessary detours, as they are designed for scenarios with fewer robots and to take larger safety margins. \textsf{T-MPC} even presents many collisions, as it is not prepared to interact with other robots, and some of the \textsf{RVO-RL} robots do not reach the goal before the time out of $100$s. Trajectories of robots using \textsf{SARL} are, again, very irregular, probably due to a reciprocal dance problem, specially noticeable in \textit{circle} scenarios where this leads to collisions. All the robots using \textsf{AVOCADO} reach their goals avoiding collisions, with short path lengths and fluid trajectories, in both scenarios.

 \begin{figure*}
 \centering
 \begin{tabular}{@{}ccccc@{}}
      
     \footnotesize{a) \textsf{ORCA}} & \footnotesize{b) \textsf{T-MPC}} & \footnotesize{c) \textsf{SARL}} & \footnotesize{d) \textsf{RVO-RL}} & \footnotesize{e) \textsf{AVOCADO}} \\\includegraphics[width=0.18\linewidth]{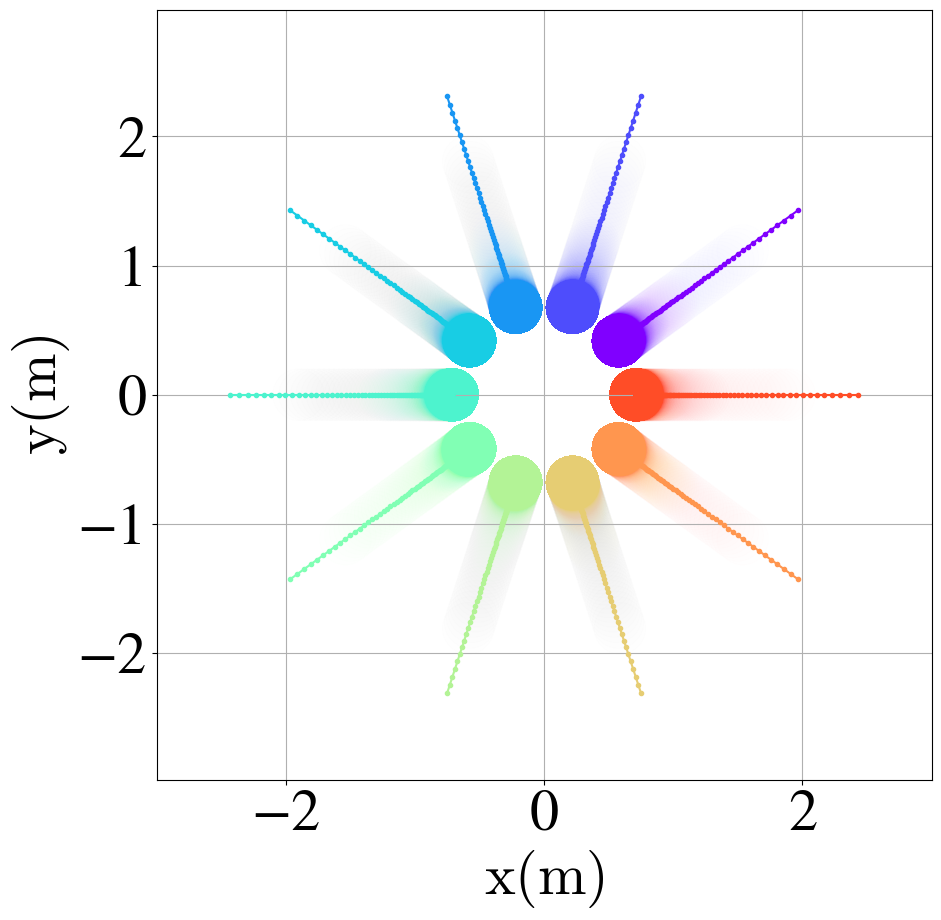} & 
      \includegraphics[width=0.18\linewidth]{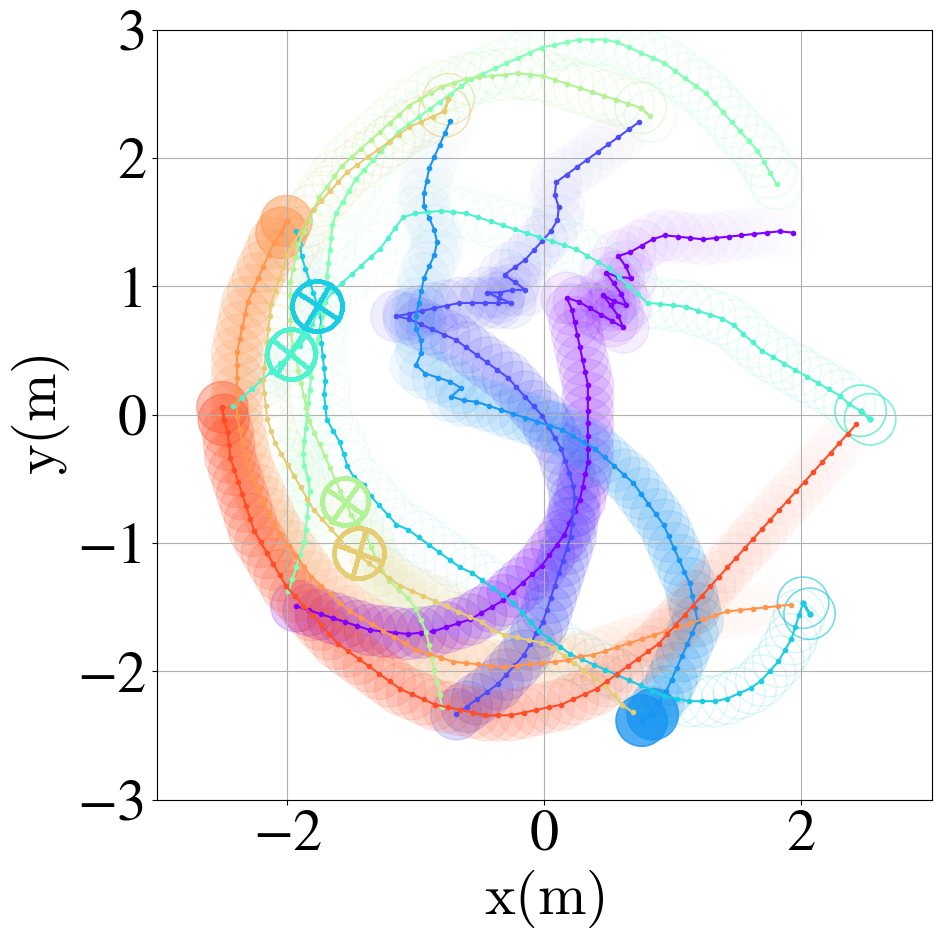} & 
      \includegraphics[width=0.18\linewidth]{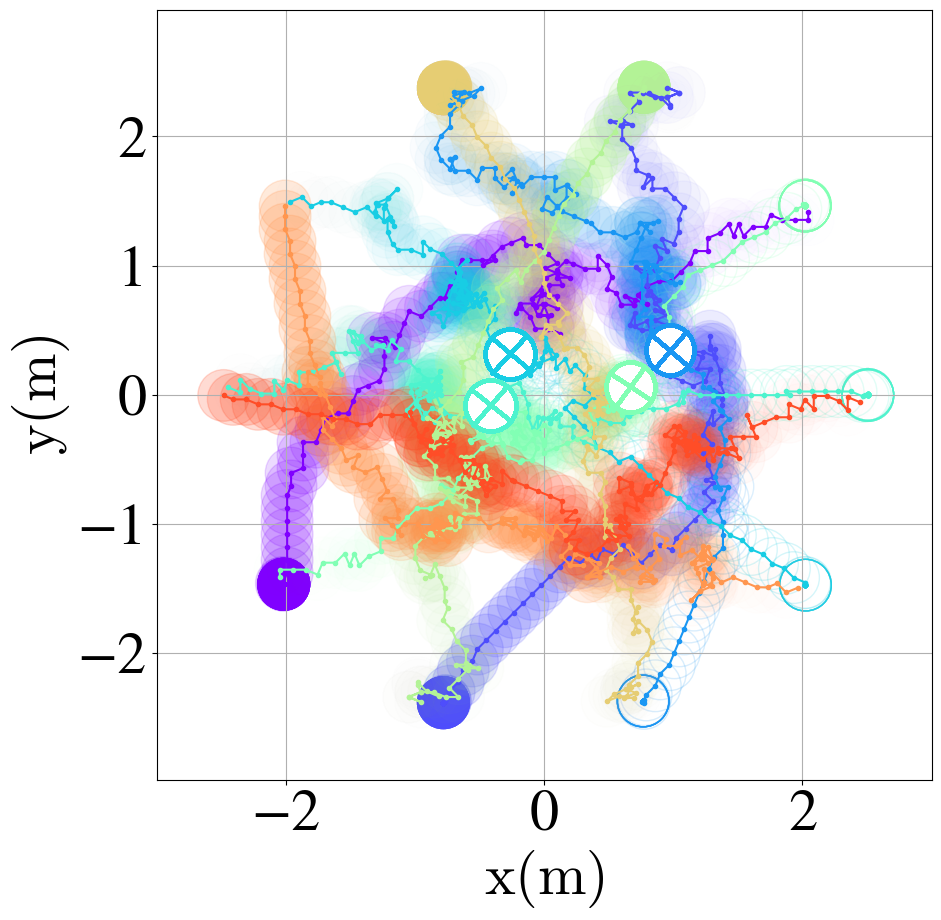} & 
      \includegraphics[width=0.18\linewidth]{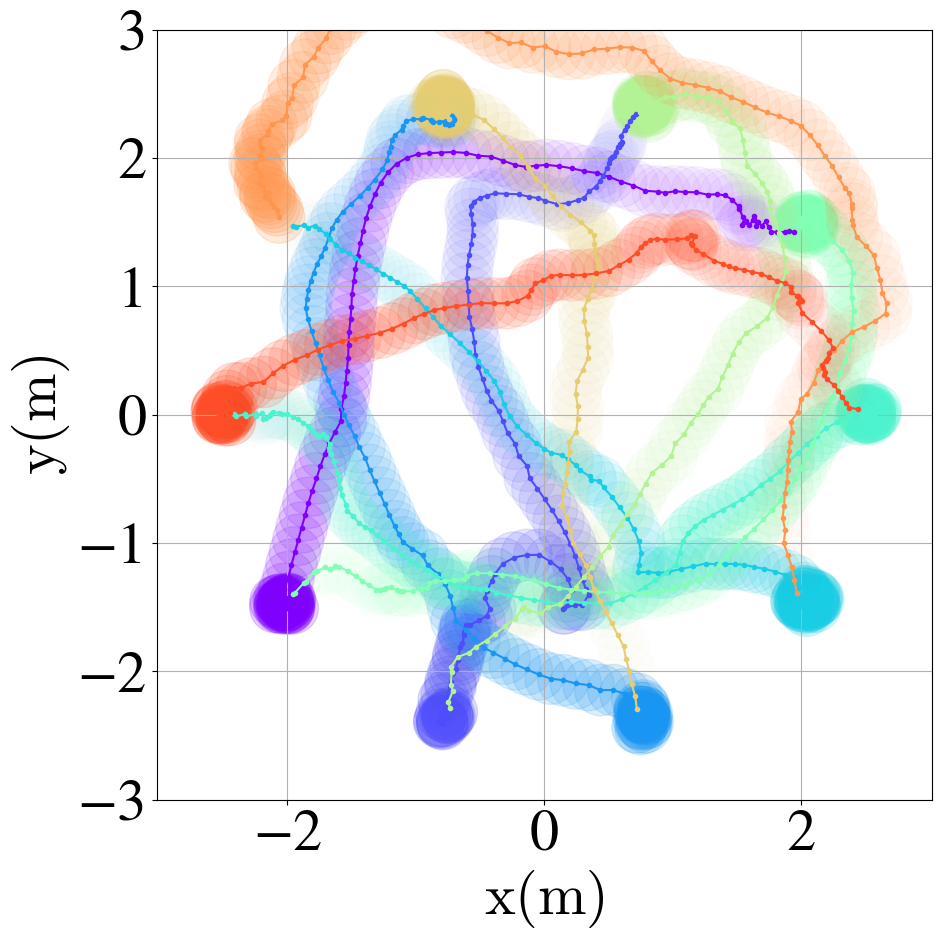} & \includegraphics[width=0.18\linewidth]{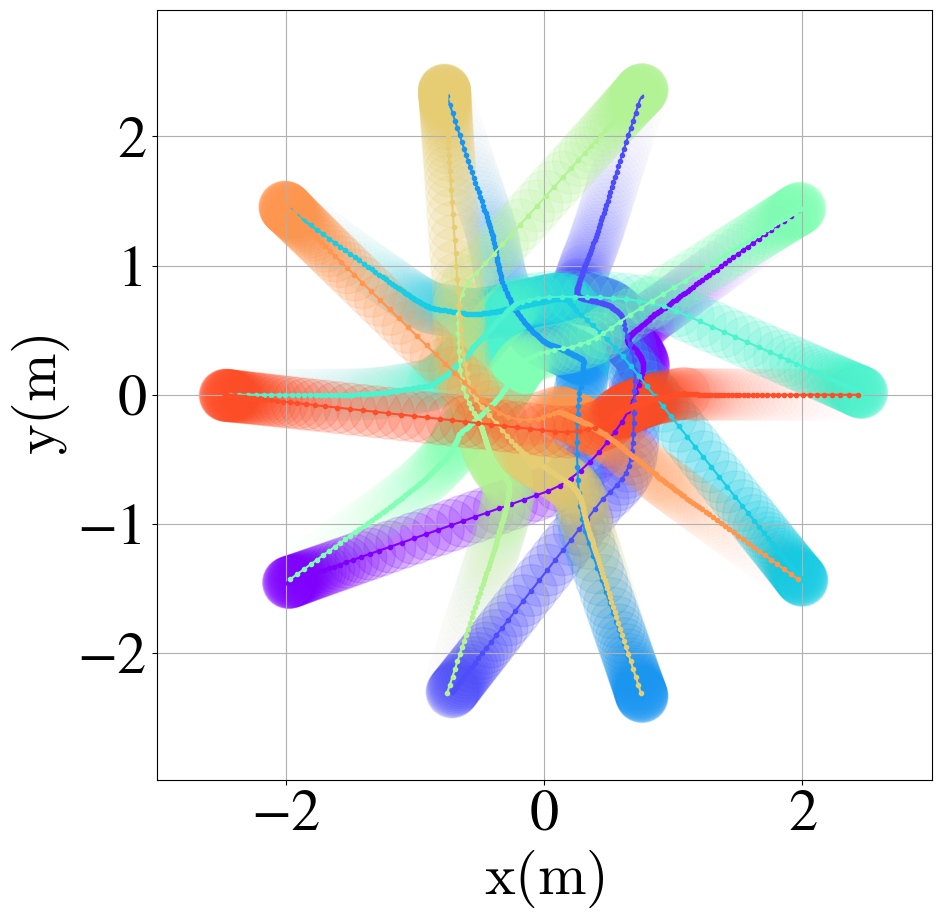}\\
      \includegraphics[width=0.18\linewidth]{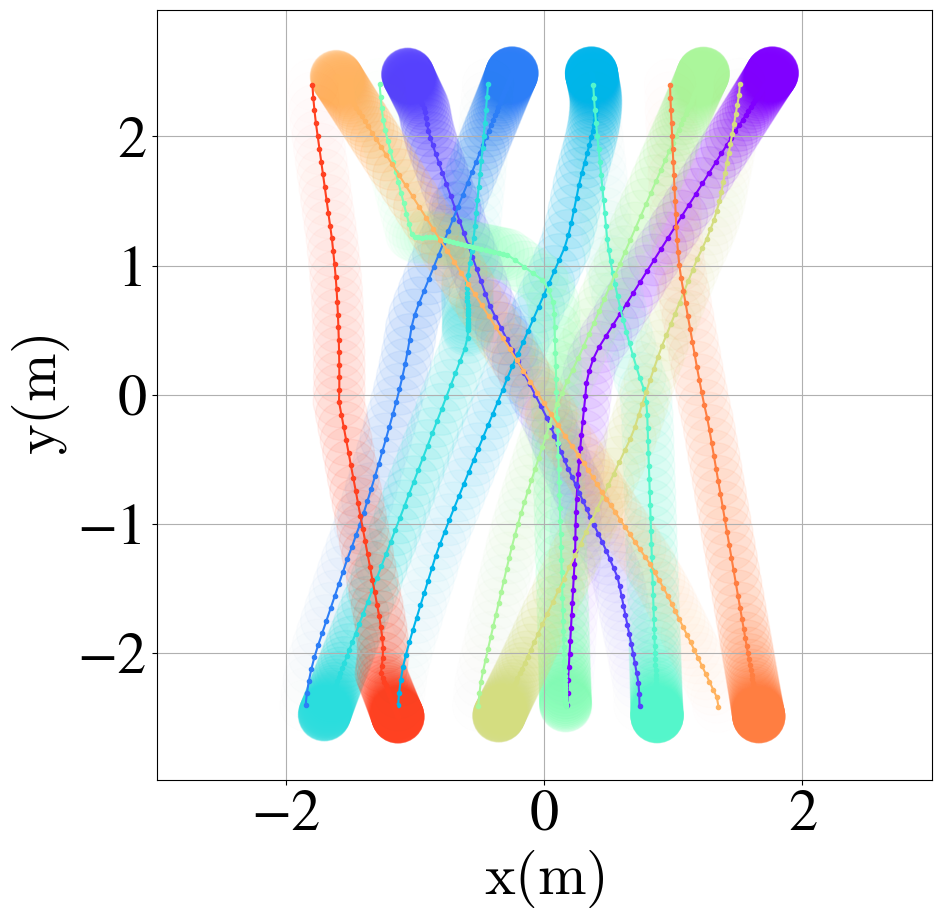} & 
     \includegraphics[width=0.18\linewidth]{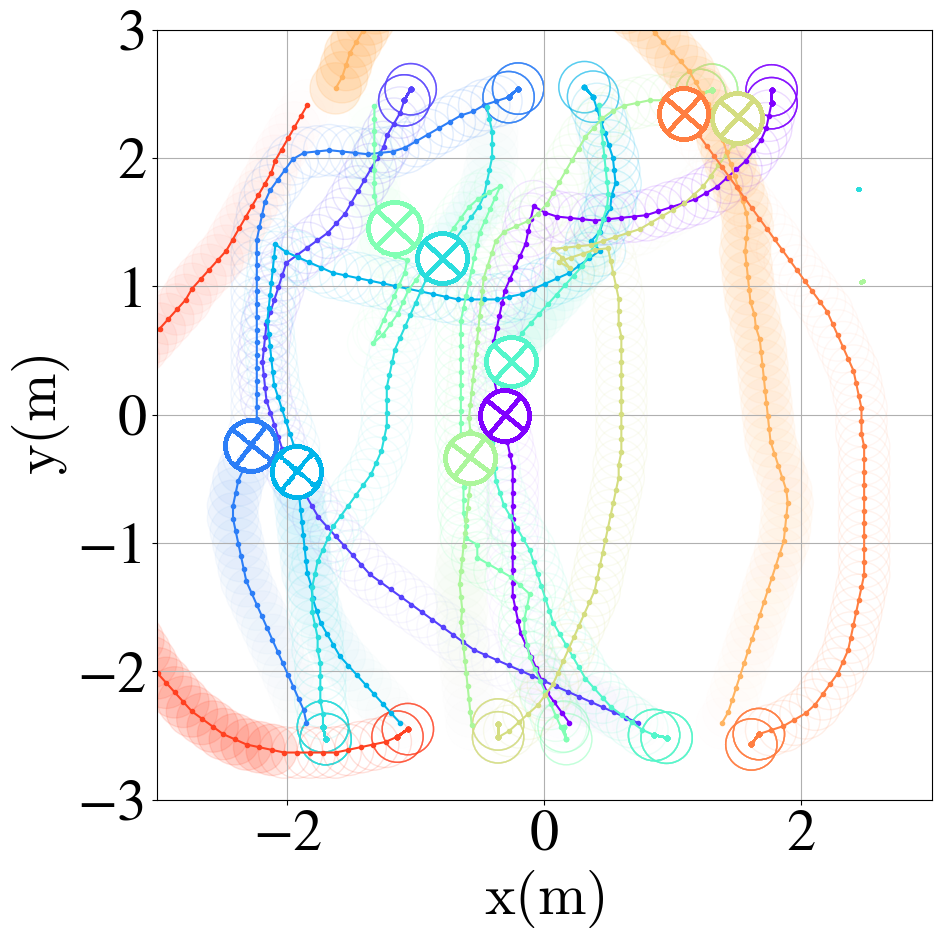} & 
      \includegraphics[width=0.18\linewidth]{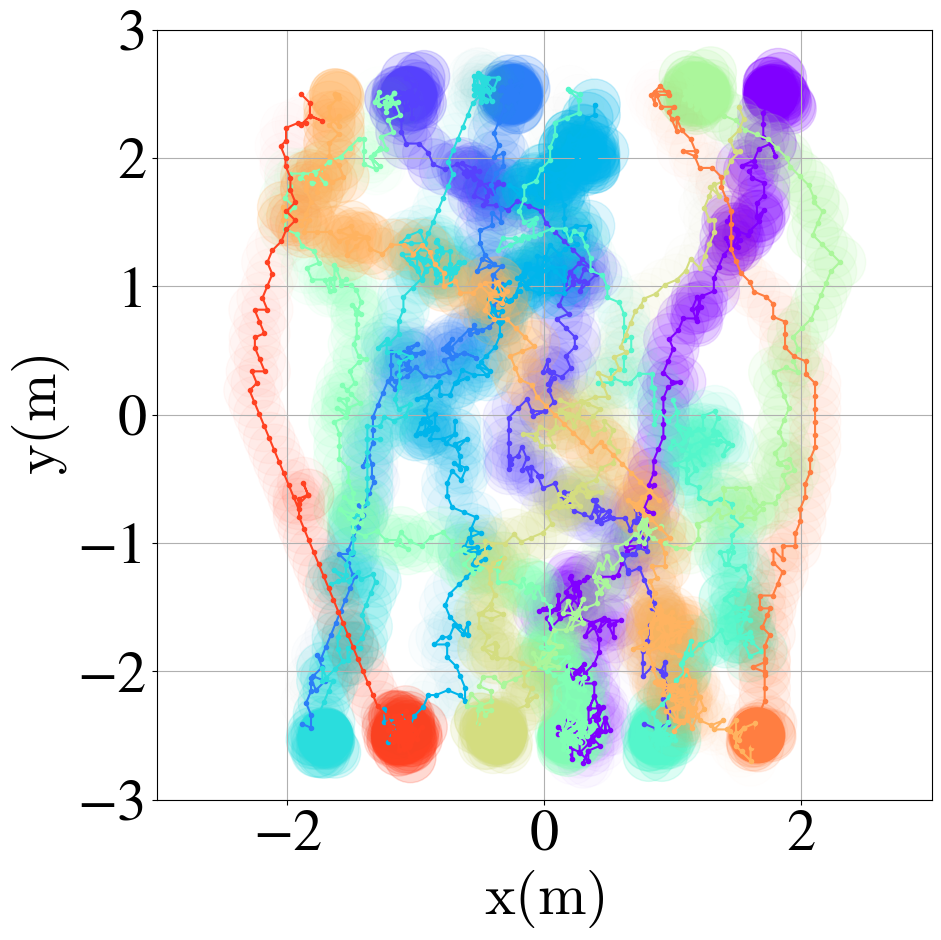} &
      \includegraphics[width=0.18\linewidth]{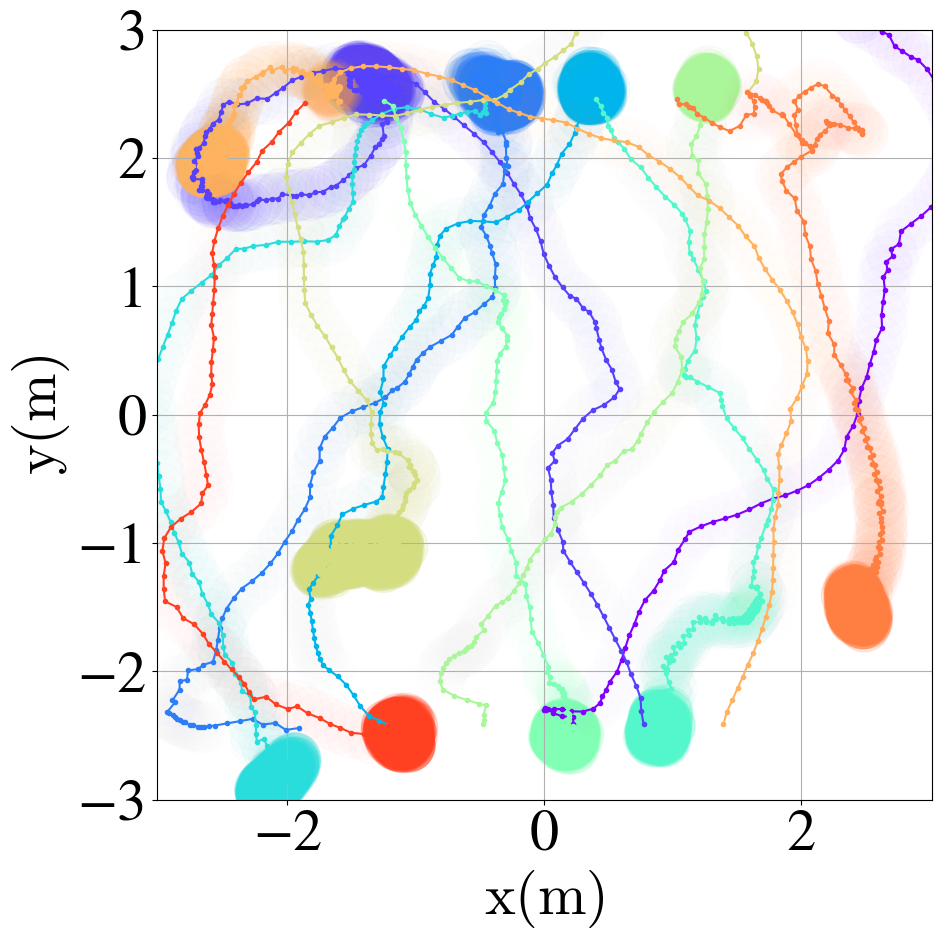} &
     \includegraphics[width=0.18\linewidth]{orca-12-head-on.png} \\ 
 \end{tabular}
 \caption{Cooperative \textit{circle} (top) and \textit{crossing} (bottom) scenarios with 10 robots and different planners. Each robot is depicted with a different color. The representation follows the same rules in Fig.~\ref{fig:non-collaborative}. Collisions are marked with crosses surrounded by circles}.
 \label{fig:col-circles}
 \end{figure*}

Mixed cooperative/non-cooperative scenarios are the most challenging for most of the planners, as they assume some sort of degree of cooperation, either reciprocity or complete absence of cooperation. Fig.~\ref{fig:part-circles} provides examples of such scenarios. \textsf{T-MPC} and \textsf{RVO-RL} avoid passing through the center of the stage, as there are many non-cooperative agents. However, they encounter cooperative robots during the detour, so they are unable to avoid collisions. \textsf{SARL} faces again the reciprocal dance problem when evading both non-cooperative agents and other robots avoiding collisions. \textsf{ORCA} and \textsf{AVOCADO} perform similar trajectories, but \textsf{AVOCADO} adaptation capabilities makes it safer, avoiding all collisions unlike \textsf{ORCA}. 

 \begin{figure*}
 \centering
 \begin{tabular}{@{}ccccc@{}}
      
     \footnotesize{a) \textsf{ORCA}} & \footnotesize{b) \textsf{T-MPC}} & \footnotesize{c) \textsf{SARL}} & \footnotesize{d) \textsf{RVO-RL}} & \footnotesize{e) \textsf{AVOCADO}} \\
     \includegraphics[width=0.18\linewidth]{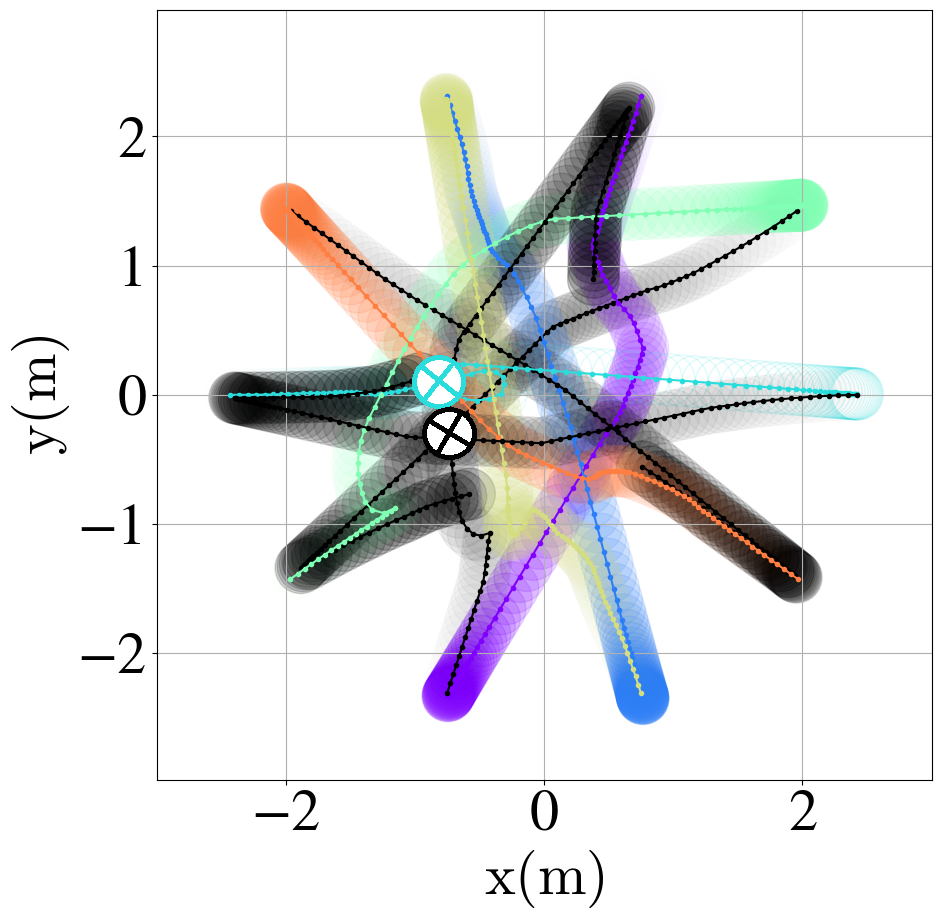} & 
      \includegraphics[width=0.18\linewidth]{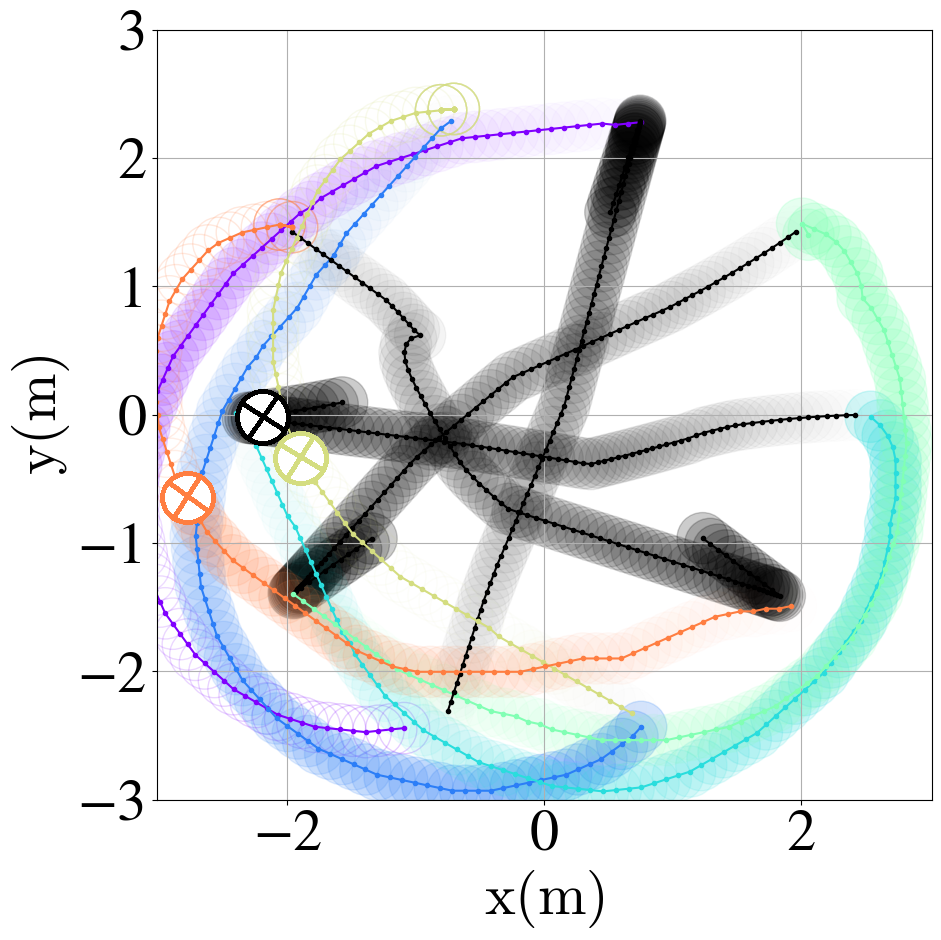} & 
      \includegraphics[width=0.18\linewidth]{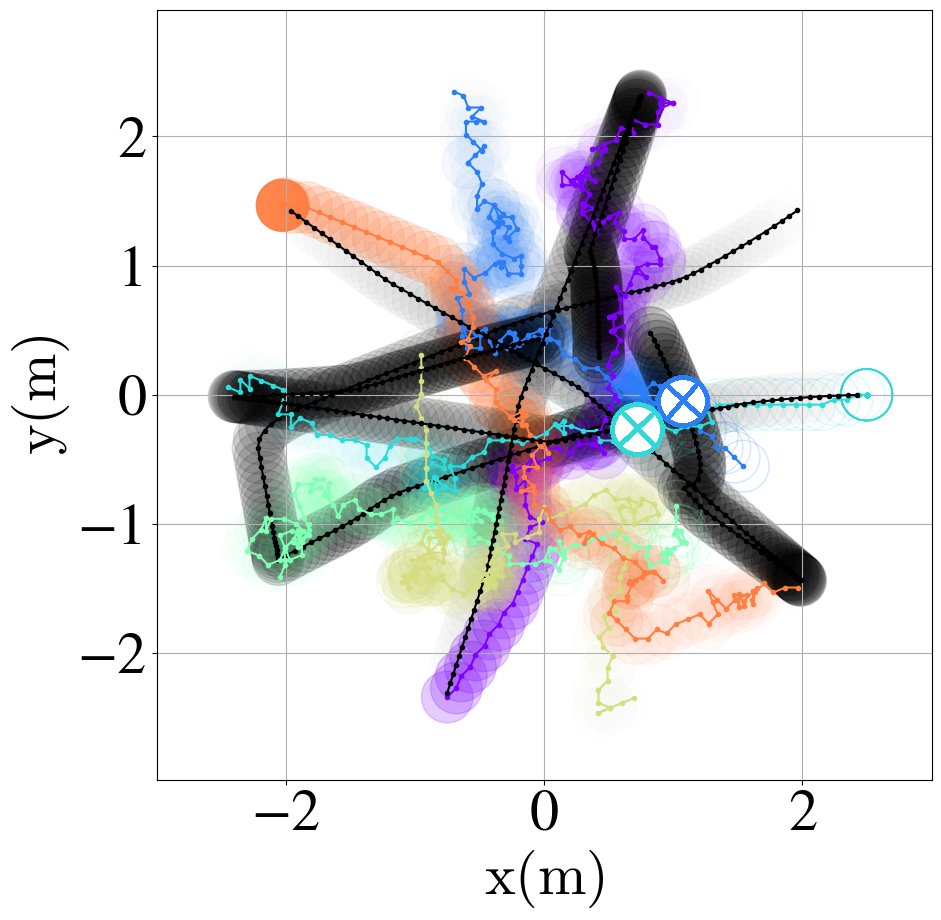} &
      \includegraphics[width=0.18\linewidth]{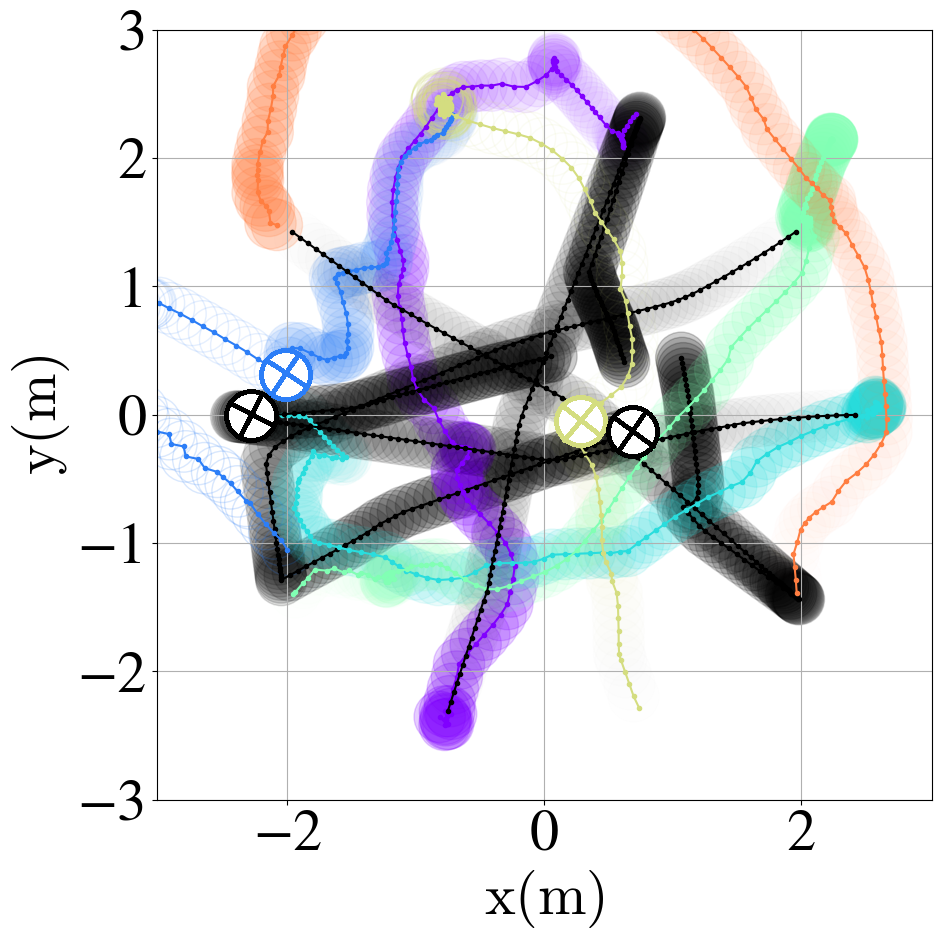} &
     \includegraphics[width=0.18\linewidth]{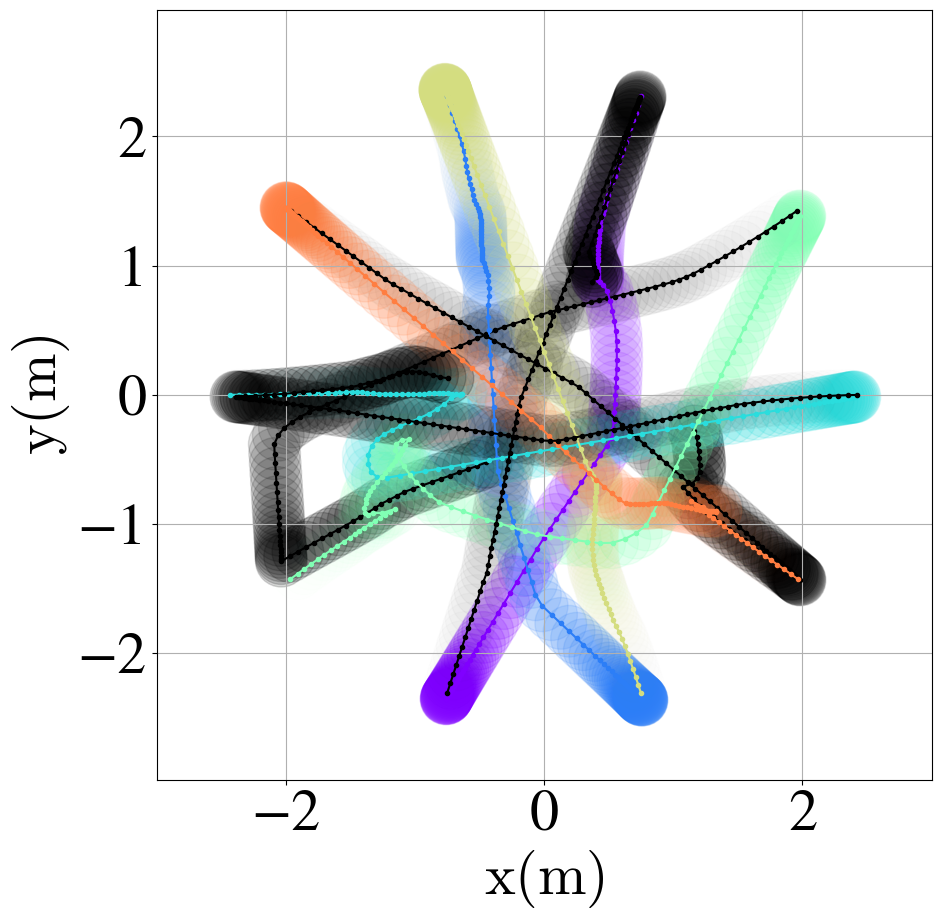} \\ 
     \includegraphics[width=0.18\linewidth]{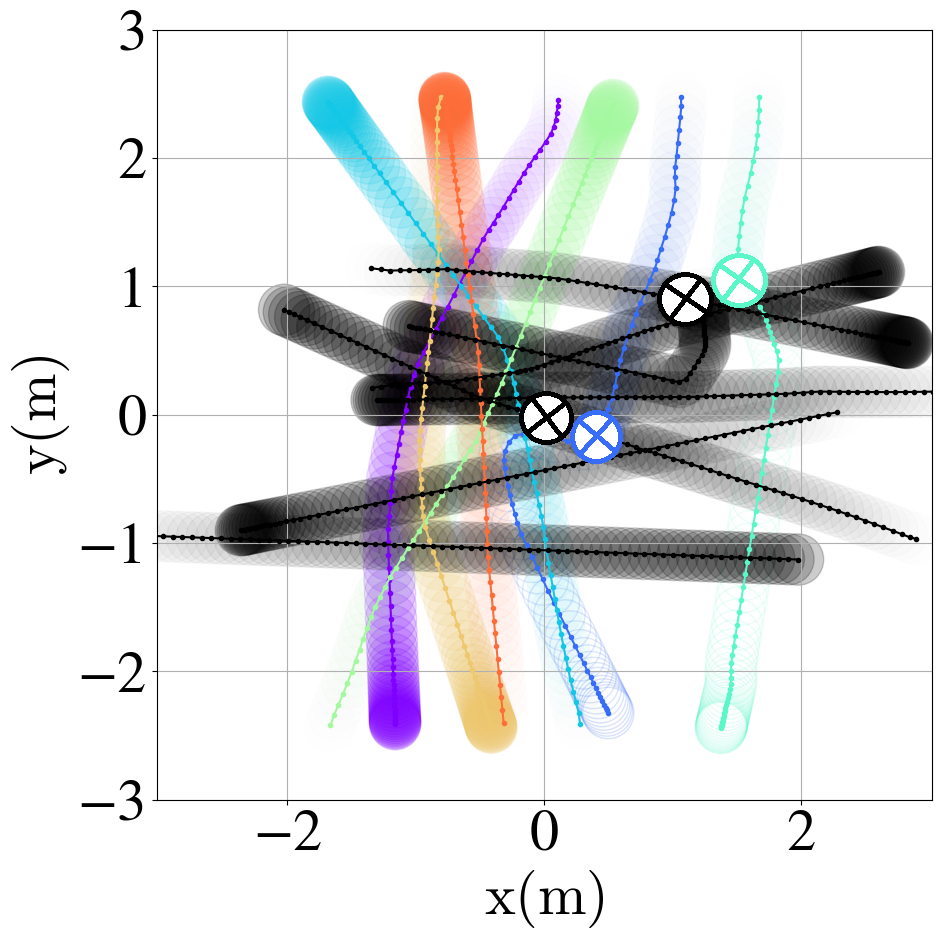} & 
     \includegraphics[width=0.18\linewidth]{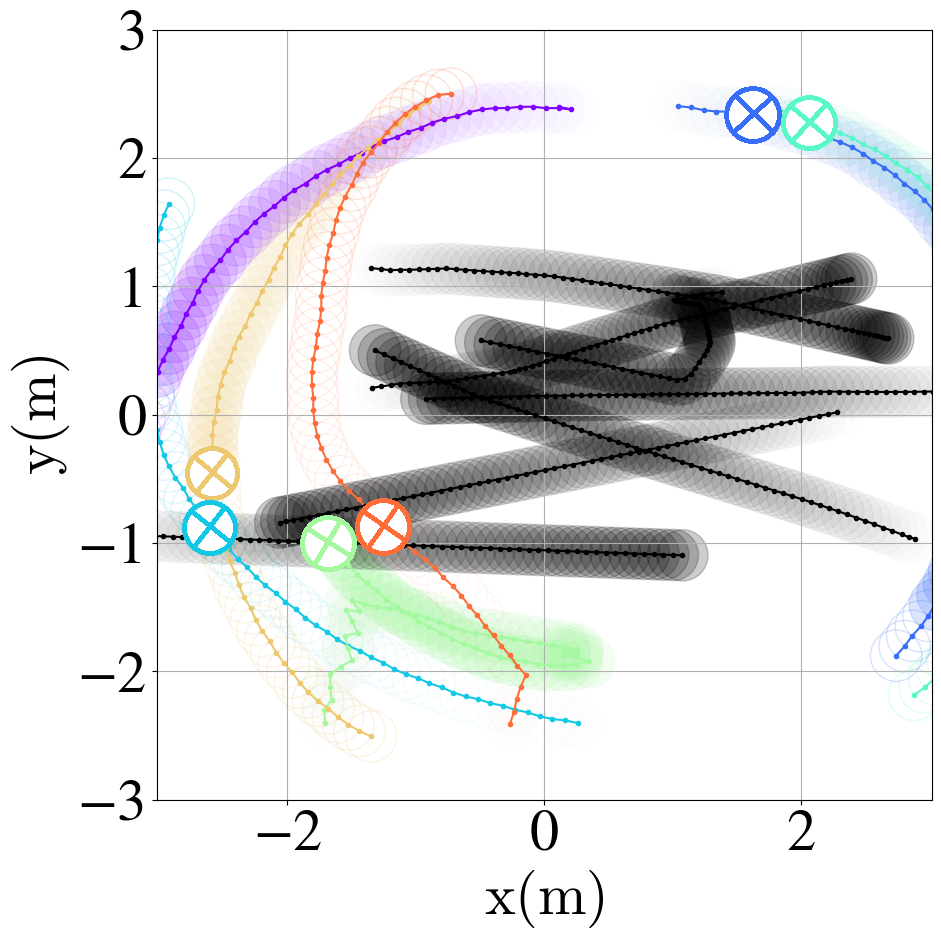} & 
      \includegraphics[width=0.18\linewidth]{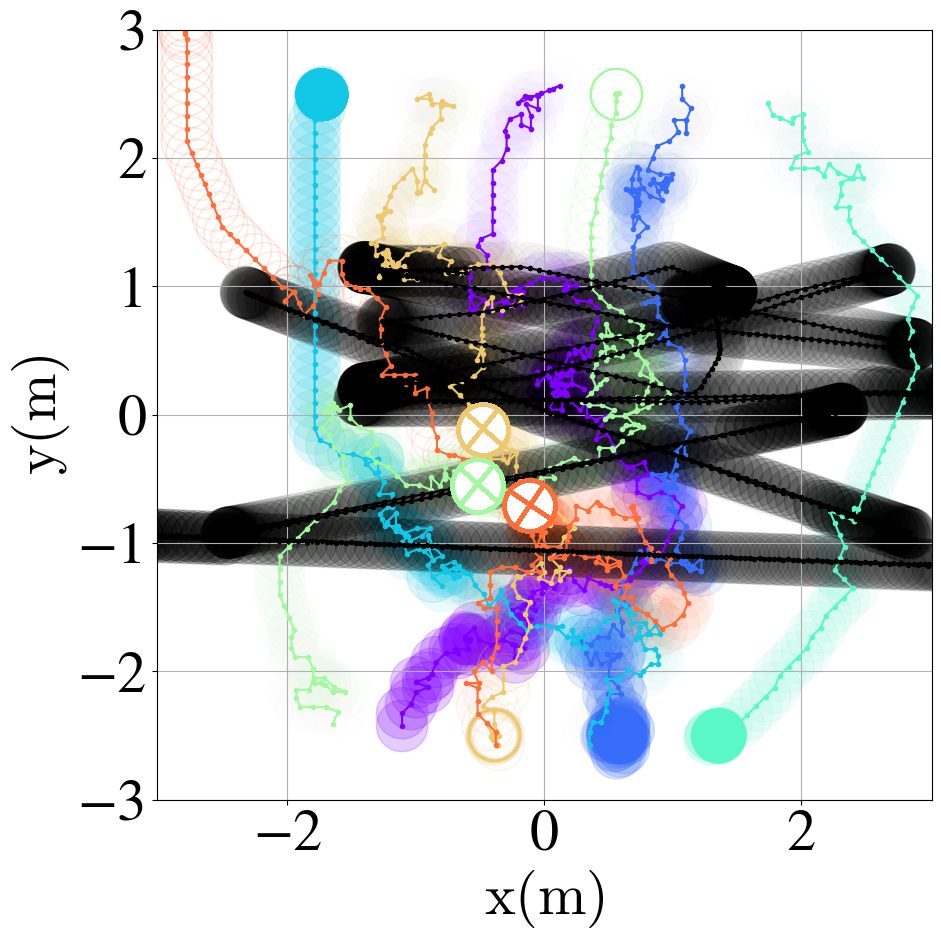} &
      \includegraphics[width=0.18\linewidth]{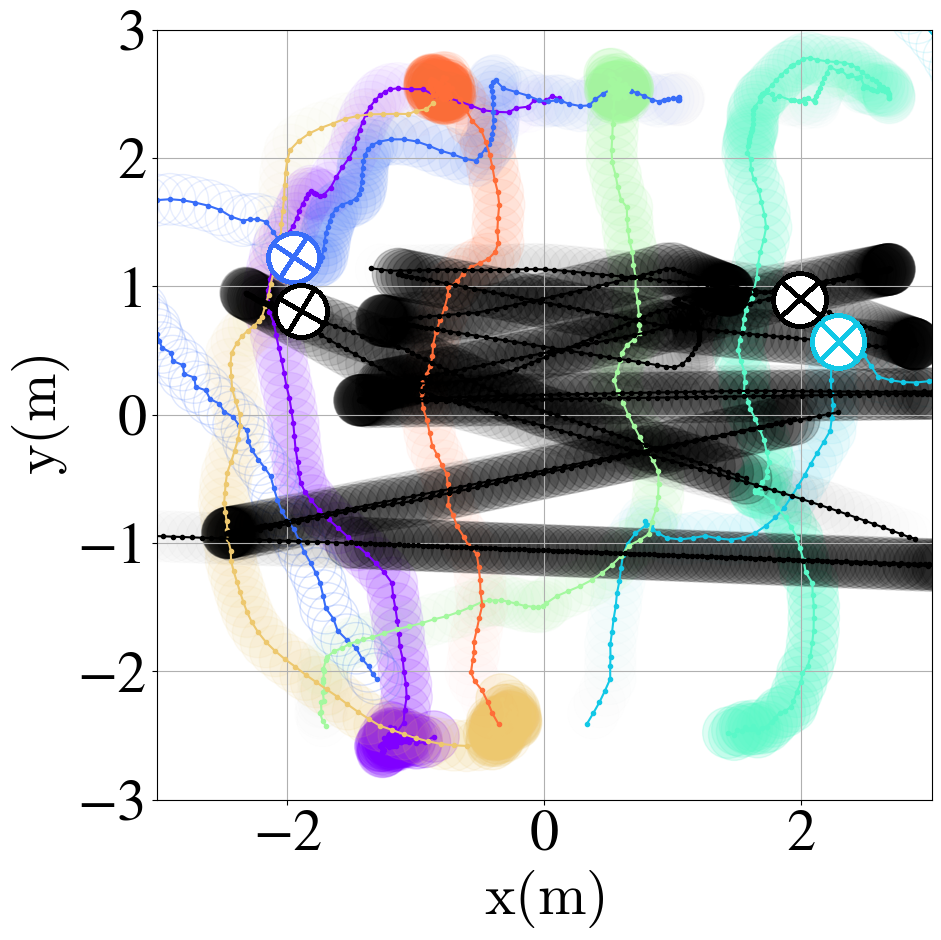} &
     \includegraphics[width=0.18\linewidth]{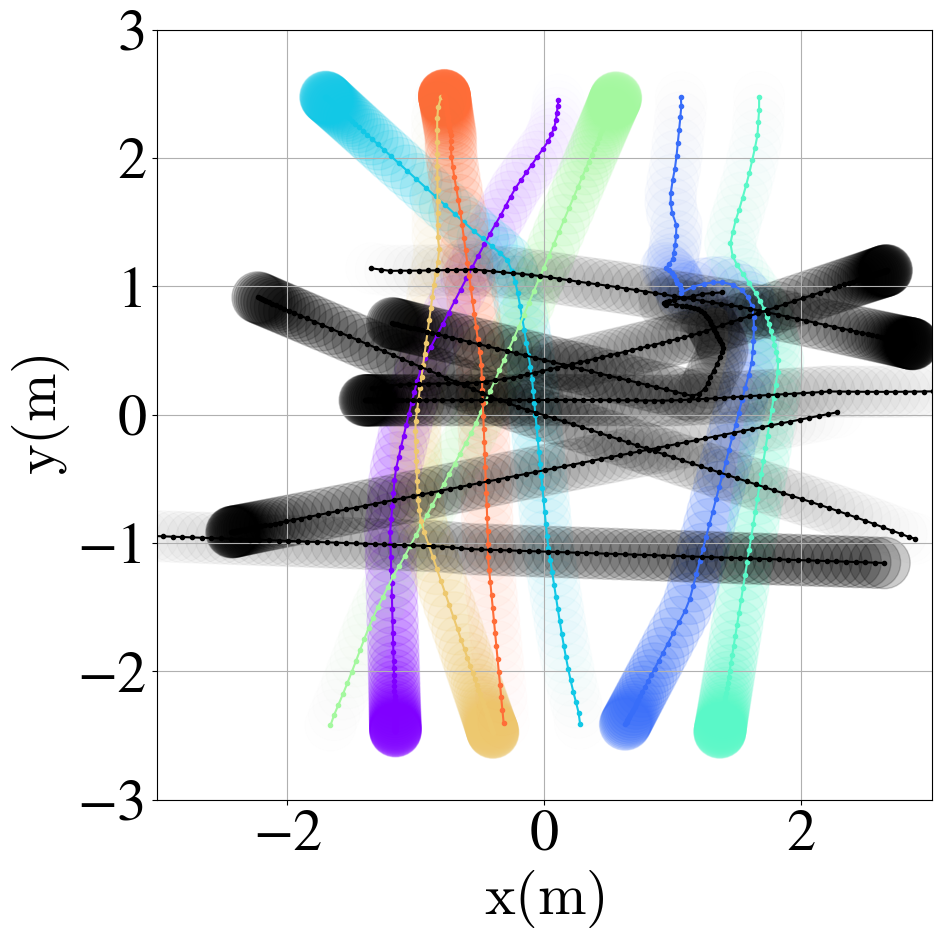}  \\
 \end{tabular}
 \caption{Mixed cooperative/non-cooperative \textit{circle} (top) and \textit{crossing} (bottom) scenario with 4 agents, 6 robots and different planners. Each robot is depicted with a different color. The representation follows the same rules in Fig.~\ref{fig:non-collaborative}. Collisions are marked with crosses surrounded by circles}.
 \label{fig:part-circles}
 \end{figure*}

We conduct a systematic series of \textit{circle} and \textit{crossing} runs to extract quantitative metrics to compare all the planners, using an increasing total number of agents. The circle size of \textit{circle} scenarios is set to $\max(2.5, \frac{2.3 \mathsf{N} r_s }{\pi})$m and the square size of \textit{crossing} scenarios is set to $1.5 \mathsf{N} r_s$m. These numbers guarantee that there is enough space for the initial and final position of all robots and agents while enforcing collision conflicts. Simulation runs range from $\mathsf{N}=10$ to $\mathsf{N}=25$ in steps of $3$. For each value of $\mathsf{N}$, we test five different proportions $\mathsf{P}$ of cooperative robots. In particular, we set the number of cooperative robots equal to $\lceil \mathsf{P} \mathsf{N}\rceil$, with $\mathsf{P} = \{0.01, 0.25, 0.5, 0.75, 1\}$, where $\lceil\bullet\rceil$ is the operator that rounds a real number to the closest greater integer. The others are non-cooperative agents. We run $128$ random scenarios for each combination of planner, number of cooperative robots and non-cooperative agents. Each run in \textit{circle} scenarios randomizes the identity of each agent (non-cooperative or the planner under evaluation), while each run in \textit{crossing} scenarios randomizes the initial and goal positions of each agent. We compare the state-of-the-art planners with 4 versions of \textsf{AVOCADO} with the following parameters:
\begin{itemize}
    \item \textsf{AVOCADO\_1}: Table~\ref{tab:parameter-values}.
    \item \textsf{AVOCADO\_2}: Table~\ref{tab:parameter-values} but $d_i=5$.
    \item \textsf{AVOCADO\_3}: Table~\ref{tab:parameter-values} but $b_i=1$ ($\frac{b_i}{d_i}=0.5)$.
    \item \textsf{AVOCADO\_4}: Table~\ref{tab:parameter-values} but $b_i=-1$ ($\frac{b_i}{d_i}=-0.5)$.
\end{itemize}

% We use four metrics. Let $\mathsf{M}$ be the number of runs; $\mathrm{succ}_{j,k} \in \{0,1\}$ the indicator that is equal to $1$ if $||\mathbf{p}_{r,j} - \mathbf{p}_{r,j}^*|| < \xi$ at the end of run $k$ or $0$ otherwise, where $\mathbf{p}_{r,j}$ is the position of robot $j$, $\mathbf{p}^*_{r,j}$ is the desired goal of robot $j$, and $\xi > 0$ is a small tolerance; $t_{j,k}>0$ the instant when robot $j$ reaches its goal at run $k$; and $a_{j,k}^{x,t}$ and $a_{j,k}^{y,t}$ the acceleration in both axis of robot $j$ at time $t$ and run $k$. Moreover, $\mathbf{p}_{r,j,k}^{t}$ is the position of robot $j$ at instant $t$ and run $k$. Then, the four metrics are :
% \begin{itemize}
%     \item Mean success rate: $\frac{1}{\mathsf{M}\mathsf{N}}\sum_{k=1}^{\mathsf{M}}\sum_{j=1}^{\mathsf{N}}
%     \mathrm{succ}_{j,k}$.
%     \item Mean time to goal: $\frac{1}{\mathsf{M}\mathsf{N}}\sum_{k=1}^{\mathsf{M}}\sum_{\mathrm{succ}_{j,k} = 1}^{\mathsf{N}} t_{j,k}$.
%     \item Mean roughness: $
% \sum_{k=1}^{\mathsf{M}}\sum_{\mathrm{succ}_{j,k} = 1}^{\mathsf{N}} \sum_t \frac{(a_{j,k}^{x,t})^2 + (a_{j,k}^{y,t})^2}{\mathsf{M}\mathsf{N}t_{j,k}}$.
%     \item Mean path length: $
%     \sum_{k=1}^{\mathsf{M}}\sum_{ \mathrm{succ}_{j,k} = 1}^{\mathsf{N}} \sum_t \frac{||\mathbf{p}^{t+1}_{r,j,k} - \mathbf{p}^{t}_{r,j,k}||}{\mathsf{M}\mathsf{N}t_{j,k}}$.
% \end{itemize}

We define the success rate as a evaluation metric. Let $\mathsf{M}$ be the number of runs; $\mathrm{succ}_{j,k} \in \{0,1\}$ is the indicator that is equal to $1$ if $||\mathbf{p}_{r,j} - \mathbf{p}_{r,j}^*|| < \xi$ at the end of run $k$ or $0$ otherwise (it is also zero when the robot $j$ at run $k$ collides), where $\mathbf{p}_{r,j}$ is the position of robot $j$, $\mathbf{p}^*_{r,j}$ is the desired goal of robot $j$, and $\xi > 0$ is a small tolerance. Then, the success rate is defined as 
$$\frac{1}{\mathsf{M}\mathsf{N}}\sum_{k=1}^{\mathsf{M}}\sum_{j=1}^{\mathsf{N}}
    \mathrm{succ}_{j,k}.$$

The success rates are in Fig.~\ref{fig:success-circle} and Fig. \ref{fig:success-square} for \textit{circle} and \textit{crossing} scenarios, respectively. 
\textsf{AVOCADO} outperforms all other approaches in success rate. \textsf{AVOCADO\_2}, with a higher value of $d_i$ than \textsf{AVOCADO\_1}, has a greater success rate in \textit{circle} scenarios where the number of cooperative robots is greater, as it is more sensitive to the introduced noise that breaks symmetries. \textsf{AVOCADO\_1} is more stable and has a better performance in partially cooperative scenarios. 
\textsf{AVOCADO\_3} presents a worse success rate due to the bias, that assumes low degrees of cooperation when it is often not true. In this sense, it is very interesting to see how \textsf{AVOCADO\_4}, with a bias towards great degrees of cooperation, presents the worst results among the \textsf{AVOCADO} versions in Fig.~\ref{fig:success-circle} unless in fully-cooperative cases. However, \textsf{AVOCADO\_4} achieves the best success rates in Fig.\ref{fig:success-square}. This is probably due to the fact that, in \textit{crossing} scenarios, the robot first faces non-cooperative perpendicular traffic, where being cautious is desirable; meanwhile, in \textit{circle} scenarios the robot faces agents coming in their same direction, leading to conflicts similar to the one previously seen in Fig.~\ref{fig:bias}b. \textsf{ORCA} experiences deadlocks from geometrical symmetries, so it fails in all cooperative \textit{circle} scenarios. The performance of \textsf{T-MPC} degrades when the number of cooperative agents increases, since, as it is observed in \cite{mavrogiannis2022winding,poddar2023crowd}, the method is suited for scenarios with less than 10 robots or agents. \textsf{RVO-RL}, tuned for fully-cooperative environments, achieves its best performance for $\mathsf{P}=1$, decreasing the success rate as non-cooperative agents appear. \textsf{SARL}, due to the non-fluid behavior, obtains poor success rates. 
\begin{figure}
 \centering
\includegraphics[width=0.9\linewidth]{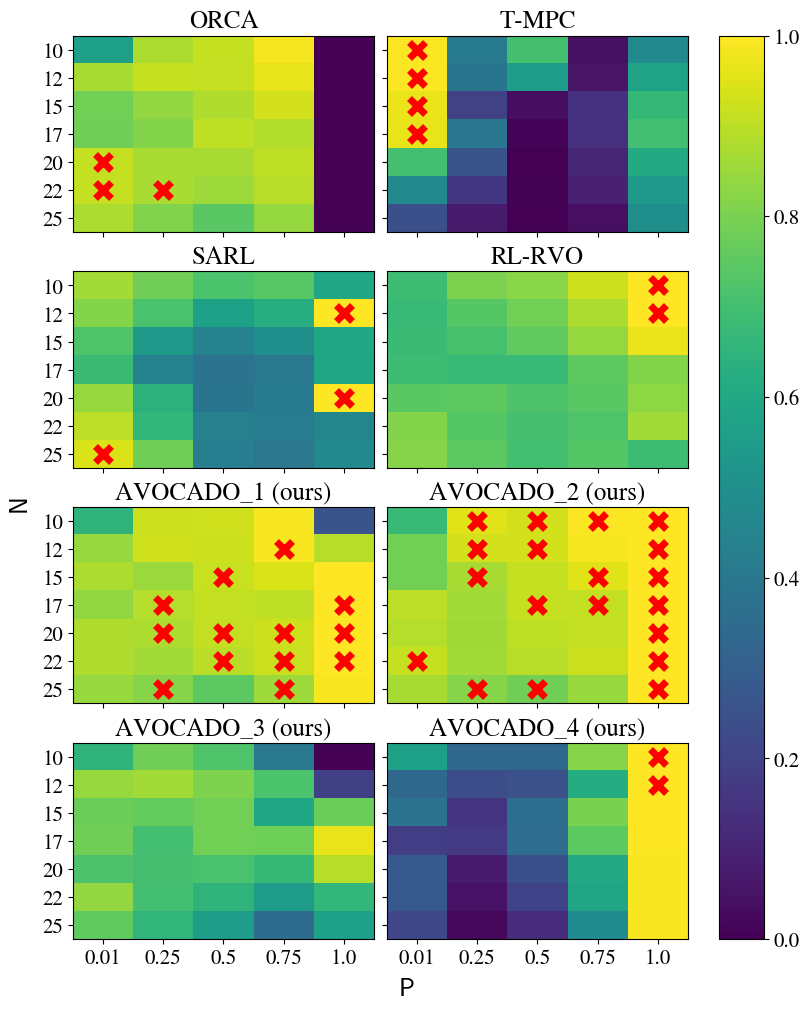} 
 \caption{Success rates of robots with the different planners in \textit{circle} scenarios. Red crosses mark the best planner. The color bar refers to success rate as defined in Section \ref{subsec:multiagent}, whereas $\mathsf{P}$ refers to the proportion of cooperative robots. The yellower the better.}
 \label{fig:success-circle}
 \end{figure}

\begin{figure}
 \centering
\includegraphics[width=0.9\linewidth]{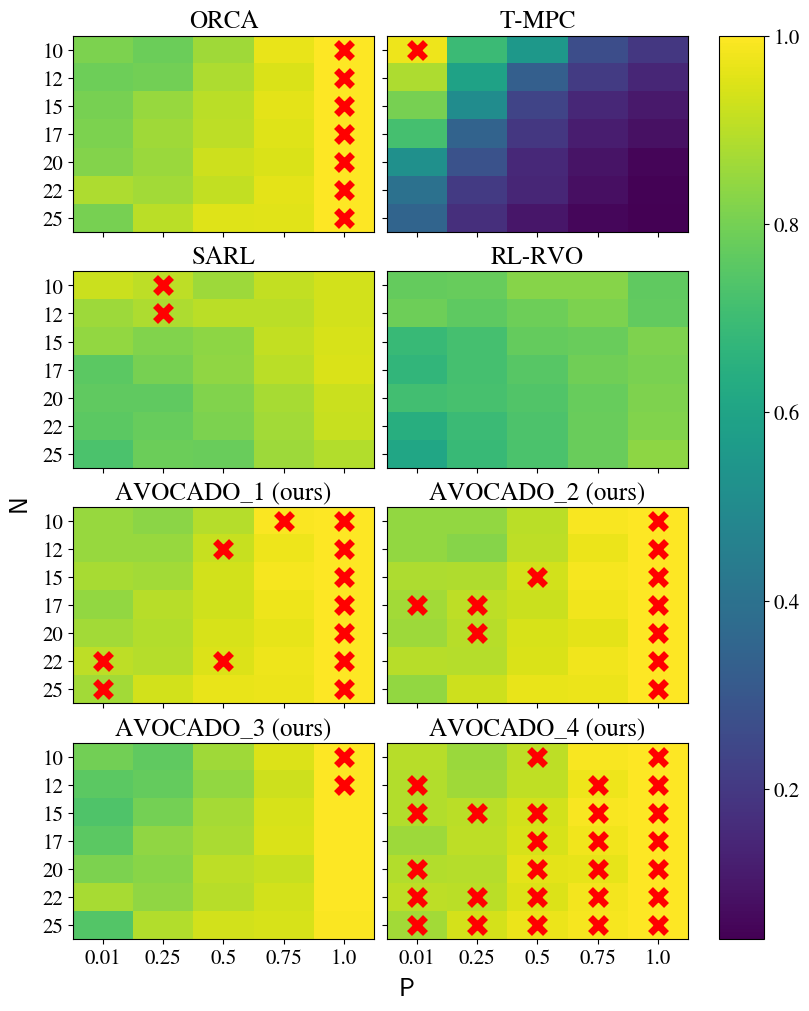} 
 \caption{Success rates of robots with the different planners in \textit{crossing} scenarios. The color bar refers to success rate as defined in Section \ref{subsec:multiagent}, whereas $\mathsf{P}$ refers to the proportion of cooperative robots. The yellower the better.}
 \label{fig:success-square}
 \end{figure}

We show in Fig.~\ref{fig:nav-times} the mean navigation time that every planner takes to reach the goal in \textit{square} scenarios with \mbox{$\textsf{P}=0.5$}. The mean time to goal is computed taking into account only the robots that do not collide: 
$$\frac{1}{\mathsf{M}\mathsf{N}}\sum_{k=1}^{\mathsf{M}}\sum_{\substack{j=1 \\ \mathrm{succ}_{j,k} = 1}}^{\mathsf{N}} t_{j,k},$$
where $t_{j,k}>0$ denotes the instant when robot $j$ reaches its goal at run $k$. The results are aligned with the qualitative results depicted in Figs.~\ref{fig:non-collaborative}, \ref{fig:col-circles}, and \ref{fig:part-circles}). \textsf{RL-RVO} and \textsf{SARL} exhibit the largest time to goal because their trajectories are irregular and present many detours. Meanwhile, the rest of the planners reach the goal in similar times. \textsf{T-MPC} manifests slightly lower times to reach the goal even tough the resulting trajectories are longer than those of \textsf{AVOCADO} or \textsf{ORCA}. This is due to the fact that \textsf{T-MPC} exerts higher velocities and it is, therefore, more risky, which explains its success rate metrics shown in Figs. \ref{fig:success-circle} and \ref{fig:success-square}.

\begin{figure}
    \centering
    \includegraphics[width=1\columnwidth]{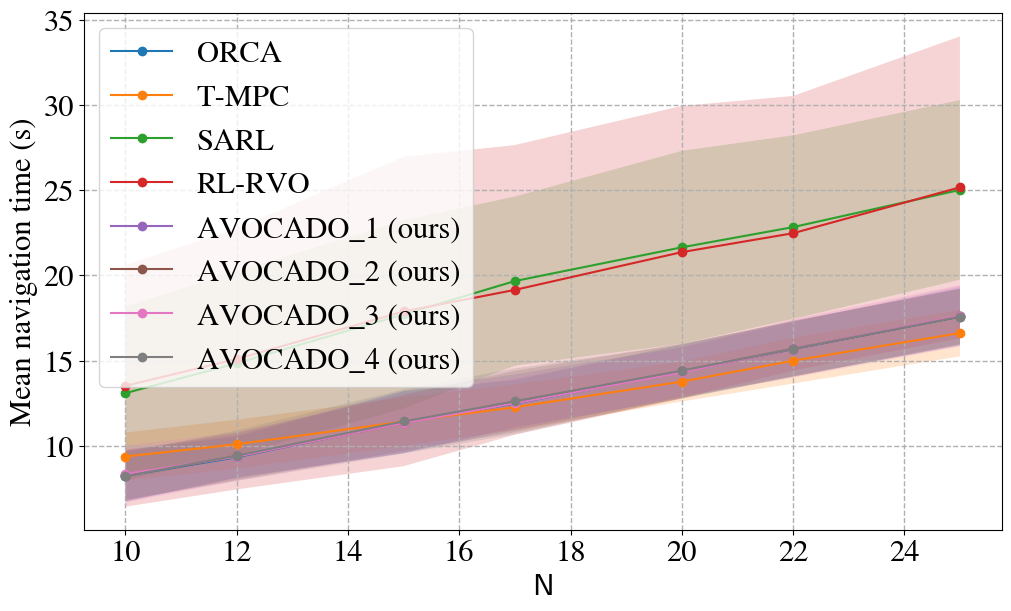}
    \caption{Mean and standard deviation of navigation times (in seconds) of successful robots for the different planners, gathered in 128 \textit{square} scenarios with $\mathsf{P}=0.5$}
    \label{fig:nav-times}
\end{figure}

Finally, regarding computation times, Fig. \ref{fig:times} collects the mean and standard deviation computation times for all the planners under comparison. The pure geometrical methods (\textsf{AVOCADO} and \textsf{ORCA}) share the same inexpensive computational cost, requiring hundreds of milliseconds to compute a solution in crowded environments. The time increases with the number of robots and agents since the number of entities under consideration grow, but this growth is linear with the number of agents and robots within the sensor range. The other methods take, by orders of magnitude, much more time to compute their navigation commands. This computational burden may prevent their use in real hardware applications with constrained resources, specially if there are other higher-level tasks that use those resources.

\begin{figure}
    \centering
    \includegraphics[width=1\columnwidth]{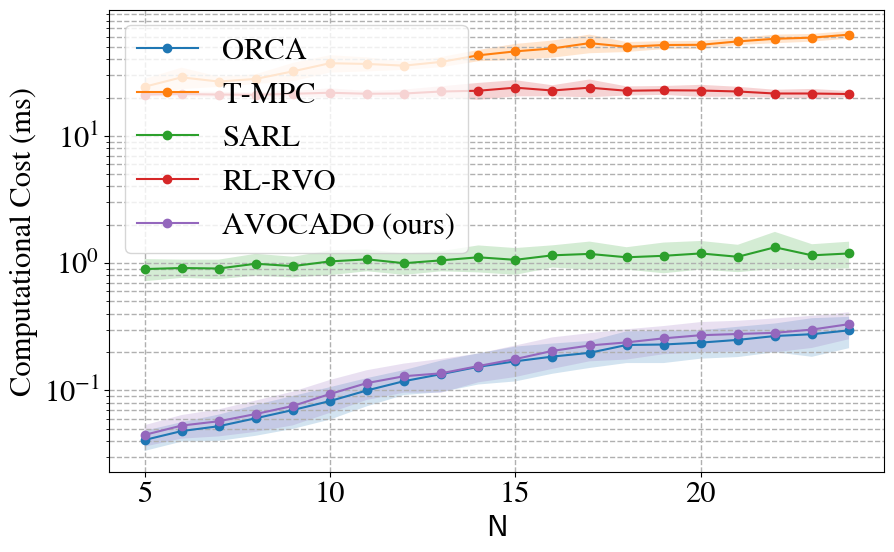}
    \caption{Mean and standard deviation of the computational times (in milliseconds) for the different planners. Y-axis is in a logarithmic scale.}
    \label{fig:times}
\end{figure}

\subsection{Static obstacles}\label{subsec:static_obstacles} 

As real-world scenarios often include other kind of obstacles, we study the behavior of \textsf{AVOCADO} in the presence of static obstacles. We followed the same approach of previous \textsf{VO}-based methods~\cite{van2008reciprocal,van2011reciprocal,rufli2013reciprocal,alonso2013optimal,alonso2018cooperative}, which consider representing the boundaries of the obstacles as segments. Half-spaces constraining the velocities that lead to collision with these segments are directly included in the linear problem optimization. In this way, collision avoidance with static obstacles is naturally achieved with no extra cost.

To verify this insight, we conduct simulated experiments where robots following \textsf{AVOCADO} navigate in a circular scenario that includes static obstacles. Fig.~\ref{fig:static} shows that the robots are able to adjust their trajectories to obstacles of different shapes.

 \begin{figure}
 \centering
 \begin{tabular}{@{}cc@{}}
  \footnotesize{a)} & \footnotesize{b)} \\
  \includegraphics[width=0.47\columnwidth]{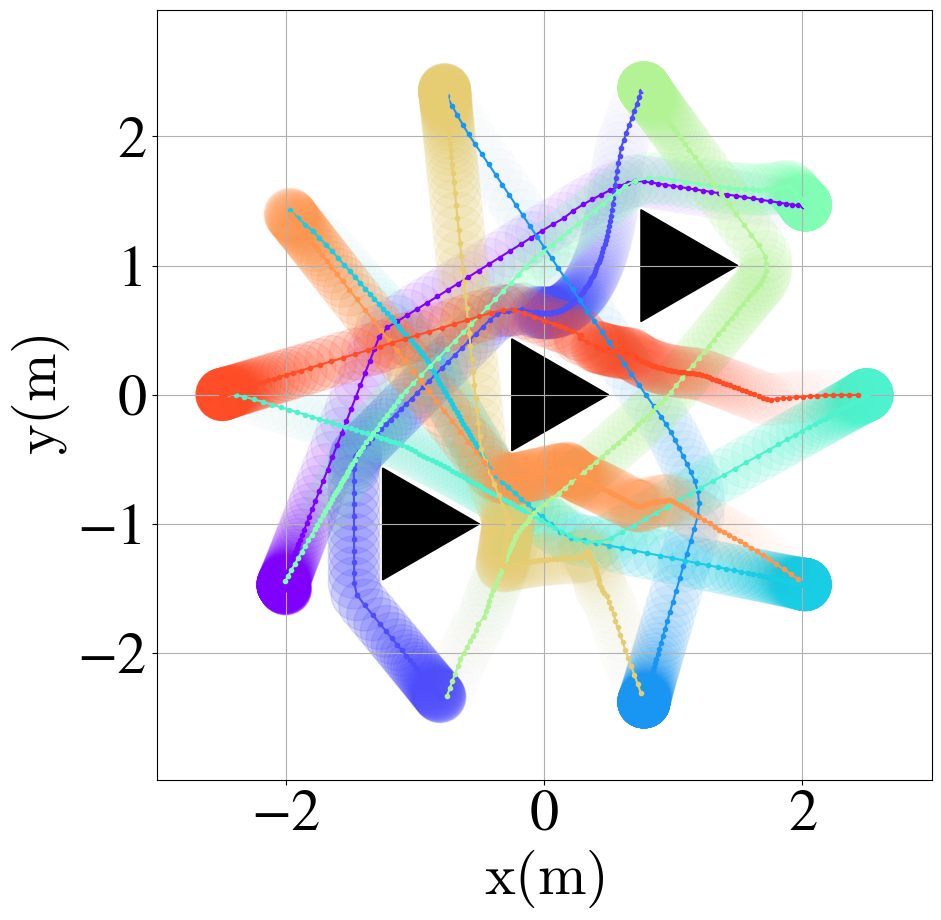} & 
  \includegraphics[width=0.47\columnwidth]{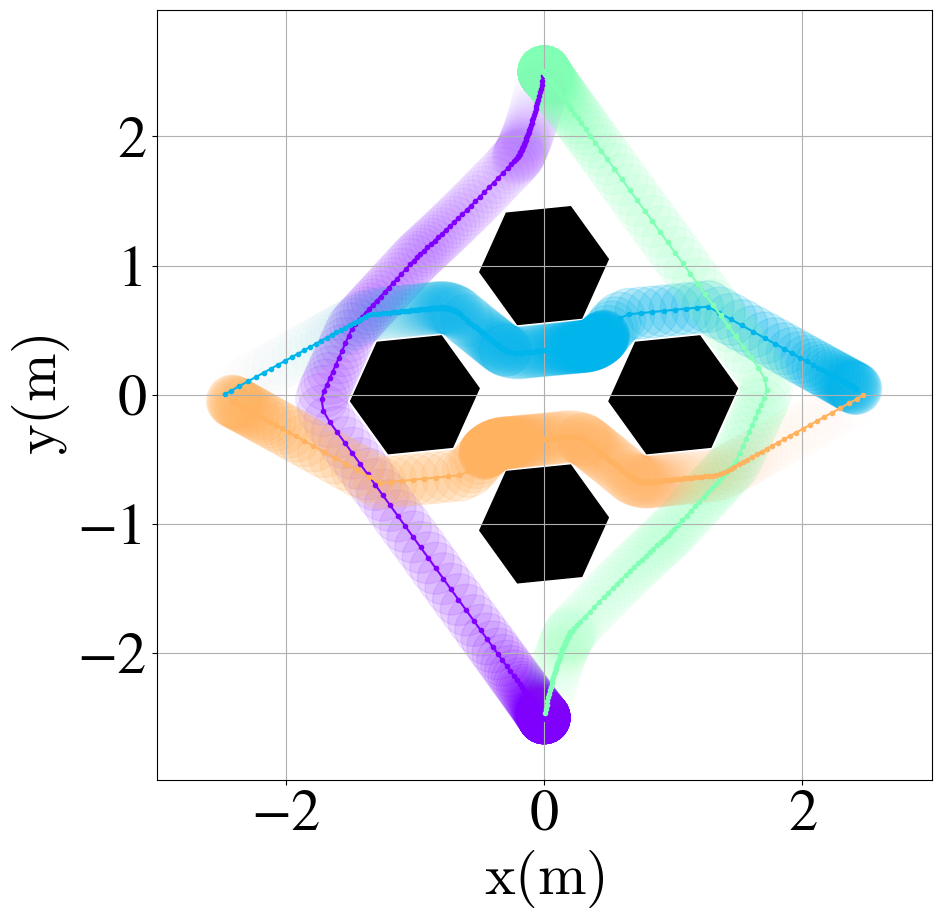}  \\
 \end{tabular}
 \caption{Multi-robot experiments with static obstacles.}
 \label{fig:static}
 \end{figure}

Nevertheless, as previously stated, an important remark is that \textsf{AVOCADO} is a planner that locally selects a non-colliding velocity. Therefore, an additional high-level planner would be needed to navigate in the presence of non-convex obstacles, such as in a maze-like environment.

\section{Experimental results}\label{sec:experiments}
After evaluating \textsf{AVOCADO} in simulated environments, in this section we conduct experiments with real ground robots. We use three Turtlebot 2 robotic platforms that use \textsf{AVOCADO}, and up to three pedestrians as external non-cooperative agents. The experiments involve 19 pedestrians, 15 of them external to the project associated to this work. The pedestrians, as the robots, have a fixed starting point and final goal for each experiment, but they have no instructions on how to behave and interact with the robots. In this way, the pedestrians decide by themselves their degree of cooperation with the robots and other pedestrians, their own velocity and the trajectory to follow to reach the goal, making each experiment different and unpredictable from the perspective of the robots. The robots use \textsf{AVOCADO} with the same parameterization as the default ones detailed in Table \ref{tab:parameter-values}. The purpose is to prove that \textsf{AVOCADO} presents zero-shot-transfer capabilities.

We conduct the experiments in an arena of $6\times6$m. We used an Optitrack Prime$^X$ 13W system of markers and 12 cameras to localize the robots and the pedestrians. We also design and implement an Extended Kalman Filter with a constant velocity assumption to track the positions and velocities of all the robots and pedestrians.

We design three representative experimental scenarios, running them using different combinations of pedestrians, leading to a total of 
33 experiments. Videos of the experiments can be found in the supplementary material. The three scenarios are as follows:
\begin{itemize}
    \item \textbf{Head-on}: two robots are placed next to each other in two corners of the arena, and two pedestrians are placed in the other two corners of the arena, each of them facing one of the robots. The goal of the robots and the pedestrians is to exchange their position with the pedestrian or robot that is in front of them.
    \item \textbf{Circle}: as in the simulations, pedestrians and robots are arranged in an evenly spaced circular formation, alternating robots and pedestrians. The goal of all the players is to go to the opposite side of the circle.
    \item \textbf{Crossing}: two robots are located in the medium point of two opposite sides of the arena, facing each other. In one of the other sides we place a pedestrian and a robot, while two pedestrians are placed in the remaining side of the arena. The team robot-pedestrian and the team pedestrian-pedestrian are told to cross the arena towards the position of the other team. The goal of the first two robots is to cross the perpendicular traffic flow to exchange their position. 
\end{itemize}

Fig.~\ref{fig:groups-exp0} shows two examples of head-on experiments. In the first experiment (a) one of the pedestrians (bottom) decides to evade the robot. \textsf{AVOCADO} detects that and adapts to the situation, following a straight trajectory. Reciprocal dance problem is not observed. The other pedestrian in Fig.~\ref{fig:groups-exp0}a slightly modifies the trajectory to cooperate with the robot. Fig.~\ref{fig:groups-exp0}b shows the same experiment with other pedestrians, that decide to not cooperate at all, demonstrating how the robots are able to take all the responsibility to do the collision avoidance maneuver.
 
 \begin{figure}
 \centering
 \begin{tabular}{@{}cc@{}}
  \footnotesize{a) Group 2} & \footnotesize{b) Group 4} \\
  \includegraphics[trim={2cm 10cm 15cm 7cm}, clip, width=0.47\columnwidth]{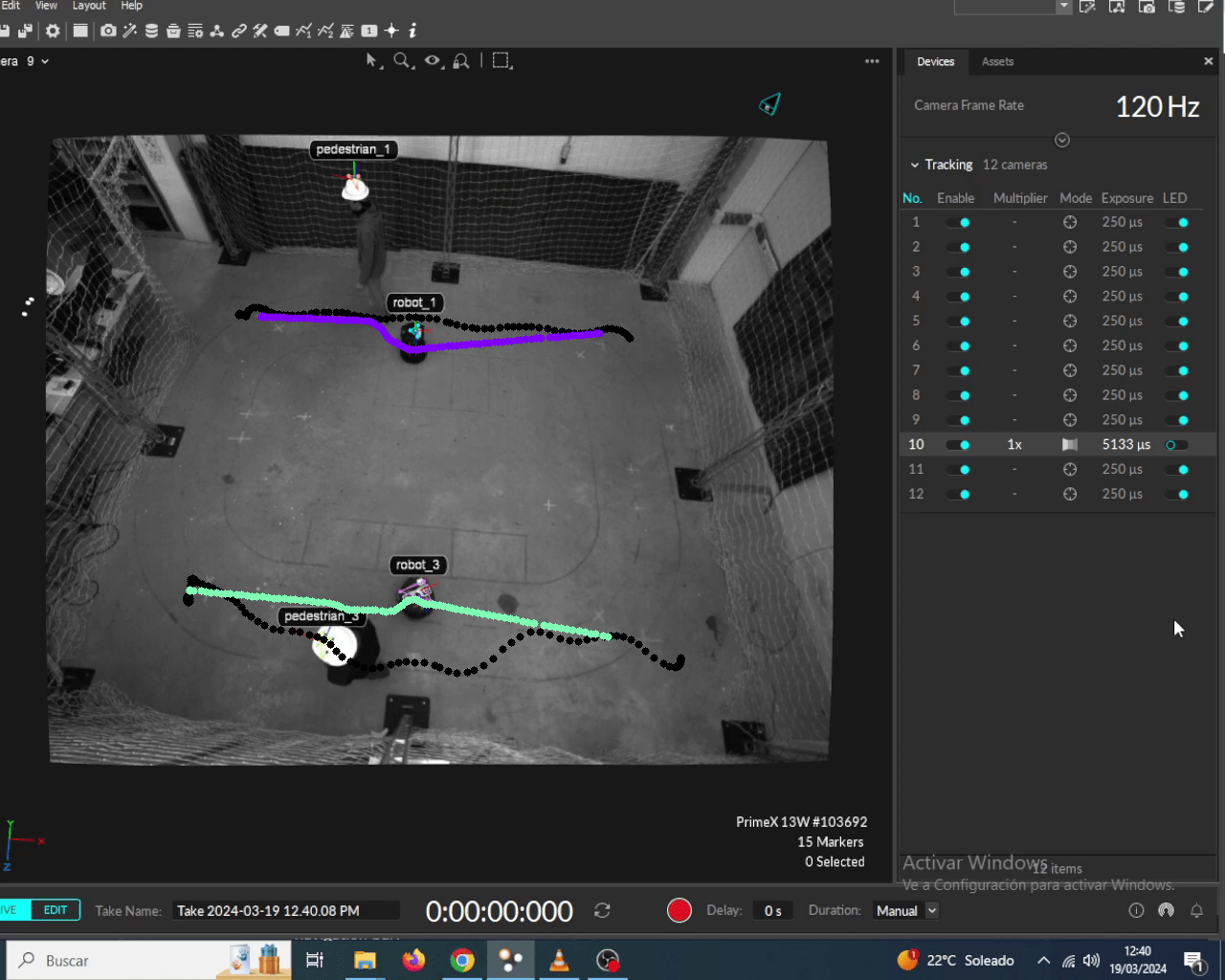} & 
  \includegraphics[trim={2cm 10cm 15cm 7cm}, clip, width=0.47\columnwidth]{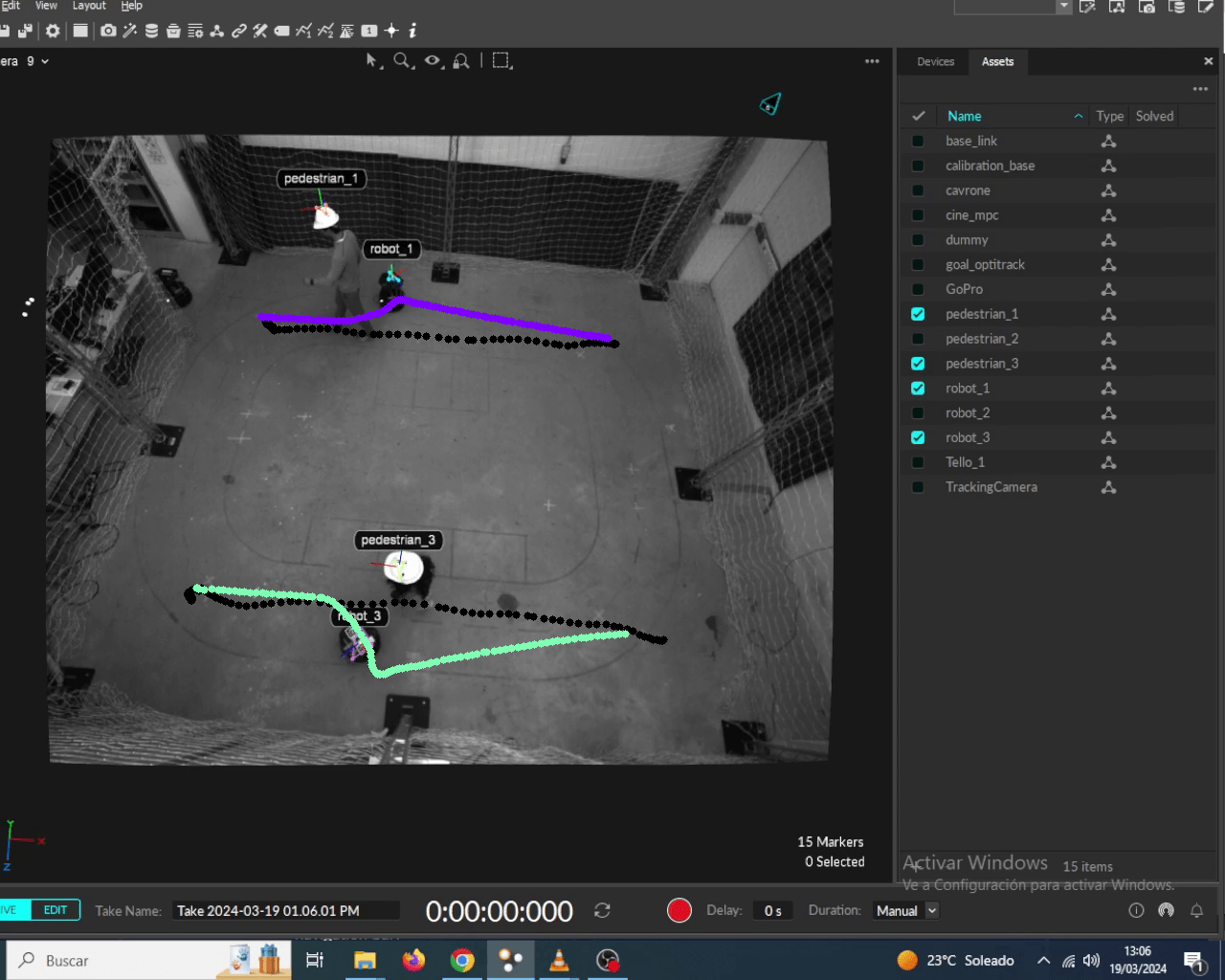}  \\
 \end{tabular}
 \caption{Head-on experiment, groups 2 and 4. Two robots face two pedestrians, exchanging their initial position. The pedestrians trajectories are represented in black and the robots with different colors.}
 \label{fig:groups-exp0}
 \end{figure}

In Fig.~\ref{fig:groups-exp1}a, pedestrians participating in the circle experiment walked with different velocities. When the person starting in the down-right leaves the intersection, the person starting in the top-center is still starting the trajectory. The robots adapt to the different velocities and safely reach the goal with fluid trajectories. The purple robot traverses the intersection moving to its right side, as there is free space there. Fig.~\ref{fig:groups-exp1}b depicts a situation where the human starting in the down-right and the one starting in the top-center have similar behaviors as in the previous experiment in terms of speed. The pedestrian starting in the bottom-left, however, is faster than its previous homologous. The robots accomodate to this situation by being more cooperative and more reactive (yellow and blue) or choosing a different trajectory in the space that pedestrian is leaving behind (purple).

 \begin{figure}
 \centering
 \begin{tabular}{@{}cc@{}}
  \footnotesize{a) Group 3} & \footnotesize{b) Group 5} \\
  \includegraphics[trim={2cm 10cm 15cm 7cm}, clip, width=0.47\columnwidth]{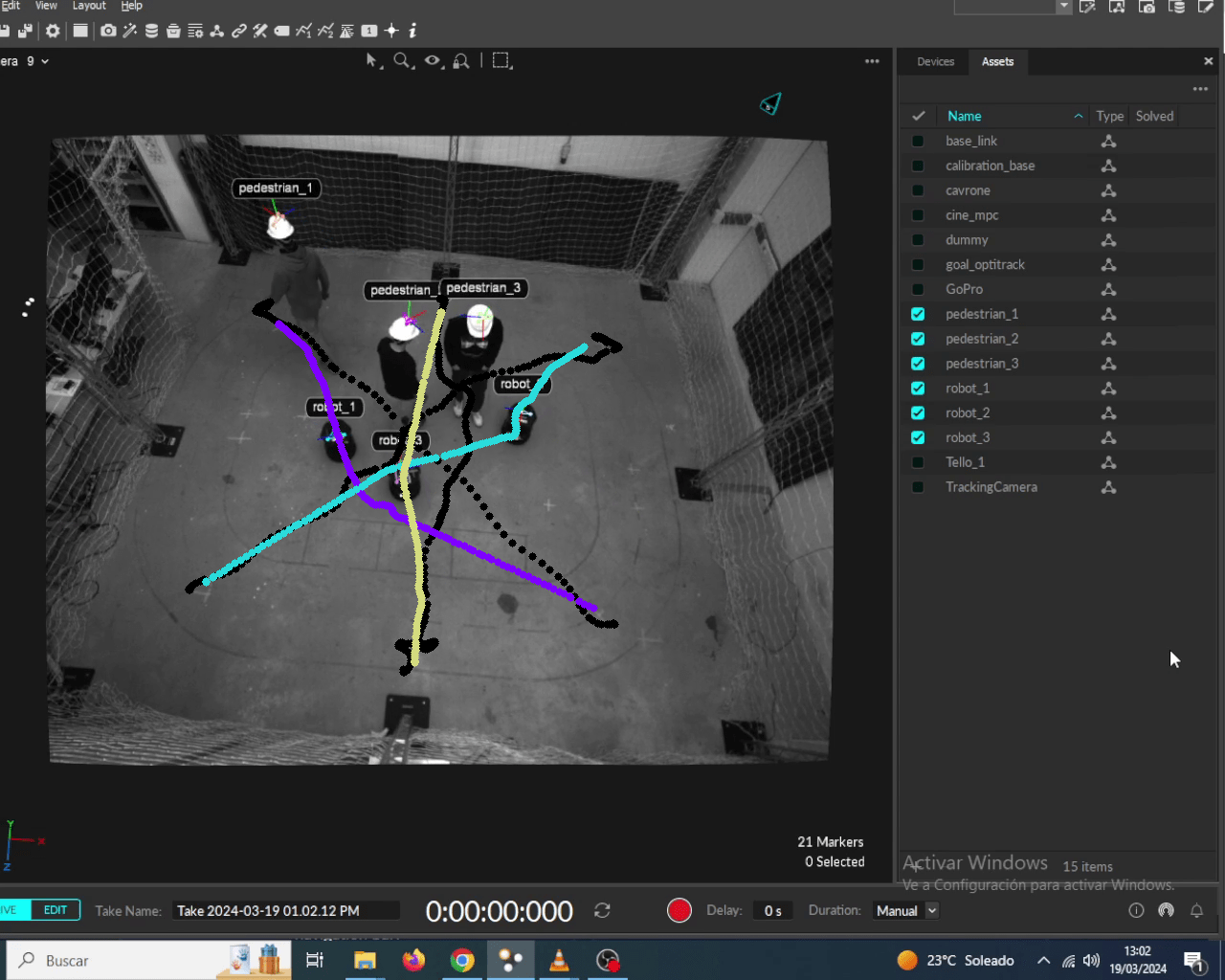} & 
  \includegraphics[trim={2cm 10cm 15cm 7cm}, clip, width=0.47\columnwidth]{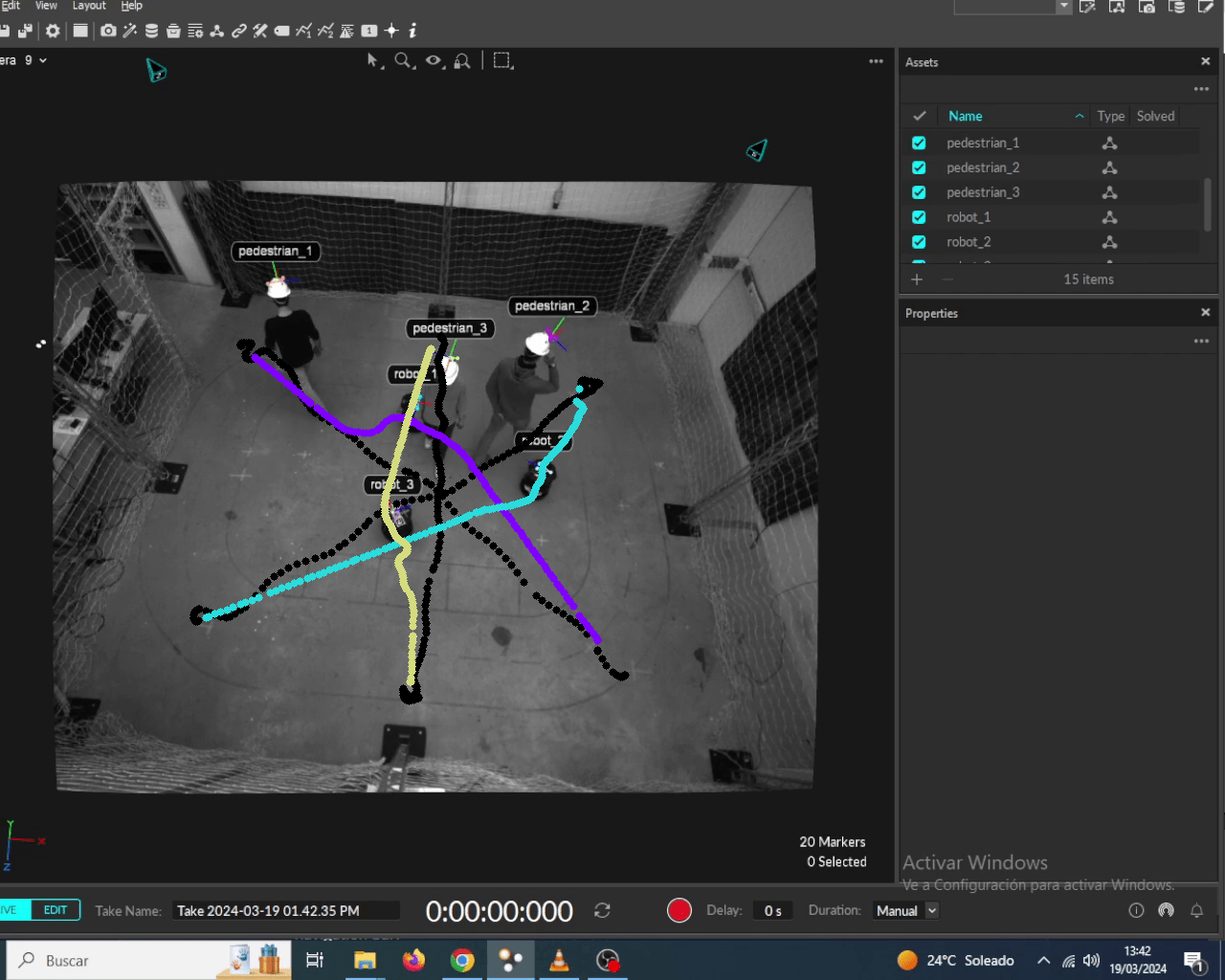}  \\
 \end{tabular}
 \caption{Circle experiment, groups 3 and 5. Three robots and three pedestrians, evenly spaced in a circle, exchange their positions with the opposite player.}
 \label{fig:groups-exp1}
 \end{figure}

A standard crossing scenario is represented in Fig.~\ref{fig:groups-exp2}a. All the agents forming the traffic in the middle of the arena respect their lane. The two robots in opposite sides choose a side to cross the intersection, passing between the two nearby agents and the two far ones. The girl starting behind the purple robot overtakes it safely, through the open space that is in their left. The robot collaborates but giving her more space to do it. Fig.~\ref{fig:groups-exp2}b, however, shows a different situation. The pedestrian behind the other pedestrian decides to accelerate and overtake through the central lane. The purple robot has to slow down to evade a frontal collision. This delay in the central traffic makes the blue robot avoid it through the right side, and the yellow adapt to it as it arrives there.

Overall, extensive hardware experiments with multiple robots and humans prove that \textsf{AVOCADO} is effective navigation strategy for collision avoidance in mixed cooperative/non-cooperative environments. \textsf{AVOCADO} is a zero-shot approach since no further tuning is required to transfer the algorithm from simulations to real robots. In this sense, \textsf{AVOCADO} preserves the properties observed in simulations in terms of success rate, fluidity, adaptation to unknown degrees of cooperation and computational efficiency.

 \begin{figure}
 \centering
 \begin{tabular}{@{}cc@{}}
  \footnotesize{a) Group 4} & \footnotesize{b) Group 5} \\
  \includegraphics[trim={2cm 10cm 15cm 7cm}, clip, width=0.47\columnwidth]{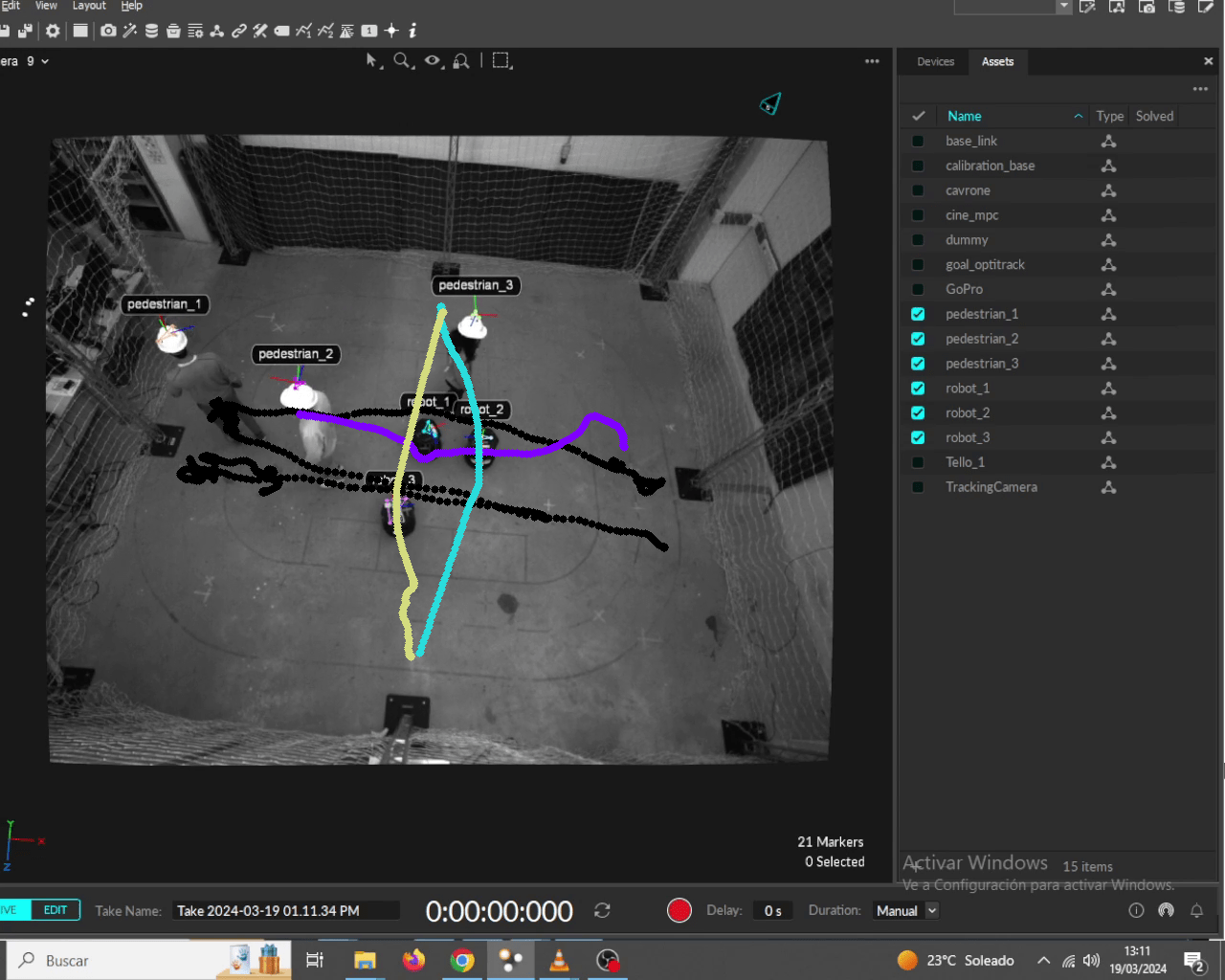} & 
  \includegraphics[trim={2cm 10cm 15cm 7cm}, clip, width=0.47\columnwidth]{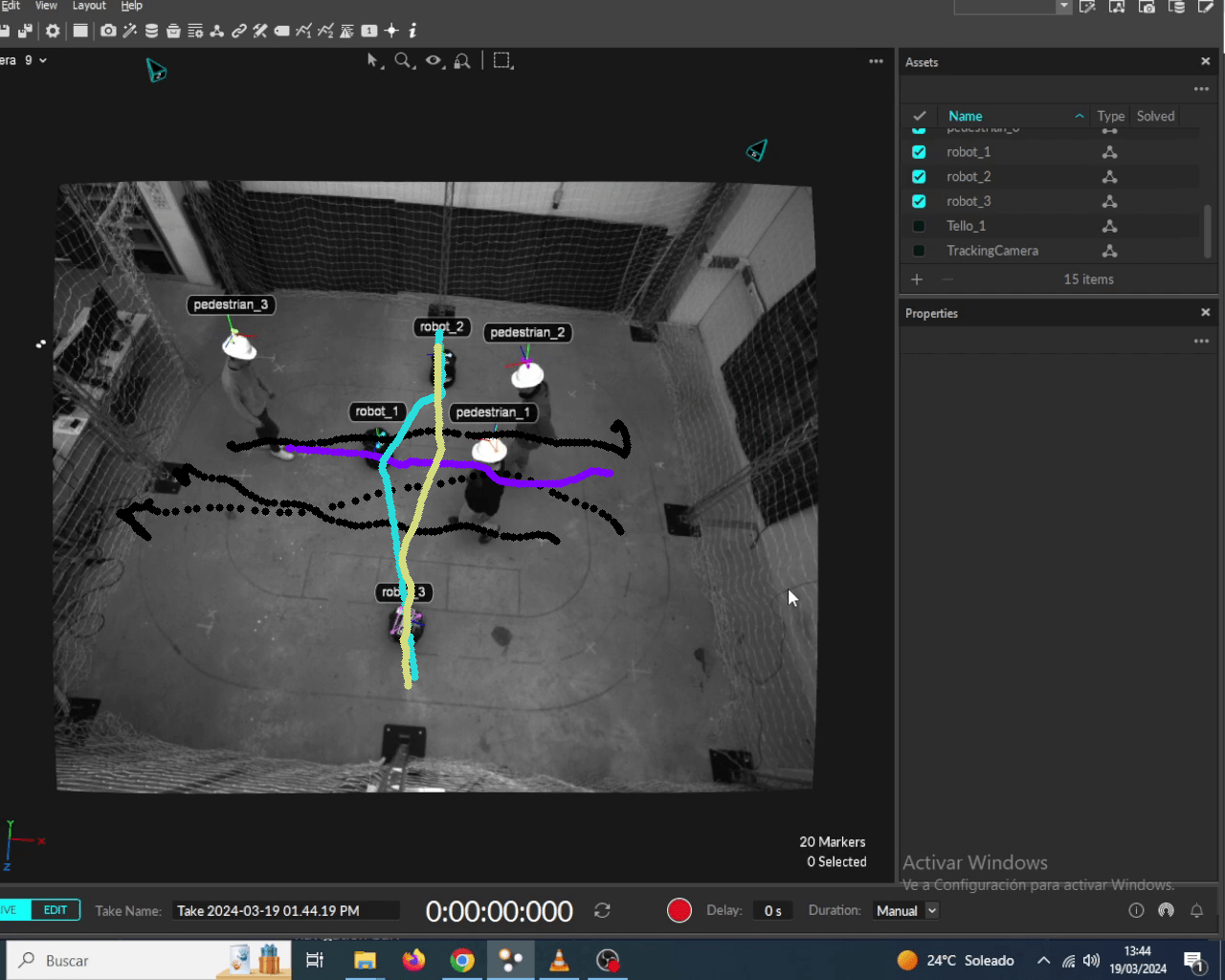}  \\
 \end{tabular}
 \caption{Crossing experiment, groups 4 and 5. Two robots, in opposite sides, exchange positions by crossing an intersection with one robot and three pedestrians.}
 \label{fig:groups-exp2}
 \end{figure}

%%%%%%%%%%%%%%%           
% CONCLUSIONS %
%%%%%%%%%%%%%%%

\section{Conclusions}\label{sec:conclusion}
In this work we have presented \textsf{AVOCADO}, a novel approach to solve the robot collision avoidance problem in mixed cooperative/non-cooperative multi-agent environments, where the degree of cooperation of the agents is unknown. \textsf{AVOCADO} parameterizes the degree of cooperation of the agents using a scalar value that is adapted in real-time using a novel adaptive law based on nonlinear opinion dynamics. We have shown that the adaptive law, under an appropriate tuning, guarantees robot decision on 
the degree of cooperation of the agent before collision. To do so, 
we first exploited the geometry of the problem to develop a novel attention mechanism that depends on the expected time to collision. The attention mechanism is also used to propose a novel method to overcome deadlocks generated by symmetries. The second key to achieve adaptation involved a novel geometrical estimator that predicts the opinion that the agent has on the adapted degree of cooperation of the robot. The estimator only relies on the measured position and velocity of the agent. Finally, \textsf{AVOCADO} integrates the adapted degree of cooperation in a linear program that minimizes the difference between desired and actual velocity while avoiding collision.

Simulated experiments have validated \textsf{AVOCADO}, showing that it is inexpensive to compute and overcomes existing approaches in terms of success rate, number of avoided collisions, computational time and time to reach the goal. In this sense, our solution is readily implementable in low-cost hardware devices, leaving the computational and memory resources free for any desired high-level task such as semantic mapping, autonomous tracking or package delivery. Moreover, extensive experiments with multiple robots and humans have demonstrated that \textsf{AVOCADO} presents zero-shot-transfer capabilities and works in scenarios with uncertainties, noise and evolving unpredictable human behaviors. 

As future work, one direction is to explore the integration of \textsf{AVOCADO} in a higher-level planner that avoids dense crowded areas or convex obstacles. \textsf{AVOCADO}, since it does not have an horizon planner, sometimes enter in conflicting situations that lead to local minima or unavoidable collisions. Another research direction consists in extending \textsf{AVOCADO} to general robot dynamic models, as other \textsf{VO}-based planners~\cite{alonso2013optimal, alonso2018cooperative}; and to 3-D settings. Lastly, it would be interesting to develop a method that, by just relying on onboard sensors, is able to consider how neighboring interactions are affecting the motion of the nearby agents, in order to distinguish non-cooperative behaviors stemming from the degree of cooperation of the agents from the ones enforced by their own collision avoidance goals.

%%%%%%%%%%%%%%%%%%%%
%%%% APPENDICES %%%%
%%%%%%%%%%%%%%%%%%%%

% \appendices
% \section{Extended simulated results}\label{appendix:results}
% \input{60_Appendix}

% \section{Otra version de resultados cuantitativos}\label{appendix:cuantitatives}
% \input{70_Posible_cambio_figuras}

% \section{Soft Actor-Critic Hyperparameters}\label{sec:sac_parameters}
% \input{70_SAC}

%%%%%%%%%%%%%%           
% REFERENCES %
%%%%%%%%%%%%%%
%\balance
\bibliographystyle{IEEEtran}
\bibliography{IEEEabrv,IEEEexample.bib}

%%%%%%%%%%%%%%%         
% BIOGRAPHIES %
%%%%%%%%%%%%%%%
\newpage

\begin{IEEEbiography}[{\includegraphics[width=1in,height=1.25in,clip,keepaspectratio]{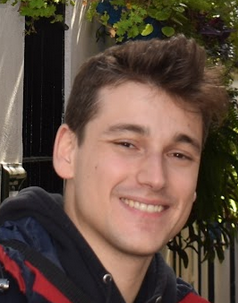}}]
{Diego Martinez-Baselga} (Graduate Student Member, IEEE) received the B.Eng. degree in computer science and the M.Sc. degree in Robotics, Graphics and Computer Vision from the University of Zaragoza, Spain, in 2020 and 2022, respectively, where he is working toward the Ph.D. degree in computer science and systems engineering. He was a visiting researcher at TU Delft, Netherlands, in 2023, in the Autonomous Multi-Robots Lab. His research interests include motion planning and collision avoidance using learning and control methods, and multi-robot systems.
\end{IEEEbiography}

\begin{IEEEbiography}[{\includegraphics[width=1in,height=1.25in,clip,keepaspectratio]{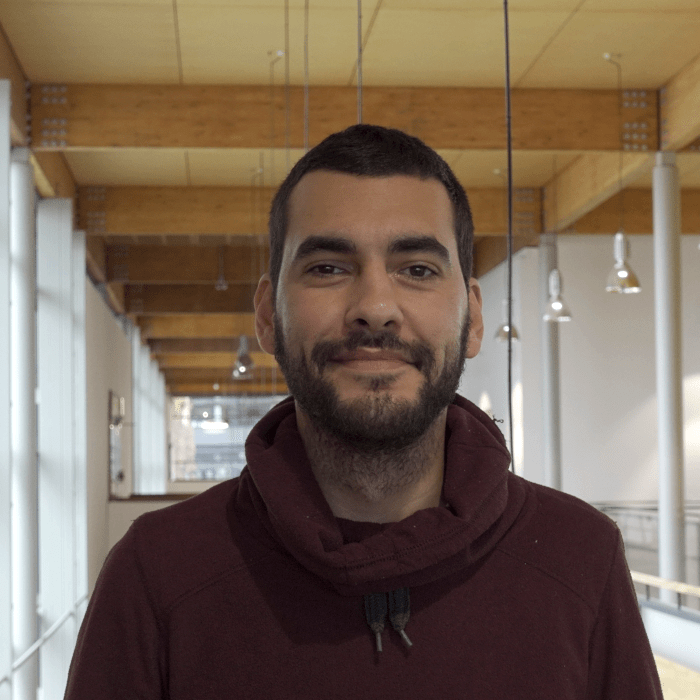}}]
{Eduardo Sebasti\'{a}n} (Graduate Student Member, IEEE) received the B.Eng. (Hons.) degree in electronic and automatic engineering and the M.Eng. (Hons.) degree in electronics in 2019 and 2020, respectively, from the Universidad de Zaragoza, Zaragoza, Spain, where he is currently working toward the Ph.D. degree in computer science and systems engineering. He has been a visiting scholar at the University of California San Diego, USA, in 2022 and 2024. His research interests include nonlinear control and distributed multirobot systems. Eduardo is a Fulbright Scholar and a DAAD AInet Fellow.
\end{IEEEbiography}

\begin{IEEEbiography}[{\includegraphics[width=1in,height=1.25in,clip,keepaspectratio]{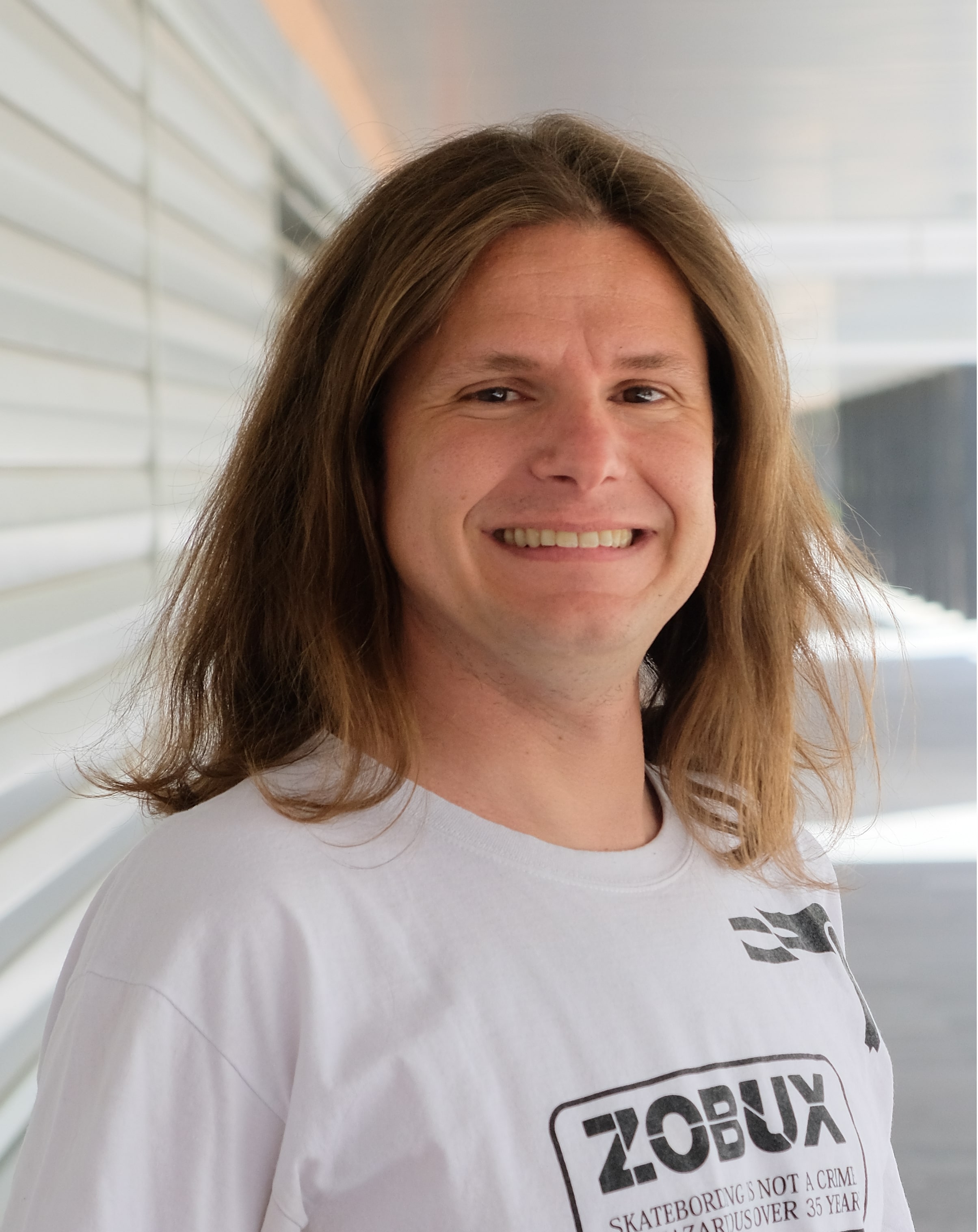}}]{Eduardo Montijano}
(Member, IEEE) received the M.Sc. degree in computer science and the Ph.D. degree in robotics and systems engineering from the Universidad de Zaragoza, Zaragoza, Spain, in 2008 and 2012, respectively.
From 2012 to 2016, he was a faculty member with Centro Universitario de la Defensa, Zaragoza.

He is currently an Associate Professor with the Departamento de Inform\'atica e Ingenier\'ia de Sistemas, Universidad de Zaragoza. His current research focuses on the design of scene understanding algorithms for planning and control of multiple robots.
\end{IEEEbiography}

\begin{IEEEbiography}[{\includegraphics[width=1in,height=1.25in,clip,keepaspectratio]{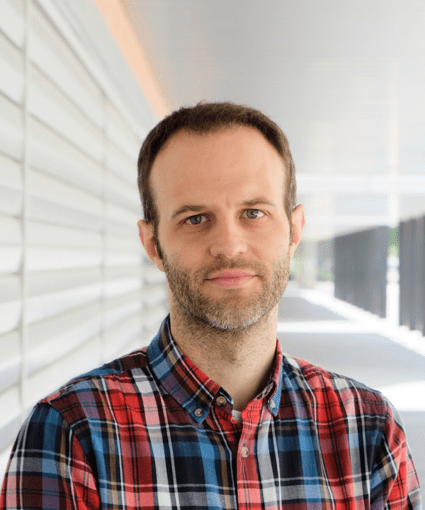}}]{Luis Riazuelo} (Member, IEEE) obtained his MsC in Computer Science in 2006 and his MEng in Biomedical Engineering in 2008 from the University of Zaragoza. In 2018 he obtained his PhD in Computer Science at the University of Zaragoza. He is currently an associate profesor at the Computer Science and Systems Engineering department of the University of Zaragoza. Since 2007 he is researcher of the Robotics, Perception and Real Time group at the Arag\`on Engineering Research. His current research topics include scene understanding for mobile robotics in underground environments and topological semantic mapping and localization in intracorporeal medical scenes.
\end{IEEEbiography}

\begin{IEEEbiography}[{\includegraphics[width=1in,height=1.25in,clip,keepaspectratio]{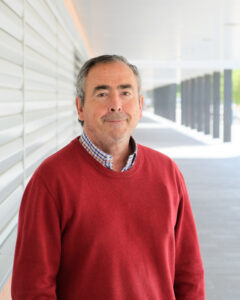}}]{Carlos Sag\"{u}\'{e}s} (Senior Member, IEEE) received the M.Sc. degree in computer science and systems
engineering and the Ph.D. degree in industrial engineering from the University of Zaragoza, Zaragoza, Spain, in 1989 and 1992, respectively.

In 1994, he joined as an Associate Professor with the Departamento de Informatica e Ingenieria de Sistemas, University of Zaragoza, where he became a Full Professor in 2009 and also the Head Teacher. He was engaged in research on force and infrared sensors for robots. His current research interests include control systems and industry applications, computer vision, visual control, and multi-vehicle cooperative control.
\end{IEEEbiography}

\begin{IEEEbiography}[{\includegraphics[width=1in,height=1.25in,clip,keepaspectratio]{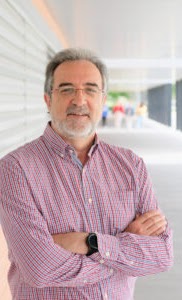}}]{Luis Montano} (Senior Member, IEEE) received his degree in industrial engineering in 1981 and his doctorate in 1987 from the University of Zaragoza, Spain. He is Full Professor at the University of Systems
Engineering and Automation at the University of Zaragoza. He was Director of the Department of Computer Science and Systems Engineering and Deputy Director of the Aragon Engineering Research Institute of the University of Zaragoza. He is Principal Researcher of the Robotics, Computer Vision and Artificial Intelligence research Group of the Institute. His main research interests in robotics are: planning and navigation in dynamic environments, multi-robot systems.
\end{IEEEbiography}

\end{document}